\documentclass{article}

\PassOptionsToPackage{numbers,sort, compress}{natbib}

\usepackage[final]{neurips_2021}

\usepackage{amssymb}
\usepackage{amsmath}
\usepackage{graphicx}

\usepackage[utf8]{inputenc} 
\usepackage[T1]{fontenc}    
\usepackage{hyperref}       
\usepackage{url}            
\usepackage{booktabs}       
\usepackage{amsfonts}       
\usepackage{nicefrac}       
\usepackage{microtype}      

\usepackage{xcolor}         
\usepackage{pifont}
\usepackage{wrapfig}

\usepackage{mathtools}

\usepackage{tikz}
\usepackage{caption}
\usepackage{subcaption}

\hypersetup{colorlinks,citecolor=blue}

\usepackage{amsthm}
\usepackage{thmtools}
\usepackage{thm-restate}
\theoremstyle{plain}

\usepackage{xargs}

\usepackage[ruled, vlined]{algorithm2e} 

\newcommand{\E}{\mathbb{E}}
\newcommand{\R}{\mathbb{R}}

\newcommand{\designnet}{\pi_\phi}
\newcommand{\policy}{\pi}
\newcommand{\latent}{\theta}
\newcommand{\critic}{U}
\newcommand{\criticnet}{U_\psi}

\newcommand{\prior}{p(\latent)}
\newcommand{\priori}[1]{p(\latent_{#1})}

\newcommandx{\histmarg}[1][1=\designnet]{p(h_T| #1)}
\newcommandx{\histlik}[1][1=\designnet]{p(h_T|\latent, #1)}
\newcommandx{\histliki}[2][2=\designnet]{p(h_T|\latent_{#1}, #2)}

\newcommand{\obslik}{p(y|\latent, \xi)}
\newcommand{\obsliki}[1]{p(y_{#1}|\latent, \xi_{#1})}
\newcommand{\obsmarg}{p(y|\xi)}

\newcommand{\obspost}{p(\latent|\xi, y)}

\newcommand{\sEIG}[1]{\mathcal{I}_T(#1)}

\newcommand{\cmark}{\ding{51}}
\newcommand{\xmark}{\ding{55}}

\newcommand{\iid}{\overset{\scriptstyle{\text{i.i.d.}}}{\sim}}

\title{Implicit Deep Adaptive Design: Policy--Based Experimental Design without Likelihoods}

\author{Desi R.~Ivanova\textsuperscript{\textdagger}\hspace{0.10em}
Adam Foster\textsuperscript{\textdagger}\hspace{0.10em}
Steven Kleinegesse\textsuperscript{\ddag}\hspace{0.10em}
Michael U. Gutmann\textsuperscript{\ddag}\hspace{0.10em}
Tom Rainforth\textsuperscript{\textdagger} \smallskip \\
\textsuperscript{\textdagger}Department of Statistics, University of Oxford \\
\textsuperscript{\ddag}  School of Informatics, University of Edinburgh \smallskip \\
\texttt{desi.ivanova@stats.ox.ac.uk}
}

\begin{document}

\maketitle

\begin{abstract}
We introduce implicit Deep Adaptive Design (iDAD), 
a new method for performing adaptive experiments in \emph{real-time} with \emph{implicit} models. iDAD amortizes the cost of Bayesian optimal experimental design (BOED) by learning a design \emph{policy network} upfront, which can then be deployed quickly at the time of the experiment. 
The iDAD network can be trained on any model which simulates differentiable samples, unlike previous design policy work that requires a closed form likelihood and conditionally independent experiments.
At deployment, iDAD allows design decisions to be made in milliseconds, in contrast to traditional BOED approaches that require heavy computation during the experiment itself.
We illustrate the applicability of iDAD on a number of experiments, and show that it provides a fast and effective mechanism for performing adaptive design with implicit models.
\end{abstract}


\section{Introduction}

Designing experiments to maximize the information gathered about an underlying  process is a key challenge in science and engineering.
Most such experiments are naturally \emph{adaptive}---we can design later iterations on the basis of data already collected, refining our understanding of the  process with each step~\citep{myung2013,ryan2016review,rainforth2017thesis}.
For example, suppose that a chemical contaminant has accidentally been released and
is rapidly spreading; we need to quickly discover its unknown source. 
To this end, we measure the contaminant concentration level at locations  $\xi_1, \dots, \xi_T$ (our experimental designs), obtaining observations $y_1,\dots,y_T$.
Provided we can perform the necessary computations sufficiently quickly, we can design each $\xi_t$ using data from steps $1,\dots,t-1$ to narrow in on the source.

Bayesian optimal experimental design (BOED) \citep{lindley1956,chaloner1995} is a principled model-based framework for choosing designs optimally; it has been successfully adopted in a diverse range of scientific fields \citep{vanlier2012,shababo2013bayesian,vincent2017darc}. 
In BOED, the unknown quantity of interest (e.g.~contaminant location) is encapsulated by a parameter $\theta$, and our initial information about it by a prior $\prior$.
A simulator, or likelihood, model $y|\theta,\xi$ describes the relationship between $\theta$, our controllable design $\xi$, and the experimental outcome $y$.
To select designs \emph{optimally}, the guiding principle is \emph{information maximization}---we select the design that maximizes the expected (Shannon) information gained about $\theta$ from the data $y$, or, equivalently, that maximizes the mutual information between $\theta$ and $y$.

This naturally extends to adaptive settings by considering the \emph{conditional} expected information gain given previously collected data.
The traditional approach, depicted in Figure~\ref{fig:boed_approaches_traditional}, is to fit a posterior  $p(\theta|\xi_{1:t-1},y_{1:t-1})$ after each iteration, and then select $\xi_t$ in a myopic fashion  using the one-step mutual information (see, e.g.,~\citep{ryan2016review} for a review). 
Unfortunately, this approach necessitates significant computation at each $t$ and does not lend itself to selecting optimal designs quickly and adaptively.

Recently, \citet{foster2021dad} proposed an exciting alternative approach,  called Deep Adaptive Design (DAD), that is based on learning design \emph{policies}.
DAD provides a way to avoid significant computation at deployment-time by, prior to the experiment itself, learning a design policy network that takes past design-outcome pairs and near--instantaneously returns the design for the next stage of the experiment. 
The required training is done using simulated experimental histories, without the need to estimate any posterior or marginal distributions.
DAD further only needs a single policy network to be trained for multiple experiments, further allowing for \emph{amortization} of the adaptive design process.
Unfortunately, DAD requires conditionally independent experiments and only works for the restricted class of models that have an explicit likelihood model we can simulate from, evaluate the density of, and calculate derivatives for, substantially reducing its applicability.

To address this shortfall, we instead consider a far more general class of models where we require only the ability to simulate $y | \theta,\xi$ and compute the derivative $\partial y/\partial \xi$, e.g.~via automatic differentiation~\citep{baydin2018ad}.
Such models are ubiquitous in scientific modelling and include differentiable \emph{implicit models} \citep{graham2017asymptotically}, for which the likelihood density $p(y|\theta,\xi)$ is intractable. Examples include mixed effects models \citep{foster2019variational, gelman2013bayesian}, various models from chemistry and epidemiology \citep{allen2008construction}, the Lotka Volterra model used in ecology \citep{graham2017asymptotically}, and models specified via stochastic differential equations (such as the SIR model \citep{cook2008optimal}). 


To perform rapid adaptive experimentation with this large class of models, we introduce \emph{implicit Deep Adaptive Design} (iDAD), a method for learning adaptive design policy networks using only simulated outcomes (see Figure~\ref{fig:boed_approaches_idad}).
To achieve this, we introduce likelihood-free lower bounds on the total information gained from a sequence of experiments, which iDAD utilizes to learn a deep policy network. 
This policy network amortizes the cost of experimental design for implicit models and can be run in milliseconds at deployment-time.
To train it, we show how the InfoNCE \citep{oord2018representation} and NWJ \citep{nguyen2010nwj} bounds, popularized in representation learning, can be applied to the policy-based experimental design setting.
The optimization of both of these bounds involves simultaneously learning an auxiliary \textit{critic network}, bringing an important added benefit: it can be used to perform likelihood-free posterior inference of the parameters given the data acquired from the experiment.

We also relax DAD's requirement for experiments to be conditionally independent, allowing its application in complex settings like time series data, and, through innovative architecture adaptations, also provide improvements in the conditionally independent setting as well.
This further expands the model space for policy-based BOED, and leads to additional performance improvements.

Critically, iDAD forms the first method in the literature that can practically perform real-time adaptive BOED with implicit models: previous approaches are either not fast enough to run in real-time for non-trivial models, or require explicit likelihood models.
We illustrate the applicability of iDAD on a range of experimental design problems, highlighting its benefits over existing baselines, even finding that it often outperforms costly non-amortized approaches.
Code for iDAD is publicly available at {\small \url{https://github.com/desi-ivanova/idad}}.

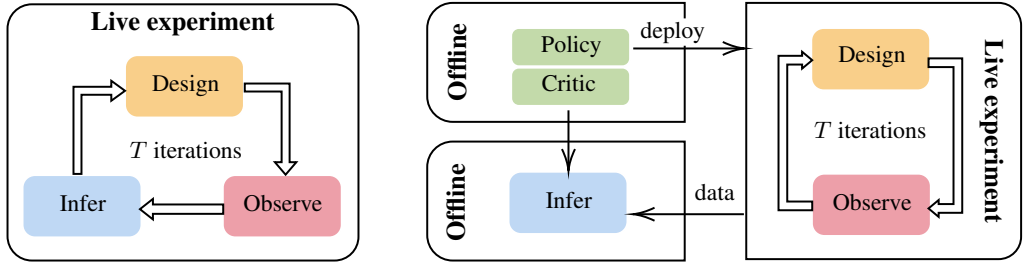
\begin{figure}[t]
\begin{subfigure}[b]{0.37\textwidth}
\centering
\tikzset{every picture/.style={line width=0.75pt}} 
\begin{tikzpicture}[x=0.75pt,y=0.75pt,yscale=-1,xscale=1]

\draw   (552.39,80.67) .. controls (559.04,80.67) and (564.43,86.06) .. (564.43,92.7) -- (564.43,197.96) .. controls (564.43,204.61) and (559.04,210) .. (552.39,210) -- (399.73,210) .. controls (392.93,210) and (387.43,204.49) .. (387.43,197.7) -- (387.43,92.96) .. controls (387.43,86.17) and (392.93,80.67) .. (399.73,80.67) -- cycle ;
\draw  [color={rgb, 255:red, 255; green, 255; blue, 255 }  ,draw opacity=1 ][fill={rgb, 255:red, 245; green, 166; blue, 35 }  ,fill opacity=0.53 ] (446.67,112.63) .. controls (446.67,109.15) and (449.49,106.33) .. (452.97,106.33) -- (500.37,106.33) .. controls (503.85,106.33) and (506.67,109.15) .. (506.67,112.63) -- (506.67,131.53) .. controls (506.67,135.01) and (503.85,137.83) .. (500.37,137.83) -- (452.97,137.83) .. controls (449.49,137.83) and (446.67,135.01) .. (446.67,131.53) -- cycle ;
\draw  [color={rgb, 255:red, 255; green, 255; blue, 255 }  ,draw opacity=1 ][fill={rgb, 255:red, 208; green, 2; blue, 27 }  ,fill opacity=0.38 ] (496,173.15) .. controls (496,169.68) and (498.82,166.86) .. (502.3,166.86) -- (549.7,166.86) .. controls (553.18,166.86) and (556,169.68) .. (556,173.15) -- (556,192.04) .. controls (556,195.51) and (553.18,198.33) .. (549.7,198.33) -- (502.3,198.33) .. controls (498.82,198.33) and (496,195.51) .. (496,192.04) -- cycle ;
\draw  [color={rgb, 255:red, 255; green, 255; blue, 255 }  ,draw opacity=1 ][fill={rgb, 255:red, 74; green, 144; blue, 226 }  ,fill opacity=0.35 ] (395,173.15) .. controls (395,169.68) and (397.82,166.86) .. (401.3,166.86) -- (448.7,166.86) .. controls (452.18,166.86) and (455,169.68) .. (455,173.15) -- (455,192.04) .. controls (455,195.51) and (452.18,198.33) .. (448.7,198.33) -- (401.3,198.33) .. controls (397.82,198.33) and (395,195.51) .. (395,192.04) -- cycle ;
\draw   (424.73,166.28) -- (424.73,126.33) -- (439.77,126.33) -- (439.77,129.35) -- (446.7,124.32) -- (439.77,119.28) -- (439.77,122.3) -- (420.7,122.3) -- (420.7,166.28) -- cycle ;
\draw   (507.08,124.69) -- (525.67,124.69) -- (525.67,160.36) -- (522.29,160.36) -- (527.68,166.66) -- (533.08,160.36) -- (529.7,160.36) -- (529.7,120.66) -- (507.08,120.66) -- cycle ;
\draw   (496.22,186.43) -- (460.76,186.43) -- (460.76,189.23) -- (454.22,184.07) -- (460.76,178.9) -- (460.76,181.7) -- (496.22,181.7) -- cycle ;

\draw (477.35,114.67) node [anchor=north] [inner sep=0.75pt]  [font=\small] [align=left] {Design};
\draw (526.41,175.33) node [anchor=north] [inner sep=0.75pt]  [font=\small] [align=left] {Observe};
\draw (477.44,148.33) node [anchor=north] [inner sep=0.75pt]  [font=\footnotesize] [align=left] {$\displaystyle T$ iterations};
\draw (475.78,83) node [anchor=north] [inner sep=0.75pt]  [color={rgb, 255:red, 0; green, 0; blue, 0 }  ,opacity=1 ] [align=left] {\textbf{Live experiment}};
\draw (425.41,175.33) node [anchor=north] [inner sep=0.75pt]  [font=\small] [align=left] {Infer};
\end{tikzpicture}
\caption{Traditional BOED: costly computations (design optimisation and parameter inference) are required at each iteration.}
\label{fig:boed_approaches_traditional}
\end{subfigure}
\hfill
\begin{subfigure}[b]{0.6\textwidth}
\centering
\tikzset{every picture/.style={line width=0.75pt}} 

\begin{tikzpicture}[x=0.75pt,y=0.75pt,yscale=-1,xscale=1]

\draw  [color={rgb, 255:red, 255; green, 255; blue, 255 }  ,draw opacity=1 ][fill={rgb, 255:red, 245; green, 166; blue, 35 }  ,fill opacity=0.53 ] (273.67,98.63) .. controls (273.67,95.15) and (276.49,92.33) .. (279.97,92.33) -- (327.37,92.33) .. controls (330.85,92.33) and (333.67,95.15) .. (333.67,98.63) -- (333.67,117.53) .. controls (333.67,121.01) and (330.85,123.83) .. (327.37,123.83) -- (279.97,123.83) .. controls (276.49,123.83) and (273.67,121.01) .. (273.67,117.53) -- cycle ;
\draw  [color={rgb, 255:red, 255; green, 255; blue, 255 }  ,draw opacity=1 ][fill={rgb, 255:red, 208; green, 2; blue, 27 }  ,fill opacity=0.38 ] (274,172.15) .. controls (274,168.68) and (276.82,165.86) .. (280.3,165.86) -- (327.7,165.86) .. controls (331.18,165.86) and (334,168.68) .. (334,172.15) -- (334,191.04) .. controls (334,194.51) and (331.18,197.33) .. (327.7,197.33) -- (280.3,197.33) .. controls (276.82,197.33) and (274,194.51) .. (274,191.04) -- cycle ;
\draw   (274.17,184.33) -- (257.17,184.33) .. controls (257.17,184.33) and (257.17,184.33) .. (257.17,184.33) -- (257.17,106.38) .. controls (257.17,106.38) and (257.17,106.38) .. (257.17,106.38) -- (269.07,106.38) -- (269.07,104.33) -- (274.17,108.58) -- (269.07,112.83) -- (269.07,110.79) -- (261.57,110.79) .. controls (261.57,110.79) and (261.57,110.79) .. (261.57,110.79) -- (261.57,179.93) .. controls (261.57,179.93) and (261.57,179.93) .. (261.57,179.93) -- (274.17,179.93) -- cycle ;
\draw   (332.83,107.67) -- (349.83,107.67) .. controls (349.83,107.67) and (349.83,107.67) .. (349.83,107.67) -- (349.83,185.62) .. controls (349.83,185.62) and (349.83,185.62) .. (349.83,185.62) -- (337.93,185.62) -- (337.93,187.67) -- (332.83,183.42) -- (337.93,179.17) -- (337.93,181.21) -- (345.43,181.21) .. controls (345.43,181.21) and (345.43,181.21) .. (345.43,181.21) -- (345.43,112.07) .. controls (345.43,112.07) and (345.43,112.07) .. (345.43,112.07) -- (332.83,112.07) -- cycle ;
\draw    (183.9,102.67) -- (238,102.67) ;
\draw [shift={(240,102.67)}, rotate = 180] [color={rgb, 255:red, 0; green, 0; blue, 0 }  ][line width=0.75]    (10.93,-3.29) .. controls (6.95,-1.4) and (3.31,-0.3) .. (0,0) .. controls (3.31,0.3) and (6.95,1.4) .. (10.93,3.29)   ;
\draw  [color={rgb, 255:red, 255; green, 255; blue, 255 }  ,draw opacity=1 ][fill={rgb, 255:red, 74; green, 144; blue, 226 }  ,fill opacity=0.35 ] (121,171.15) .. controls (121,167.68) and (123.82,164.86) .. (127.3,164.86) -- (174.7,164.86) .. controls (178.18,164.86) and (181,167.68) .. (181,171.15) -- (181,190.04) .. controls (181,193.51) and (178.18,196.33) .. (174.7,196.33) -- (127.3,196.33) .. controls (123.82,196.33) and (121,193.51) .. (121,190.04) -- cycle ;
\draw  [color={rgb, 255:red, 255; green, 255; blue, 255 }  ,draw opacity=1 ][fill={rgb, 255:red, 151; green, 192; blue, 108 }  ,fill opacity=0.58 ] (122.33,116.2) .. controls (122.33,114.06) and (124.06,112.33) .. (126.2,112.33) -- (178.47,112.33) .. controls (180.6,112.33) and (182.33,114.06) .. (182.33,116.2) -- (182.33,127.8) .. controls (182.33,129.94) and (180.6,131.67) .. (178.47,131.67) -- (126.2,131.67) .. controls (124.06,131.67) and (122.33,129.94) .. (122.33,127.8) -- cycle ;
\draw  [color={rgb, 255:red, 255; green, 255; blue, 255 }  ,draw opacity=1 ][fill={rgb, 255:red, 151; green, 192; blue, 108 }  ,fill opacity=0.58 ] (122.33,96.2) .. controls (122.33,94.06) and (124.06,92.33) .. (126.2,92.33) -- (178.47,92.33) .. controls (180.6,92.33) and (182.33,94.06) .. (182.33,96.2) -- (182.33,107.8) .. controls (182.33,109.94) and (180.6,111.67) .. (178.47,111.67) -- (126.2,111.67) .. controls (124.06,111.67) and (122.33,109.94) .. (122.33,107.8) -- cycle ;
\draw    (151.24,133.9) -- (151.24,162.38) ;
\draw [shift={(151.24,164.38)}, rotate = 270] [color={rgb, 255:red, 0; green, 0; blue, 0 }  ][line width=0.75]    (10.93,-3.29) .. controls (6.95,-1.4) and (3.31,-0.3) .. (0,0) .. controls (3.31,0.3) and (6.95,1.4) .. (10.93,3.29)   ;
\draw    (185.71,186.33) -- (239.71,186.33) ;
\draw [shift={(183.71,186.33)}, rotate = 0] [color={rgb, 255:red, 0; green, 0; blue, 0 }  ][line width=0.75]    (10.93,-3.29) .. controls (6.95,-1.4) and (3.31,-0.3) .. (0,0) .. controls (3.31,0.3) and (6.95,1.4) .. (10.93,3.29)   ;
\draw   (368.19,79.67) .. controls (374.45,79.67) and (379.52,84.74) .. (379.52,91) -- (379.52,197.67) .. controls (379.52,203.93) and (374.45,209) .. (368.19,209) -- (240.9,209) .. controls (240.9,209) and (240.9,209) .. (240.9,209) -- (240.9,79.67) .. controls (240.9,79.67) and (240.9,79.67) .. (240.9,79.67) -- cycle ;
\draw   (92.83,139.71) .. controls (85.74,139.71) and (80,133.97) .. (80,126.88) -- (80,92.5) .. controls (80,85.41) and (85.74,79.67) .. (92.83,79.67) -- (209.86,79.67) .. controls (209.86,79.67) and (209.86,79.67) .. (209.86,79.67) -- (209.86,139.71) .. controls (209.86,139.71) and (209.86,139.71) .. (209.86,139.71) -- cycle ;
\draw   (92.83,209.71) .. controls (85.74,209.71) and (80,203.97) .. (80,196.88) -- (80,162.5) .. controls (80,155.41) and (85.74,149.67) .. (92.83,149.67) -- (209.86,149.67) .. controls (209.86,149.67) and (209.86,149.67) .. (209.86,149.67) -- (209.86,209.71) .. controls (209.86,209.71) and (209.86,209.71) .. (209.86,209.71) -- cycle ;
\draw  [color={rgb, 255:red, 255; green, 255; blue, 255 }  ,draw opacity=1 ][fill={rgb, 255:red, 255; green, 255; blue, 255 }  ,fill opacity=1 ] (187.71,88) -- (226,88) -- (226,100.79) -- (187.71,100.79) -- cycle ;

\draw (304.35,99.67) node [anchor=north] [inner sep=0.75pt]  [font=\small] [align=left] {Design};
\draw (304.41,174.33) node [anchor=north] [inner sep=0.75pt]  [font=\small] [align=left] {Observe};
\draw (303.44,138.33) node [anchor=north] [inner sep=0.75pt]  [font=\footnotesize] [align=left] {$\displaystyle T$ iterations};
\draw (372.11,144.33) node [anchor=north] [inner sep=0.75pt]  [color={rgb, 255:red, 0; green, 0; blue, 0 }  ,opacity=1 ,rotate=-90] [align=left] {\textbf{Live experiment}};
\draw (88.31,178.67) node [anchor=north] [inner sep=0.75pt]  [color={rgb, 255:red, 0; green, 0; blue, 0 }  ,opacity=1 ,rotate=-270] [align=left] {\textbf{Offline} };
\draw (88.97,107.67) node [anchor=north] [inner sep=0.75pt]  [color={rgb, 255:red, 0; green, 0; blue, 0 }  ,opacity=1 ,rotate=-270] [align=left] {\textbf{Offline} };
\draw (151.41,173.33) node [anchor=north] [inner sep=0.75pt]  [font=\small] [align=left] {Infer};
\draw (152.74,94.67) node [anchor=north] [inner sep=0.75pt]  [font=\small] [align=left] {Policy};
\draw (151.41,114.67) node [anchor=north] [inner sep=0.75pt]  [font=\small] [align=left] {Critic};
\draw (225.01,169.67) node [anchor=north] [inner sep=0.75pt]  [font=\small] [align=left] {data};
\draw (203.88,87.67) node [anchor=north] [inner sep=0.75pt]  [font=\small] [align=left] {deploy};

\end{tikzpicture}

\caption{Policy-based BOED using iDAD: a design policy and critic are learnt before the live experiment. The policy enables quick and adaptive experiments, the critic assists likelihood-free inference.}
\label{fig:boed_approaches_idad}
\end{subfigure}
\caption{Overview of adaptive BOED approaches applicable to implicit models. } 
\label{fig:boed_approaches}
\end{figure}
\section{Background}
\label{sec:background}

The BOED framework \citep{lindley1956} begins by specifying a Bayesian model of the experimental process, consisting  of a prior on the unknown parameters $\prior$, a set of controllable designs $\xi$, and a data generating process that depends on them $y | \theta, \xi$; as usual in BOED, we assume that $\prior$ does not depend on $\xi$.
In this paper, we consider the situation where  $y | \theta, \xi$ is specified \textit{implicitly}. 
This means that it is defined by a deterministic transformation, $f(\varepsilon; \theta, \xi)$, of a base (or noise) random variable, $\varepsilon$, that is independent of the parameters and the design; e.g.,~$\varepsilon\sim \mathcal{N}(\varepsilon;0,I)$. 
The function $f$ is itself often not known explicitly in closed form, but is implemented as a stochastic computer program (i.e. simulator) with input $(\theta,\xi)$ and $\varepsilon$ corresponding to the draws from the underlying random number generator (or equivalently the random seed).
Regardless, the resulting induced likelihood density $\obslik$ is still generally intractable, but sampling $y|\theta,\xi$ is possible.

Having acquired a design-outcome pair $(\xi, y)$, we can quantify the amount of information we have gained about $\theta$ by calculating the reduction in entropy from the prior to the posterior. 
We can further assess the quality of a design $\xi$ before acquiring $y$, by computing the expected reduction in entropy with respect to the marginal distribution of the outcome, $\obsmarg = \E_{\prior}[\obslik]$. The resulting quantity, called the \textit{expected information gain} (EIG), is of central interest in BOED and is defined as
\begin{align}
    I(\xi) \coloneqq 
    \E_{\prior\obslik} \left[ \log \frac{\obspost}{\prior} \right]  = \E_{\prior\obslik} \left[ \log \frac{\obslik}{\obsmarg} \right].
    \label{eq:EIG_objective}
\end{align}
Note that $I(\xi)$ is equivalent to the mutual information (MI) between the parameters $\theta$ and data $y$ when performing experiment $\xi$. The optimal $\xi$ is then the one that maximises the EIG, i.e. $\xi^* = \arg\max_{\xi}I(\xi)$. 
Performing this optimization is a major computational challenge since the information objective is doubly intractable \citep{rainforth2018nesting}. 
For implicit models, the cost becomes even greater as the likelihood  is also not available in closed form, so estimating it, along with the marginal likelihood $p(y|\xi)$, is already itself a major computational problem \citep{Green2015, Lintusaari2017, Sisson2018, Cranmer2020}. 

Jointly optimizing the design variables for all undertaken experiments at the same time using~\eqref{eq:EIG_objective} is called \emph{static} experimental design. In practice, however, we are often more interested in performing multiple experiments \emph{adaptively} in a sequence $\xi_1, \dots, \xi_T$, so that the choice of each $\xi_t$ can be guided by past experiments, namely the corresponding \emph{history} $h_{t-1} \coloneqq \{(\xi_i, y_i) \}_{i=1:t-1}$.
The typical approach in such settings is to sequentially perform (approximate) posterior inference for $\theta|h_{t-1}$, followed by a one-step look ahead (myopic) BOED optimization that conditions on the observed history. In other words, to determine the designs $\xi_1,\ldots, \xi_T$, we sequentially optimize the objectives
\begin{equation}
\label{eq:sBOED}
    I_{h_{t-1}}(\xi_t) \coloneqq  
    \E_{p(\theta|h_{t-1}) p(y_t|\theta, \xi, h_{t -1})} \left[ \log \frac{p(y_t|\theta, \xi, h_{t -1})}{p(y_t|\xi, h_{t-1})}\right], \quad t=1, \ldots, T.
\end{equation}
However, such approaches incur significant computational cost during the experiment itself, particularly for implicit models \citep{Hainy2016likelihoodfree, kleinegesse2020sequential, foster2020unified}.
This has critical consequences: in most cases they cannot be run in real-time, undermining one's ability to use them in practice.

\subsection{Policy-based adaptive design with likelihoods}
For tractable likelihood models, \citet{foster2021dad} proposed a new framework, called Deep Adaptive Design (DAD), for adaptive experimental design that avoids expensive computations during the experiment. 
To achieve this, they introduce a parameterized deterministic design function, or policy,  $\pi_\phi$  that takes the history $h_{t-1}$ as input and returns the design $\xi_t=\pi_\phi(h_{t-1})$ to be used for the next experiment as output. 
This set-up allows them to consider the objective
\begin{align}
    \sEIG{\pi_\phi} 
    = \E_{\prior \histlik[\pi_\phi]} \left[\sum_{t=1}^{T} I_{h_{t-1}}(\xi_t)\right], \quad \xi_t = \pi_\phi(h_{t-1}), 
    \label{eq:EIG_policy}
\end{align}
which crucially depends on the policy $\policy$ rather than the individual design $\xi_t$. Learning a policy up-front, rather than designs, is what allows adaptive experiments to be performed in real-time. 

Under the assumption that $y_t$ is independent of $h_{t-1}$ conditional on the parameters~$\theta$ and the design~$\xi_t$, i.e.~$p(y_t|\theta, \xi_t, h_{t -1})=p(y_t|\theta, \xi_t)$, \citet{foster2021dad} showed that the objective can be simplified to
\begin{align}
    \sEIG{\pi_\phi} 
	= \E_{\prior \histlik[\pi_\phi]}\left[ \log \frac{\histlik[\pi_\phi]}{\histmarg[\pi_\phi]} \right], \quad \histlik[\pi_\phi]=\prod_{t=1}^T\obsliki{t}.
    \label{eq:sEIG_objective}
\end{align}
To deal with the marginal $\histmarg[\pi_\phi]$ in the denominator, they then derived several optimizable lower bounds on $\sEIG{\pi_\phi}$, such as the sequential Prior Contrastive Estimation (sPCE) bound
\begin{equation}
	\label{eq:sPCE_objective}
	\mathcal{L}^\text{sPCE}_T(\designnet,L) = \E_{\priori{0}\histlik[\designnet]
	\priori{1:L}} \left[ \log \frac{\histliki{0}[\designnet]}{\frac{1}{L+1}\sum_{\ell=0}^L \histliki{\ell}[\designnet]} \right] \le \sEIG{\designnet} ~~ \forall L \geq 1.
\end{equation}
The parameters of the policy $\phi$, which takes the form of a deep neural network, are now learned prior to the experiment(s) using stochastic gradient ascent on this bound with simulated experimental histories.
Design decisions can then be made using a single forward pass of $\pi_{\phi}$ during deployment. 
Unfortunately, training the DAD network by optimizing \eqref{eq:sPCE_objective} requires the likelihood density $\histlik[\pi]$ to be analytically available---an assumption that is too restrictive in many practical situations.
The architecture for DAD also assumes conditionally independent designs, which is unsuitable in some settings like time-series data.
Our method lifts both of these restrictions.

\section{Implicit Deep Adaptive Design}
\label{sec:method}

We have seen that the traditional step-by-step approach to adaptive design for implicit models \citep{Hainy2016likelihoodfree, kleinegesse2020sequential,foster2020unified} is too costly to deploy for most applications, whilst the only existing policy-based approach, DAD~\citep{foster2021dad}, makes restrictive assumptions that prevent it being applied to implicit models.
We aim to relax the restrictive assumptions of the latter, making policy-based BOED applicable to all models where we can sample from $y|\theta,\xi$ and compute the derivative $\partial y / \partial \xi$, a strict superset of the class of models that can be handled by DAD.
This requires new training objectives for the policy network that do not involve an explicit likelihood and are not based on conditionally independent experiments, along with new architectures that work for non-exchangeable models like time series.

\subsection{Information lower bounds for policy-based experimental design without likelihoods}
To establish a suitable likelihood-free training objective for the implicit setting, our high-level idea is to leverage recent advances in variational MI \citep[see][for an overview]{poole2018variational},
which have shown promise for \emph{static} BOED \citep{foster2020unified,kleinegesse2020mine, kleinegesse2021gradientbased}.
While using these bounds in the traditional framework of~\eqref{eq:sBOED} would not permit real-time experiments, one could consider a naive application of them to the policy objective of \eqref{eq:EIG_policy} by replacing each $I_{h_{t-1}}$ with a suitable variational lower bound that uses a `critic' $U_{t} :\mathcal{H}^{t-1} \times \Theta \rightarrow \R$ to avoid explicit likelihood evaluations, where $\mathcal{H}^{t-1}$ and $\Theta$ are the spaces of histories and parameters respectively. 
An effective critic successfully encapsulates the true likelihood, tightening the bound. Although its form depends on the choice of bound, all critics are parametrized and trained in the same way, namely by a neural network $U_{\phi_t}$ which is optimized to tighten the bound. 
Unfortunately, replacing each $I_{h_{t-1}}$ involves learning $T$ such critic networks and requires samples from all posteriors $p(\theta|h_{t-1})$, which will typically be impractically costly.

To avoid this issue, we show that we can obtain a unified information objective similar to~\eqref{eq:sEIG_objective}, \emph{even without conditionally independent experiments}.
The following proposition therefore marks the first key milestone in eliminating the restrictive assumptions of \cite{foster2021dad}, by establishing a unified objective without intermediate posteriors that is valid even when the model itself changes between time steps.
\begin{restatable}[Generalized total expected information gain]{proposition}{unified}
	\label{thm:unified}
	Consider the data generating distribution
	   $\histlik[\pi]=\prod_{t=1:T} p(y_t | \theta, \xi_t, h_{t-1}),$
	where $\xi_t = \policy(h_{t-1})$ are the designs generated by the policy and, unlike in \eqref{eq:sEIG_objective}, $y_t$ is allowed to depend on the history $h_{t-1}$.
	Then we can write 
	\eqref{eq:EIG_policy} as
	\begin{align}
	\label{eq:our_unified_EIG}
	    \sEIG{\pi} 
	= \E_{\prior \histlik[\pi]}\left[ \log \histlik[\pi]\right] - \E_{\histmarg[\pi]}\left[\log \histmarg[\pi] \right].
	\end{align}
\end{restatable}
Proofs are presented in Appendix~\ref{sec:app_proofs}. The advantage of \eqref{eq:our_unified_EIG} is that we can draw samples from $\prior \histlik[\pi]$ simply by sampling our model and taking forward passes through the design network. 
However, neither of the \emph{densities} $\histlik[\pi]$ and $\histmarg[\pi]$ are tractable for implicit models.

To side-step this intractability, we observe that $\sEIG{\pi}$ takes an analogous form to a MI between $\theta$ and $h_T$.
For measure-theoretic reasons, namely because the $\xi_{1:T}$ are deterministic given $y_{1:T}$ (see Appendix~\ref{sec:app_proofs} for a full discussion), it is not the true MI.
However, the following two propositions show that we can treat $\sEIG{\pi}$ \emph{as if it were this MI}.
Specifically, we show that the InfoNCE \citep{oord2018representation} and NWJ \citep{nguyen2010nwj} bounds on the MI can be adapted to establish tractable lower bounds on our unified objective $\sEIG{\pi}$.
These two bounds both utilize a \emph{single} auxiliary critic network $U_{\psi}$ that is trained simultaneously with the design network.

\begin{restatable}[NWJ bound for implicit policy-based BOED]{proposition}{seqboundsnwj}
	\label{thm:seqbounds_nwj}
	For a design policy $\policy$ and a critic function $U:\mathcal{H}^T \times \Theta \rightarrow \R$, let
    	\begin{equation}
        	\label{eq:seqNWJ_objective}
             \mathcal{L}^\text{NWJ}_T(\policy, \critic) \coloneqq  \E_{\prior\histlik[\policy]} \left[ \critic(h_T, \theta) \right] - e^{-1}\E_{\prior\histmarg[\policy] } \left[ \exp(\critic(h_T, \theta)) \right],
    	\end{equation}
    	then $\sEIG{\policy} \geq \mathcal{L}^\text{NWJ}_T(\policy, \critic)$ holds for any $U$. Further, the inequality is tight for the \textit{optimal} critic $U^*_{\text{NWJ}}(h_T, \theta)=\log \histlik[\policy] - \log \histmarg[\policy] + 1$.
\end{restatable}

\begin{restatable}[InfoNCE bound for implicit policy-based BOED]{proposition}{seqboundsinfo}
	\label{thm:seqbounds_info}
	Let $\theta_{1:L}\sim \priori{1:L} \!=\! \prod_i\priori{i}$ be a set of contrastive samples where $L\ge 1$.
	For design policy $\policy$ and critic function $U\!:\!\mathcal{H}^T \!\times\! \Theta \!\rightarrow\! \R$, let
    	\begin{equation}
        	\label{eq:seqNCE_objective}
            \mathcal{L}^\text{NCE}_T(\policy, \critic; L) \coloneqq \E_{\priori{0}\histliki{0}[\policy]}\E_{\priori{1:L}} \left[ \log \frac{\exp(\critic(h_T, \theta_0))} {\frac{1}{L+1}\sum_{i=0}^L \exp(\critic(h_T, \theta_i))} \right], 
    	\end{equation}
    then $\sEIG{\policy} \geq \mathcal{L}^\text{NCE}_T(\policy, \critic; L)$ for any $U$ and $L\geq 1$.
    Further, the optimal critic, $\critic^*_{\text{NCE}}(h_T, \theta)= \log \histlik[\policy] + c(h_T)$ where $c(h_T)$ is any arbitrary function depending only on the history, recovers the sPCE bound in~\eqref{eq:sPCE_objective};
    the inequality is tight in the limit as $L\rightarrow \infty$ for this \textit{optimal} critic. 
\end{restatable}
We propose these two alternative bounds due to their complementary properties: the NWJ bound can have large variance, but tends to be less biased.
That is, the NWJ bound tends to be tighter for good critics, but is itself more difficult to reliably estimate and thus optimize. 
While the NWJ critic must learn to self-normalize, the InfoNCE bound avoids this issue but typically will not be tight for finite $L$ even with an optimal critic  (note $\mathcal{L}^\text{NCE}_T \le\log(L+1)$~\cite{poole2018variational}).
Consequently, 
only the NWJ objective recovers the true optimal policy if our critic has infinite capacity and our optimization scheme is perfect, i.e.~$\arg\max_{\policy}\max_{\critic} \mathcal{L}^\text{NWJ}_T(\policy, \critic)=\policy^*\neq\arg\max_{\policy}\max_{\critic} \mathcal{L}^\text{NCE}_T(\policy, \critic;L)$ in general, but it can be more difficult to work with in practice. 
We present a third bound that provides a potential solution to this,  and further discuss the relative merits of the two bounds,  in Appendix~\ref{sec:app_proofs}.

We note that for both bounds the optimal critic does not depend on the learned policy.
The final trained critic can be used to approximate the density ratio $\histlik[\pi]/\histmarg[\pi] = p(\theta|h_T)/\prior$, either directly in the case of the NWJ critic, or via self-normalization for the InfoNCE bound. 
We can use this to help approximate the posterior over $\theta$ given the collected real data from the experiment. 
This means we can perform likelihood-free inference after training the critic, which extends previous results \cite{kleinegesse2020mine, kleinegesse2021gradientbased} from the static to the adaptive policy-based setting.

\begin{algorithm}[t]
\SetAlgoLined
\SetKwInput{Input}{Input}
\SetKwInput{Output}{Output}
\Input{Differentiable simulator $f$,  sampler for prior $p(\theta)$, 
number of experimental steps $T$}
\Output{Design network $\pi_\phi$, critic network $U_{\psi}$}
\While{\textnormal{Computational training budget not exceeded}}{
    Sample $\theta \sim p(\theta)$ and set $h_0=\varnothing$
    
    \For {$t=1,...,T$}{
    Compute $\xi_t = \pi_\phi(h_{t-1})$
    
    Sample $\varepsilon_t\sim p(\varepsilon)$ and compute $y_t = f(\varepsilon_t; \xi_{t},\theta,h_{t-1})$ 
    
    Set $h_t=\{(\xi_1,y_1),...,(\xi_t,y_t)\}$ 
    }
Estimate $\nabla_{\phi,\psi} \mathcal{L}_T (\designnet, \criticnet)$ as per \eqref{eq:infonce_gradient} where  $\mathcal{L}_T$ is $\mathcal{L}^\text{NWJ}_T$ \eqref{eq:seqNWJ_objective} or $\mathcal{L}^\text{NCE}_T$ \eqref{eq:seqNCE_objective} 
    
    Update the parameters $(\phi, \psi)$ using stochastic gradient ascent scheme
}
For deployment, use the deterministic trained design network $\pi_\phi$ to obtain a designs $\xi_t$ directly.
\caption{Implicit Deep Adaptive Design with (iDAD)}
\label{algo:iDAD}
\end{algorithm}

\subsection{Parameterization and gradient estimation}

In practice, we represent the policy $\policy$ and the critic $\critic$ as neural networks, $\designnet$ and $\criticnet$ respectively,~such that the lower bounds become a function $\mathcal{L}(\designnet, \criticnet)$ of their parameters. 
By simultaneously optimizing $\mathcal{L}(\designnet, \criticnet)$ with respect to both $\phi$ and $\psi$, we both learn a tight bound that accurately represents the true MI and a design policy network that produces high-quality designs under this metric.

We optimize these bounds using stochastic gradient methods~\citep{robbins1951stochastic,kingma2014adam}.
For this, we must account for the fact that the parameter $\phi$ affects the probability distributions with respect to which expectations are taken.
We deal with this problem by utilizing the reparametrization trick~\citep{rezende2014stochastic,mohamed2020monte}, for which we assume that design space $\Xi$ and observation space $\mathcal{Y}$ are continuous.
To this end, we first formalize the notion of a differentiable implicit model in the adaptive design setting as
\begin{align}
        y_t = f(\varepsilon_t; \xi_t(h_{t-1}),\theta,h_{t-1}), \quad \text{where} \quad \theta \sim p(\theta), \quad \varepsilon_t \sim p(\varepsilon)~\; \forall t \in \{1,\dots,T\}
\end{align}
and we assume that we can compute the derivatives $\partial f / \partial \xi$ and $\partial f / \partial h$.
Interestingly, it is possible to use an implicit prior without access to the density $p(\theta)$, and we do not need access to $\partial f / \partial \theta$.

Under these conditions, we can express the bounds in terms of expectations that do not depend on $\phi$ or $\psi$, and hence move the gradient operator inside. For $\mathcal{L}^\text{NCE}_T(\policy_\phi, \critic_\psi; L)$, for example, we have
\begin{align}
\label{eq:infonce_gradient}
    \nabla_{\phi,\psi}\mathcal{L}^\text{NCE}_T = \E_{p(\theta_{0:L})p(\varepsilon_{1:T})}\left[ \nabla_{\phi,\psi} \log \frac{\exp(\critic_\psi(h_T(\varepsilon_{1:T},\pi_\phi), \theta_0))} {\frac{1}{L+1}\sum_{i=0}^L \exp(\critic_\psi(h_T(\varepsilon_{1:T},\pi_\phi), \theta_i))} \right].
\end{align}
While each element of the history $h_T$ depends on $\phi$ in a possibly nested manner, we do not need to explicitly keep track of these dependencies thanks to automatic differentiation \citep{baydin2018ad,paszke2019pytorch}.

Like DAD, our new method---which we call \emph{implicit Deep Adaptive Design} (iDAD)---is trained with simulated  histories $h_T=\{(\xi_i, y_i) \}_{i=1:T}$ prior to the actual experiment, allowing design decision to be made using a single forward pass during deployment. 
Unlike DAD, however, it does not require knowledge of the likelihood function, nor the assumption of conditionally independent designs, which significantly broadens its applicability.
A summary of the iDAD approach is given in Algorithm~\ref{algo:iDAD}.

\subsection{Network architectures}

\begin{figure}[t]
\begin{subfigure}[b]{0.47\textwidth}
\centering
\tikzset{every picture/.style={line width=0.75pt}} 
\begin{tikzpicture}[x=0.75pt,y=0.75pt,yscale=-1,xscale=1]

\draw  [fill={rgb, 255:red, 248; green, 231; blue, 28 }  ,fill opacity=0.37 ] (125.71,40.99) .. controls (125.71,37.56) and (128.49,34.79) .. (131.91,34.79) -- (150.51,34.79) .. controls (153.94,34.79) and (156.71,37.56) .. (156.71,40.99) -- (156.71,59.59) .. controls (156.71,63.01) and (153.94,65.79) .. (150.51,65.79) -- (131.91,65.79) .. controls (128.49,65.79) and (125.71,63.01) .. (125.71,59.59) -- cycle ;
\draw  [fill={rgb, 255:red, 155; green, 201; blue, 175 }  ,fill opacity=0.6 ] (173.71,41.28) .. controls (173.71,37.85) and (176.49,35.08) .. (179.91,35.08) -- (267.51,35.08) .. controls (270.94,35.08) and (273.71,37.85) .. (273.71,41.28) -- (273.71,59.88) .. controls (273.71,63.3) and (270.94,66.08) .. (267.51,66.08) -- (179.91,66.08) .. controls (176.49,66.08) and (173.71,63.3) .. (173.71,59.88) -- cycle ;
\draw    (274,50) -- (290.37,66.37) ;
\draw [shift={(291.79,67.79)}, rotate = 225] [color={rgb, 255:red, 0; green, 0; blue, 0 }  ][line width=0.75]    (10.93,-3.29) .. controls (6.95,-1.4) and (3.31,-0.3) .. (0,0) .. controls (3.31,0.3) and (6.95,1.4) .. (10.93,3.29)   ;
\draw  [fill={rgb, 255:red, 144; green, 19; blue, 254 }  ,fill opacity=0.2 ] (293,64.28) .. controls (293,60.85) and (295.78,58.08) .. (299.2,58.08) -- (364.8,58.08) .. controls (368.22,58.08) and (371,60.85) .. (371,64.28) -- (371,82.88) .. controls (371,86.3) and (368.22,89.08) .. (364.8,89.08) -- (299.2,89.08) .. controls (295.78,89.08) and (293,86.3) .. (293,82.88) -- cycle ;
\draw  [fill={rgb, 255:red, 155; green, 155; blue, 155 }  ,fill opacity=0.3 ] (173.71,84.28) .. controls (173.71,80.85) and (176.49,78.08) .. (179.91,78.08) -- (267.51,78.08) .. controls (270.94,78.08) and (273.71,80.85) .. (273.71,84.28) -- (273.71,102.88) .. controls (273.71,106.3) and (270.94,109.08) .. (267.51,109.08) -- (179.91,109.08) .. controls (176.49,109.08) and (173.71,106.3) .. (173.71,102.88) -- cycle ;
\draw    (274,95) -- (291.3,77.7) ;
\draw [shift={(292.71,76.29)}, rotate = 495] [color={rgb, 255:red, 0; green, 0; blue, 0 }  ][line width=0.75]    (10.93,-3.29) .. controls (6.95,-1.4) and (3.31,-0.3) .. (0,0) .. controls (3.31,0.3) and (6.95,1.4) .. (10.93,3.29)   ;
\draw  [color={rgb, 255:red, 0; green, 0; blue, 0 }  ,draw opacity=1 ][fill={rgb, 255:red, 74; green, 144; blue, 226 }  ,fill opacity=0.35 ] (125.71,84) .. controls (125.71,80.58) and (128.49,77.8) .. (131.91,77.8) -- (150.51,77.8) .. controls (153.94,77.8) and (156.71,80.58) .. (156.71,84) -- (156.71,102.6) .. controls (156.71,106.02) and (153.94,108.8) .. (150.51,108.8) -- (131.91,108.8) .. controls (128.49,108.8) and (125.71,106.02) .. (125.71,102.6) -- cycle ;
\draw    (156,51) -- (171.71,51) ;
\draw [shift={(173.71,51)}, rotate = 180] [color={rgb, 255:red, 0; green, 0; blue, 0 }  ][line width=0.75]    (10.93,-3.29) .. controls (6.95,-1.4) and (3.31,-0.3) .. (0,0) .. controls (3.31,0.3) and (6.95,1.4) .. (10.93,3.29)   ;
\draw    (156,94) -- (171.71,94) ;
\draw [shift={(173.71,94)}, rotate = 180] [color={rgb, 255:red, 0; green, 0; blue, 0 }  ][line width=0.75]    (10.93,-3.29) .. controls (6.95,-1.4) and (3.31,-0.3) .. (0,0) .. controls (3.31,0.3) and (6.95,1.4) .. (10.93,3.29)   ;

\draw (142.05,50.43) node  [font=\large] [align=left] {$\displaystyle h_{T}$};
\draw (222.02,43) node [anchor=north] [inner sep=0.75pt]  [font=\small] [align=left] {History encoder};
\draw (333.25,66) node [anchor=north] [inner sep=0.75pt]  [font=\small] [align=left] {Dot Product};
\draw (141.05,93.43) node  [font=\large] [align=left] {$\displaystyle \theta $};
\draw (220.56,87) node [anchor=north] [inner sep=0.75pt]  [font=\small] [align=left] {$\theta$ encoder};
\end{tikzpicture}
\caption{Critic network, $\criticnet$}
\label{fig:architecture_critic}
\end{subfigure}
\hfill
\begin{subfigure}[b]{0.48\textwidth}
\centering
\tikzset{every picture/.style={line width=0.75pt}} 

\begin{tikzpicture}[x=0.75pt,y=0.75pt,yscale=-1,xscale=1]

\draw  [fill={rgb, 255:red, 248; green, 231; blue, 28 }  ,fill opacity=0.37 ] (104.71,45.99) .. controls (104.71,42.56) and (107.49,39.79) .. (110.91,39.79) -- (129.51,39.79) .. controls (132.94,39.79) and (135.71,42.56) .. (135.71,45.99) -- (135.71,64.59) .. controls (135.71,68.01) and (132.94,70.79) .. (129.51,70.79) -- (110.91,70.79) .. controls (107.49,70.79) and (104.71,68.01) .. (104.71,64.59) -- cycle ;
\draw  [fill={rgb, 255:red, 155; green, 155; blue, 155 }  ,fill opacity=0.3 ] (271,46.2) .. controls (271,42.78) and (273.78,40) .. (277.2,40) -- (297.8,40) .. controls (301.22,40) and (304,42.78) .. (304,46.2) -- (304,64.8) .. controls (304,68.22) and (301.22,71) .. (297.8,71) -- (277.2,71) .. controls (273.78,71) and (271,68.22) .. (271,64.8) -- cycle ;
\draw  [fill={rgb, 255:red, 155; green, 201; blue, 156 }  ,fill opacity=0.6 ] (153.71,46.28) .. controls (153.71,42.85) and (156.49,40.08) .. (159.91,40.08) -- (247.51,40.08) .. controls (250.94,40.08) and (253.71,42.85) .. (253.71,46.28) -- (253.71,64.88) .. controls (253.71,68.3) and (250.94,71.08) .. (247.51,71.08) -- (159.91,71.08) .. controls (156.49,71.08) and (153.71,68.3) .. (153.71,64.88) -- cycle ;
\draw    (253,55) -- (268.71,55) ;
\draw [shift={(270.71,55)}, rotate = 180] [color={rgb, 255:red, 0; green, 0; blue, 0 }  ][line width=0.75]    (10.93,-3.29) .. controls (6.95,-1.4) and (3.31,-0.3) .. (0,0) .. controls (3.31,0.3) and (6.95,1.4) .. (10.93,3.29)   ;
\draw  [color={rgb, 255:red, 0; green, 0; blue, 0 }  ,draw opacity=1 ][fill={rgb, 255:red, 245; green, 166; blue, 35 }  ,fill opacity=0.53 ] (321.7,46) .. controls (321.7,42.58) and (324.48,39.8) .. (327.9,39.8) -- (346.5,39.8) .. controls (349.92,39.8) and (352.7,42.58) .. (352.7,46) -- (352.7,64.6) .. controls (352.7,68.02) and (349.92,70.8) .. (346.5,70.8) -- (327.9,70.8) .. controls (324.48,70.8) and (321.7,68.02) .. (321.7,64.6) -- cycle ;
\draw    (304,55) -- (319.71,55) ;
\draw [shift={(321.71,55)}, rotate = 180] [color={rgb, 255:red, 0; green, 0; blue, 0 }  ][line width=0.75]    (10.93,-3.29) .. controls (6.95,-1.4) and (3.31,-0.3) .. (0,0) .. controls (3.31,0.3) and (6.95,1.4) .. (10.93,3.29)   ;
\draw    (136,55) -- (151.71,55) ;
\draw [shift={(153.71,55)}, rotate = 180] [color={rgb, 255:red, 0; green, 0; blue, 0 }  ][line width=0.75]    (10.93,-3.29) .. controls (6.95,-1.4) and (3.31,-0.3) .. (0,0) .. controls (3.31,0.3) and (6.95,1.4) .. (10.93,3.29)   ;
\draw    (121.07,79.07) .. controls (167.45,114.42) and (280.25,113.75) .. (334.43,78.29) ;
\draw [shift={(119,77.43)}, rotate = 39.43] [color={rgb, 255:red, 0; green, 0; blue, 0 }  ][line width=0.75]    (10.93,-3.29) .. controls (6.95,-1.4) and (3.31,-0.3) .. (0,0) .. controls (3.31,0.3) and (6.95,1.4) .. (10.93,3.29)   ;

\draw (121.05,55.43) node  [font=\large] [align=left] {$h_{t}$};
\draw (203.02,48) node [anchor=north] [inner sep=0.75pt]  [font=\small] [align=left] {History \ encoder};
\draw (287.28,49) node [anchor=north] [inner sep=0.75pt]  [font=\small] [align=left] {MLP};
\draw (338.05,46) node [anchor=north] [inner sep=0.75pt]  [font=\large] [align=left] {$\xi_{t+1}$};
\draw (227.05,80) node [anchor=north] [inner sep=0.75pt]   [align=left] {$y_{t+1} \sim y|\xi_{t+1}, \theta ,h_{t}$};
\end{tikzpicture}
\caption{Policy network, $\designnet$}
\label{fig:architecture_policy}
\end{subfigure}
\caption{Overview of network architectures used in iDAD.\vspace{-10pt}} 
\label{fig:architectures}
\end{figure}
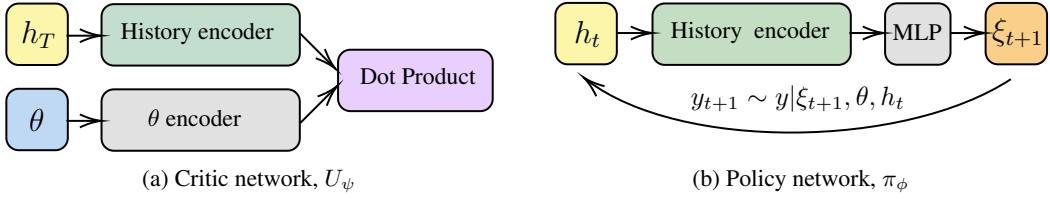

The iDAD approach involves the simultaneous training of the \emph{policy} $\pi_\phi$ and \emph{critic} $U_\psi$ networks.
It is essential to choose the neural architectures of these two components carefully to learn effective policies: poor choices of critic architecture will lead to loose, unrepresentative, bounds, while poor choices of policy architecture will directly lead to ineffective policies.
Good choices of architecture need to balance flexibility with ease of training, and will typically require the incorporation of problem-specific inductive biases.
A high-level summary of our architectures is shown in Figure~\ref{fig:architectures}.

The critic network, $U_{\psi}$, takes a \emph{complete} history $h_T$ and the parameter $\theta$ as input, and outputs a scalar.
Our suggested architecture first encodes the two inputs separately to representations of the same dimension, using a \emph{history encoder}, $E_{\psi_h}$, and a \emph{parameter encoder}, $E_{\psi_\theta}$, respectively.
The output of the critic is then simply taken as their dot product $\criticnet(h_T, \theta) \coloneqq E_{\psi_h}(h_T)^\top E_{\psi_\theta}(\theta)$; after training, the two encodings correspond to approximate sufficient statistics \citep{Chen2021}.
This setup corresponds to a separable critic architecture, as is commonly used in the representation learning literature~\cite{oord2018representation,bachman2019learning,chen2020simclr}.
While we use a simple MLP for $E_{\psi_\theta}$, the setup for $E_{\psi_h}$ varies with the context as we discuss below.

The policy network, $\pi_{\phi}$, takes the available history, $h_t$, as input, and outputs a design.
Our suggested architecture makes use of a history encoder, $E_{\phi_h}$, of the same form as $E_{\psi_h}$, except that it must now take in varying length inputs; its output remains a fixed dimensional embedding.
We then pass this embedding through an MLP to produce the design $\xi_{t+1}$.
At the next iteration of the experiment, the same policy network is then called again with the updated history $h_{t+1} = h_t \cup \{(\xi_{t+1},y_{t+1})\}$.

We use the same architecture for both history encoders, $E_{\psi_h}$ and $E_{\phi_h}$, but do not share network parameters between them.
This architecture first individually embeds each design--outcome pair $(\xi_t,y_t)$ to a corresponding representation, $r_t$, using a simple MLP that is shared across all time steps.
The produced history encoding is then an aggregation of these representations, with how this is done depending on whether the experiments are \emph{conditionally independent}, i.e.~$y_t \perp \!\!\! \perp h_{t-1} | \theta, \xi_t$, or not.

\citet{foster2021dad} proved that if experiments are conditionally independent, then the optimal policy is invariant to the order of the history. 
We prove that the same is true for the critic in Proposition~\ref{thm:invariance} in Appendix~\ref{sec:appendix:arch}.
In our setup, we can exploit this result by using a \emph{permutation invariant} aggregation strategy for $\{r_1,\dots,r_t\}$ when conditional independence holds.
The simplest approach to do this would be to use sum-pooling \citep{zaheer2017deep}, as was done in DAD. 
However, to improve on this, we instead propose using a more advanced permutation invariant architecture based on self-attention~\citep{vaswani2017attention,devlin2019bert,ramachandran2019standalone,huang2018musictransformer,parisotto2020stabilizing}, namely that of~\citet{parmar2018image}; we find this provides notable empirical gains.
When conditional independence does not hold, this approach is no longer appropriate and we instead use an LSTM~\cite{hochreiter1997long} for the aggregation.
See Appendix~\ref{sec:appendix:arch} for further details.

\section{Related work}
Adaptive policy-based BOED has only recently been introduced \citep{foster2021dad} and has not yet been extended to implicit models---the gap that this work addresses.
Previous approaches to adaptive experiments usually follow the two-step greedy procedure described in Section~\ref{sec:background}. 
Methods for MI/EIG estimation without likelihoods include the use of variational bounds \citep{foster2019variational,foster2020unified,kleinegesse2020mine} and ratio estimation \citep{kleinegesse2018efficient, kleinegesse2020sequential}; approximate Bayesian computation together with kernel density estimation \citep{Price2016}; and approximating the intractable likelihood first, for example via polynomial chaos expansion \citep{huan2013simulation}, followed by applying likelihood-based estimators, such as nested Monte Carlo~\cite{rainforth2018nesting}.
The maximization step in more traditional methods tends to rely on gradient-free optimization, including grid-search, evolutionary algorithms \citep{price2018induced}, Bayesian optimization \cite{kleinegesse2020sequential, foster2019variational}, or Gaussian process surrogates \citep{overstall2020bayesian}. More recently, gradient-based approaches have been introduced~\cite{foster2019variational,kleinegesse2020mine}, some of which allow the estimation and optimization simultaneously in a single stochastic-gradient scheme~\citep{huan2014gradient,foster2020unified, kleinegesse2021gradientbased}.
From a posterior estimation perspective, likelihood-free inference can be performed via approximate Bayesian computation \citep{Lintusaari2017, Sisson2018}, ratio estimation \citep{dutta2016likelihood}, conventional MCMC for methods that make tractable approximation to the likelihood \citep{huan2013simulation,huan2014gradient}, or as a byproduct of MI estimation \cite{foster2020unified, kleinegesse2018efficient, kleinegesse2020sequential, kleinegesse2021gradientbased}.

\section{Experiments}
\label{sec:experiments}

\begin{wraptable}[10]{r}{6.9cm}
\vspace{-0.4cm}
\setlength{\tabcolsep}{4pt}  
\renewcommand{\arraystretch}{0.9}  
	\caption{Key properties of considered methods.\vspace{-3pt}} 
	\label{tab:baselines}
	\begin{tabular}{lcccc}
	\toprule
		           & Adaptive & Real-time  & Implicit     \\
		\toprule
		Random            & \xmark   &   N/A &  \cmark   \\
		Equal interval      & \xmark   &  N/A  & \cmark \\
		MINEBED            & \xmark   &  N/A  & \cmark  \\
		SG-BOED             &  \xmark  &  N/A  & \cmark  \\
		Variational        &  \cmark  &  \xmark  & \cmark  \\
		DAD                &  \cmark  &  \cmark  &  \xmark   \\
		\textbf{iDAD}     &  \cmark  &  \cmark  &   \cmark    \\
		\bottomrule
	\end{tabular}
\end{wraptable}
We evaluate the performance of \textbf{iDAD} on a number of real-world experimental design problems and a range of baselines.
A summary of all the methods that we consider is given in Table~\ref{tab:baselines}.
Since we aim to perform adaptive experiments in \textit{real-time}, we focus mostly on baselines that do not require significant computational time during the experiment.
These include heuristic approaches that require no training, namely \textbf{equal} interval designs (when possible) and \textbf{random} designs, as well as static BOED approaches, where we, non-adaptively, choose all the designs prior to the experiment by optimising the mutual information objective of Equation~\eqref{eq:EIG_objective} with $\xi=\{\xi_1,...,\xi_T\}$ and $y=\{y_1,\dots,y_T\}$.
The static BOED approaches we consider are the \textbf{MINEBED} method of \cite{kleinegesse2020mine} and the likelihood-free ACE approach of \cite{foster2020unified}, where we use the prior as a proposal distribution, referring to this baseline as \textbf{SG-BOED}. al
We also implement the expensive traditional non-amortized myopic strategy described in Section~\ref{sec:background}, for which we use the \textbf{variational} approach of~\cite{foster2020unified}, with the Barber-Agakov bound~\citep{barber2003imalgorithm,foster2019variational},
at each experiment step (see Appendix~\ref{sec:app_variational_baseline} for details). 
Finally, where possible, we compare our method with DAD~\citep{foster2021dad}, in order to assess the performance gap that would arise if we had an analytic likelihood.
This comparison is done primarily for evaluation purposes---because it has access to the likelihood density, DAD serves as an upper bound on the performance iDAD can achieve; one should use explicit likelihood methods whenever possible.  

The main performance metric that we focus on is the total EIG, $\mathcal{I}_T(\pi)$, as given in \eqref{eq:our_unified_EIG}. 
In cases where the likelihood is available, 
we estimate the $\mathcal{I}_T(\pi)$ using the sPCE lower bound in~\eqref{eq:sPCE_objective} and its sister upper bound, the sequential Nested Monte Carlo bound~\citep[sNMC;][]{foster2021dad}. To ensure that the bounds are tight, we evaluate them with a large number of contrastive samples, i.e.~$L\ge10^5$. 
Where the likelihood is truly intractable, 
we assess the iDAD strategy in a more qualitative manner by looking at the optimal designs and approximate posteriors. 
For the adaptive experiments, we further consider the deployment time (i.e.~the time required to propose a design), which is a critical metric for our aims. 
All deployment times exclude the time needed to determine the first experiment as it can be computed up-front, during the training phase.  
Timings for training the policy itself are given in Appendix~\ref{sec:appendix_experiments}.

\subsection{Location Finding} \label{sec:experiments_locfin}

\begin{table}[t]
\centering
	\caption{Lower bounds on the total information, $\mathcal{I}_{10}(\pi)$, for the location finding experiment in Section~\ref{sec:experiments_locfin}. The bounds were estimated using $L=5\times10^5$ contrastive samples. Errors indicate $\pm1$ standard errors estimated over 4096  histories (128 for variational).  Corresponding upper bounds are given in Table~\ref{tab:locfin_T10_highd_upper} in Appendix~\ref{sec:appendix_experiments}. 
	\label{tab:locfin_T10_highd}}
	\smallskip
	\begin{tabular}{lcccc}
		Method \textbackslash~$\theta$ dim.   &  4D &  6D & 10D & 20D   \\
		\toprule
		Random &  4.791 $\pm$ 0.040	 &  3.468 $\pm$ 0.014  & 1.889 $\pm$ 0.011 &  0.552 $\pm$ 0.006 \\
		MINEBED       & 5.518 $\pm$ 0.028   &  4.221 $\pm$ 0.028 & 2.458 $\pm$ 0.029 & 0.801 $\pm$ 0.019   \\
		SG-BOED       & 5.547 $\pm$ 0.028   & 4.215 $\pm$ 0.030 & 2.454 $\pm$  0.029 & 0.803 $\pm$ 0.019 \\
		Variational   &  4.639   $\pm$ 0.144  & 3.625 $\pm$ 0.165 & 2.181 $\pm$ 0.151 & 0.669 $\pm$ 0.097 \\
		\textbf{iDAD} (NWJ)    &  \textbf{7.694} $\pm$ \textbf{0.045}  &  5.765 $\pm$ 0.036 & \textbf{3.252} $\pm$ \textbf{0.039} & \textbf{0.877} $\pm$ \textbf{0.022}   \\
		\textbf{iDAD} (InfoNCE) & \textbf{7.750} $\pm$ \textbf{0.039}  & \textbf{5.986} $\pm$  \textbf{0.037} & \textbf{3.251} $\pm$ \textbf{0.039} & \textbf{0.871}  $\pm$ \textbf{0.020} \\
		\midrule
		DAD     &  7.967 $\pm$ 0.034  & 6.300 $\pm$ 0.030 &  3.337 $\pm$ 0.039 & 0.937 $\pm$ 0.022    \\
		\bottomrule
	\end{tabular}
\end{table}

We first demonstrate our approach on the location finding experiment from \cite{foster2021dad}. 
Inspired by the acoustic energy attenuation model \cite{Sheng2005}, this experiment involves finding the locations of multiple hidden sources, each emitting a signal with intensity that decreases according to the inverse-square law.
The \emph{total intensity}---a superposition of these signals---can be measured noisily at any location. 
The design problem is choosing where to measure the total signal in order to uncover the sources.

\begin{table}[t]
\centering
	\caption{Lower and upper bounds on MI $\mathcal{I}_{10}(\pi)$ for different network architectures on location finding experiment using the InfoNCE bound. All estimates obtained as in Table~\ref{tab:locfin_T10_highd}.
		\label{tab:locfin_T10_attention}
	}
	\smallskip
	\begin{tabular}{llcc}
		Design & Critic  &  Lower bound  &  Upper bound \\
		\toprule
		\textbf{Attention} & \textbf{Attention}  &  \textbf{7.750} $\pm$ \textbf{0.039}  & \textbf{7.863} $\pm$ \textbf{0.043} \\
		Attention & Pooling   & 7.567 $\pm$ 0.037 &   7.632 $\pm$ 0.039  \\
		Pooling & Attention   & 7.398 $\pm$ 0.040 &   7.470 $\pm$ 0.042  \\
		Pooling & Pooling   & 7.135 $\pm$ 0.034   & 7.192  $\pm$  0.041     \\
		\bottomrule
	\end{tabular}
\end{table}

\begin{wraptable}[10]{r}{6.4cm}
\setlength{\tabcolsep}{4pt} 
\renewcommand{\arraystretch}{0.9} 
	\caption{Deployment time of adaptive methods in 2D, measured on a CPU. Errors were calculated on the basis of 10 runs.}
	\label{tab:locfin_T10_deployment_times}
	\begin{tabular}{lr}
		Method  &  Deployment time (sec.)  \\
		\toprule
		Variational &  2256.0\phantom{000} $\pm$ 1\% \\
		iDAD (NWJ)    &   0.0167 $\pm$ 2\%   \\
		iDAD (InfoNCE) &  0.0168 $\pm$ 2\%    \\ 
		\midrule
		DAD     &  0.0070 $\pm$ 6\% \\
		\bottomrule
	\end{tabular}
\end{wraptable}
In Table~\ref{tab:locfin_T10_highd} we can see that iDAD substantially outperforms all baselines including, perhaps surprisingly, the traditional (non-amortized) adaptive variational approach, despite its large computational cost shown Table~\ref{tab:locfin_T10_deployment_times}. The poor performance of the variational approach is likely driven by the inability of the mean-field variational family to capture the highly non-Gaussian true posterior, highlighting the detrimental effect that wrong posteriors can have on determining optimal designs when using the traditional sequential BOED approach.

Table~\ref{tab:locfin_T10_highd} further shows that the performance gap to the likelihood-based DAD method is small, even as the dimension of the design and parameter space grows. 
Though the information gained by all methods decreases with the dimensionality, this is to be expected: in higher dimensions it is inherently more difficult to infer the relative direction of the sources from observing their intensity.
Overall, this experiment demonstrates that iDAD is able to learn near-optimal amortized design policies without likelihoods, while being run in milliseconds at deployment.

\textbf{Ablation: attention to history.}~ We next assess the benefit of utilizing our proposed more sophisticated permutation invariant architectures, compared to the simple pooling of \citep{zaheer2017deep} used in \cite{foster2021dad}. Our approach incorporates attention layers into both networks that we train. This leads us to four possible combinations of network architectures.
Table~\ref{tab:locfin_T10_attention} compares the efficacy of the resulting design policies and strongly suggests that incorporating attention mechanisms in either and/or both networks improves performance, with inclusion in the design network particularly important.

We preform further ablation studies to investigate and demonstrate important properties of our method, such as its scalability with the number of experiments $T$, stability between different training runs and performance to errors in the design network (introduced by not training the network to convergence). Results and discussion are provided in Appendix~\ref{sec:app_experiments_locfin_furtherresults}. 

\subsection{Pharmacokinetic model} \label{sec:experiments_pharmaco}

\begin{figure}[t]
  \centering
  \includegraphics[width=0.31\textwidth]{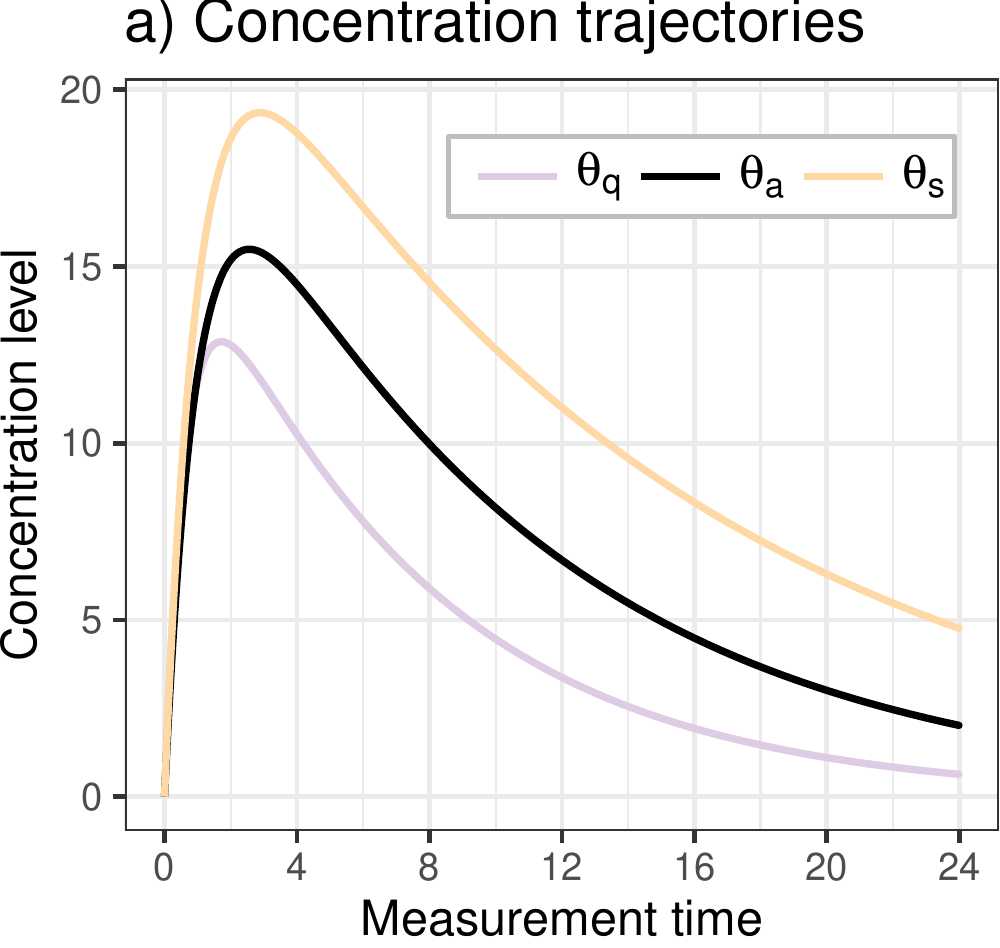}
  ~
  \includegraphics[width=0.31\textwidth]{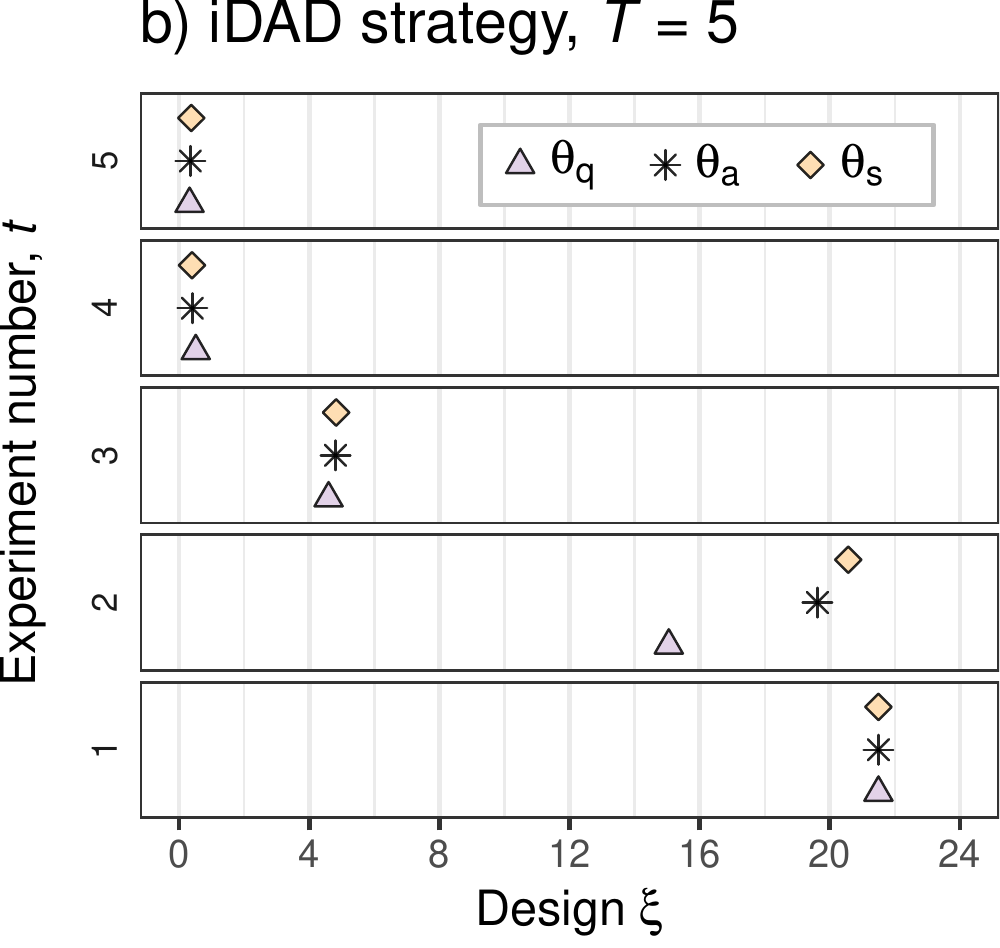}
  ~
   \includegraphics[width=0.31\textwidth]{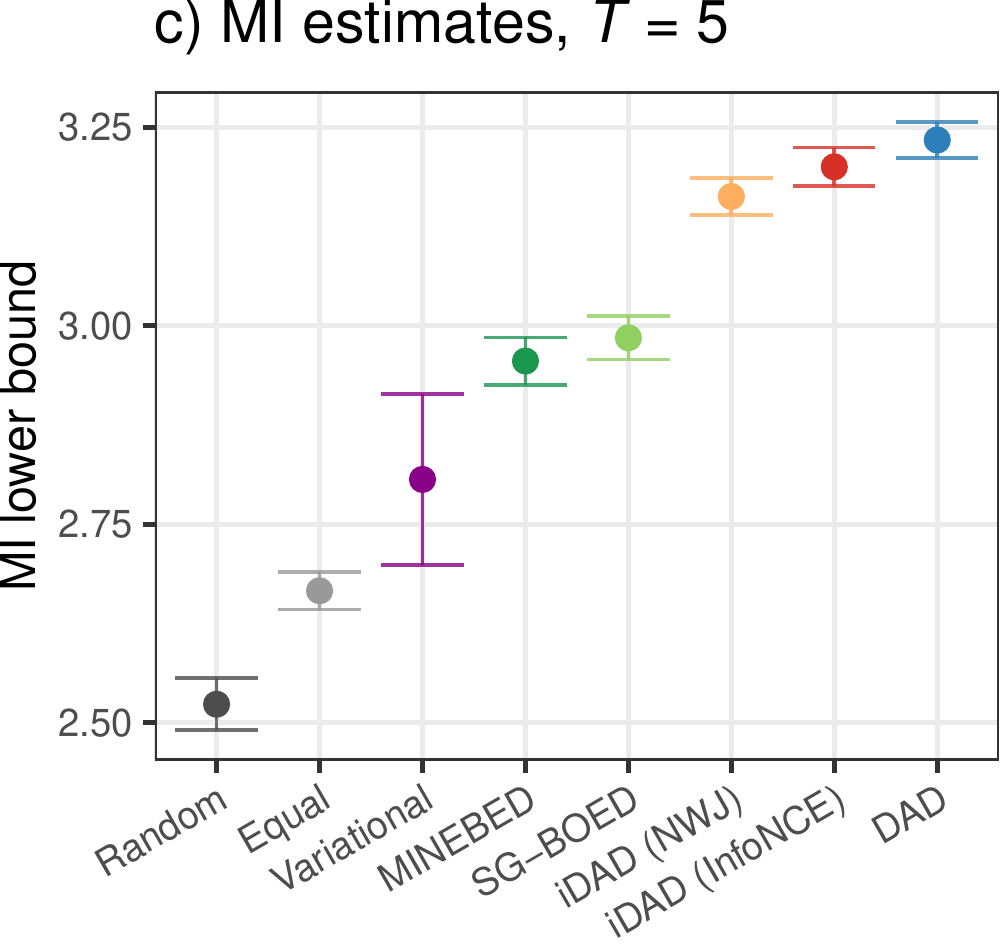}
  \caption{Plots for pharmacokinetics experiment. %
  a) Visualisation of model showing concentration level as a function of measurement time for 3 values of $\theta$, resulting in a quick ($\theta_q$), average ($\theta_a$), or slow ($\theta_s$) trajectory. 
  b) Designs selected by an iDAD policy trained with InfoNCE. 
  c) MI lower bounds achieved by iDAD and baselines. All estimates obtained as in Table~\ref{tab:locfin_T10_highd}.
  \vspace{-0.2cm}
  \label{fig:pk_trajectories_designs_eval}}
\end{figure}

Our next experiment is taken from the pharmacokinetics literature and has been studied in other recent works on BOED for implicit models \citep{kleinegesse2020mine, zhang2021sagabed}. 
Specifically, we consider the compartmental model of \cite{ryan2014towards}, for which the distribution of an administered drug through the body is governed by three parameters: absorption rate $k_\alpha$, elimination rate $k_e$, and volume $V$, which form the parameters of interest, i.e.~$\theta=(k_\alpha, k_e, V)$. Given $T=5$ patients, the design problem is to adaptively choose blood sampling times, $0\leq \xi_t \leq 24$ hours, for each, measured from the the point the drug was administered (with patient $2$ not being administered until after sampling patient $1$ etc).
Plausible concentration trajectories are shown in Figure~\ref{fig:pk_trajectories_designs_eval}{\color{red}a)}.
Full details and further results are given in Appendix~\ref{sec:appendix_experiments_pk}. 

We first qualitatively consider the design policy of iDAD (trained with the InfoNCE objective) in Figure~\ref{fig:pk_trajectories_designs_eval}{\color{red}b)}.
As we have not yet observed any data, the optimal design for the first patient (bottom row) is the same for all $\theta$. 
For the second patient, only guided by $\xi_1$ and the outcome $y_1$, iDAD is already able to distinguish between quickly and slowly decaying concentration trajectories: it proposes a significantly earlier measurement time for the quickly decaying trajectory (purple triangle, $\theta_q$) and later time for the slowly decaying one (yellow diamond, $\theta_s$).
For the third patient, iDAD always targets the peak of the drug concentration distribution which is quite similar for all $\theta$.
Measurements for the last two patients are made soon after the drug has been administered ($\sim15-30$ min), when concentration levels increase rapidly, to capture information about how quickly the drug is absorbed. 

\begin{wrapfigure}[13]{r}{0.32\textwidth}
  \vspace{-0.2cm}
  \centering
  \includegraphics[width=0.32\textwidth]{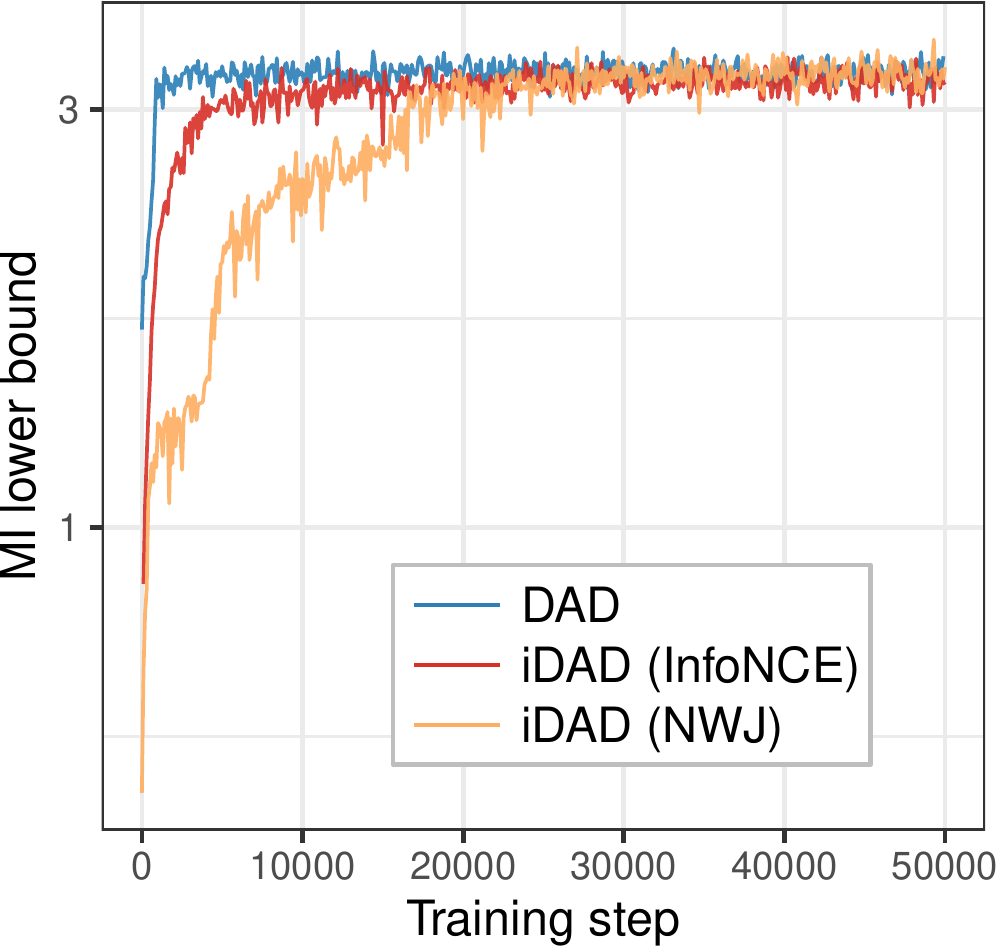}
  \vspace{-15pt}
  \caption{Convergence of MI lower bounds.}
  \label{fig:pk_loss}
\end{wrapfigure}
To provide more quantitative assessment and compare to our baselines, we again consider the final EIG values as shown in Figure~\ref{fig:pk_trajectories_designs_eval}{\color{red}c)}.
This reveals that the iDAD strategies perform best among the methods that are applicable to implicit models, confirming that the learnt policies propose superior designs. The performance gap to DAD, which relies on explicit likelihoods, is not statistically significant (at the 5\% level) for iDAD trained with InfoNCE, while significant, but still small, for NWJ. 

Finally, we consider the convergence of the iDAD networks under the different training objectives and compare to DAD for reference.
As shown in Figure~\ref{fig:pk_loss}, although all three converge to approximately the same value, they do so at rather different speeds: while DAD requires about 5000 gradient updates, implicit methods need longer training and tend to exhibit higher variance, particularly NWJ. 

\subsection{SIR Model} \label{sec:experiments_sir}

In this experiment, we demonstrate our approach on an implicit model from epidemiology. Namely, we consider a formulation of the stochastic SIR model~\citep{cook2008optimal} that is based on stochastic differential equations (SDEs), as done by~\citep{kleinegesse2021gradientbased}. Here, individuals in a fixed population belong to one of three categories: susceptible, infected or recovered.  
Susceptible people can become infected and then recover, with the dynamics of these two events being governed by two model parameters---the infection rate $\beta$ and the recovery rate $\gamma$.
Our aim is to determine the optimal  times $\tau$ at which to measure the  number of infected people, $I(\tau)$, in order to estimate the two parameters. This implicit model is challenging because data simulation is expensive, since we need to solve many SDEs, and experimental designs have a time-dependency. See Appendix~\ref{sec:appendix_sir} for full details.

\begin{wraptable}[11]{r}{5.6cm}
\centering
\caption{MI lower bounds ($\pm 1$ s.e.).}
	\begin{tabular}{lrr}
	\toprule
		Method              & Lower bound     \\
		\toprule
		Random              & 1.915 $\pm$ 0.032       \\ 
		Equal interval      & 2.669 $\pm$ 0.023       \\ 
		MINEBED      &  3.400 $\pm$     0.001       \\
		SG-BOED & 3.752 $\pm$ 0.020 \\
		iDAD (NWJ) & 3.869 $\pm$ 0.001    \\ 
		\textbf{iDAD} (InfoNCE) &  \textbf{3.915} $\pm$  \textbf{0.020}   \\ 
		\bottomrule
	\end{tabular}
	\label{tab:sir_results}
	\vspace{-0.2cm}
\end{wraptable}
We train a iDAD networks to perform $T=5$ experiments and compare against random, equal interval, and static design baselines; DAD cannot be run because the problem corresponds to a true implicit model. 
Table~\ref{tab:sir_results} shows lower bound estimates on the MI and demonstrates that iDAD outperforms all compared methods. 
Note that a degree of caution is required when analysing the results, as they are influenced by unavoidable biases in the estimation process. 
Namely, a critic is still required to estimate the MI lower bound, and there may be variations in the effectiveness of these critics, with less effective ones corresponding to looser bounds and therefore underestimating the true MI.
Nonetheless, for other models where such checks are possible, we have found the bounds to be relatively tight, while, even if this turns out not to be the case here, the fact that the critics for the static approaches are easier to train should mean our relative evaluations for iDAD (and Random) are still conservative compared to the other baselines.

Figure~\ref{fig:sir_results} further demonstrates important qualitative results for this model. Figure~\ref{fig:sir_results}{\color{red}a)} shows different epidemic trajectories, i.e.~the number of infected $I(\tau)$ people as a function of measurement time $\tau$, whilst \ref{fig:sir_results}{\color{red}b)} plots their corresponding designs obtained from the learned iDAD policy.
Importantly, diseases with a significantly different profile, e.g.~a slow ($R=1.3$) or a fast ($R=8.0$) spread result in different sets of optimal designs, highlighting the adaptivity of iDAD. 
Finally, Figure~\ref{fig:sir_results}{\color{red}c)} shows an example posterior distribution estimate from the learnt iDAD critic network, which we see is consistent with the ground truth parameters.

\begin{figure}[t]
  \centering
  \includegraphics[width=0.31\textwidth]{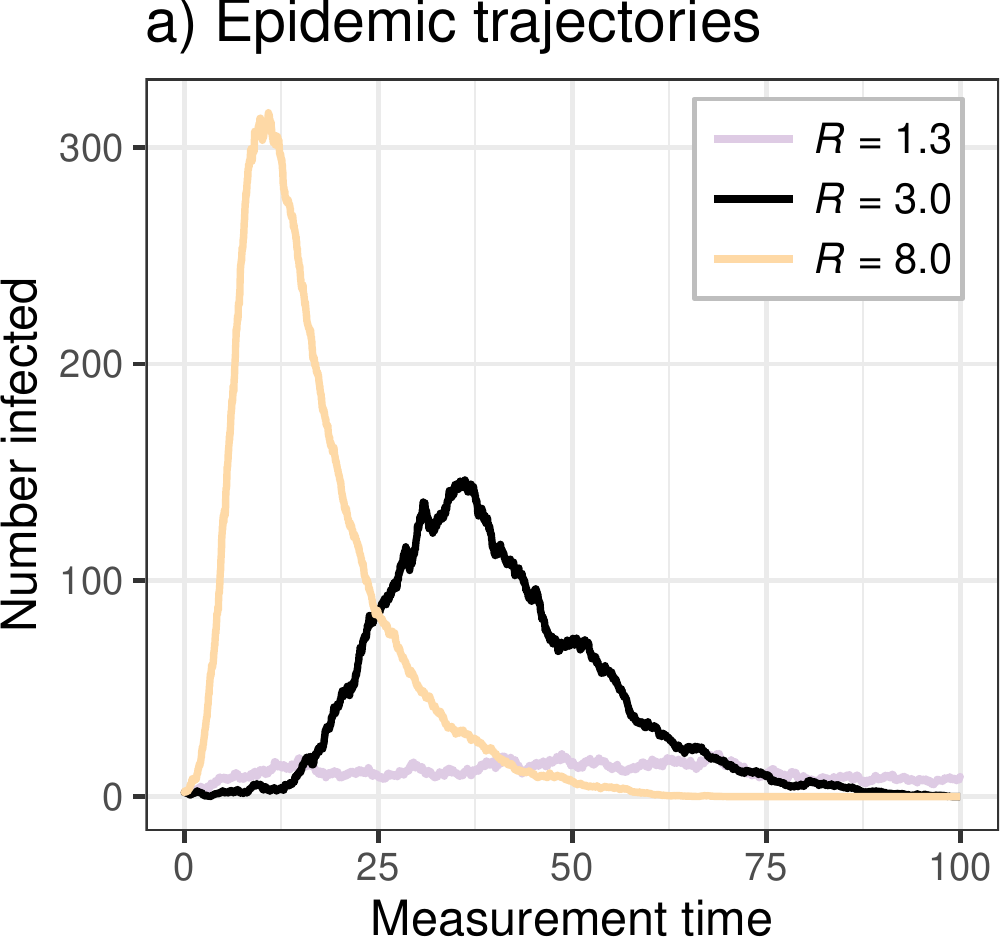}
  ~
  \includegraphics[width=0.31\textwidth]{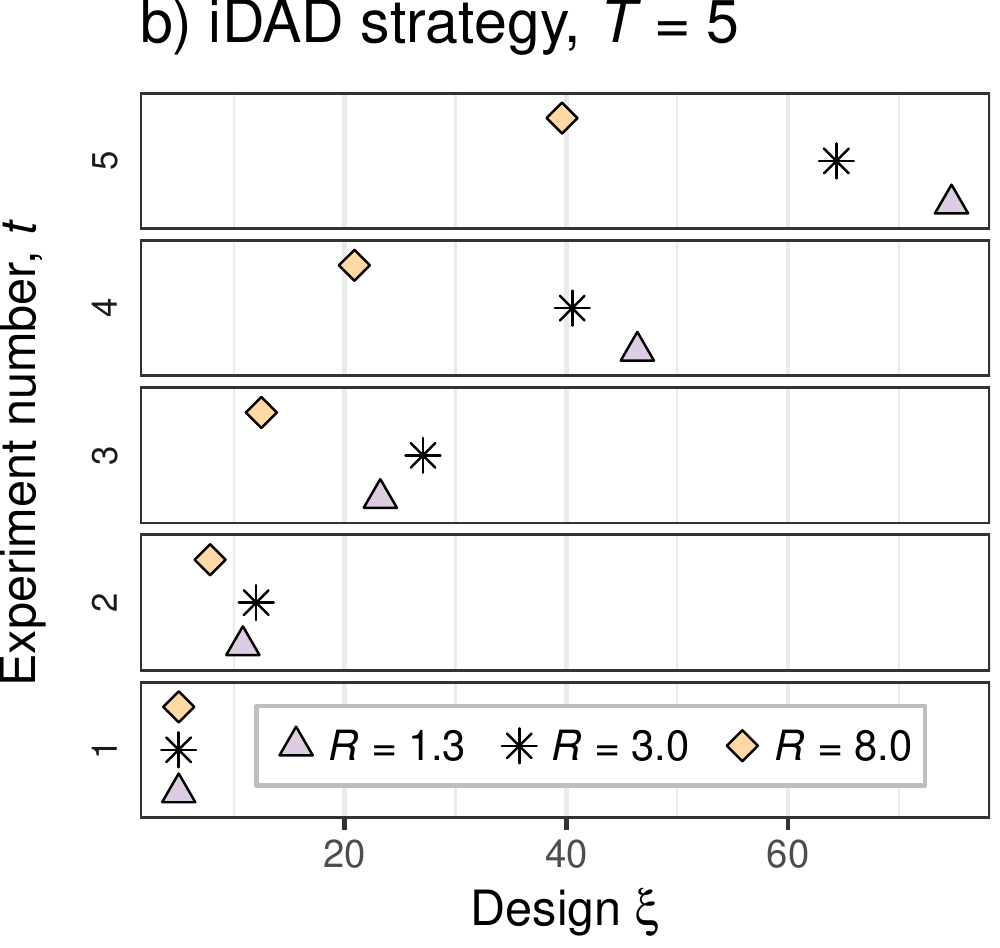}
  ~ 
   \includegraphics[width=0.31\textwidth]{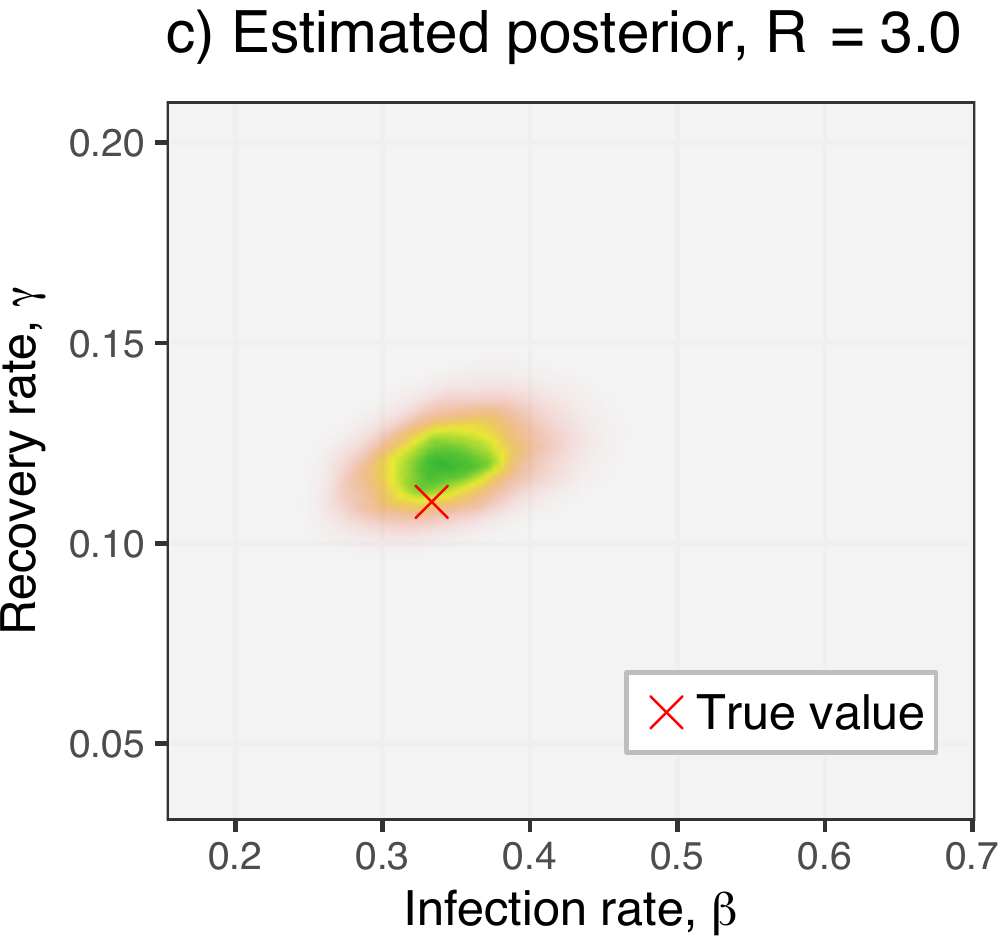}
  \caption{a) Epidemic trajectories for $3$ realization of $(\beta, \gamma)$ with different reproduction numbers $R=\beta / \gamma$. b) Designs selected by an iDAD policy trained with NWJ. c) Example posterior estimates from the critic network 
  given data generated with the ground-truth parameters shown by the red cross.}
  \label{fig:sir_results}
  \vspace{-0.3cm}
\end{figure}

\section{Discussion}
\label{sec:discussion}
{\bf Limitations.}
The benefit that iDAD can be used in live experiments comes at the cost of substantial training that can be computationally expensive. 
However, this is mitigated by its amortization of the adaptive design process, such that only one network needs training, even if we have multiple experiment instances.
The cost--performance trade-off can also be directly controlled by judicious choices of architecture and the amount of training performed.
Another natural limitation is that the use of gradients naturally restricts the approach to continuous design settings, something which future work might look to address. 

{\bf Conclusions.}
In this paper we introduced iDAD---the first policy-based adaptive BOED method that can be applied to implicit models.
By training a design network without likelihoods upfront, iDAD is thus the first method that allows real-time adaptive experiments for simulator-based models.
In our experiments, iDAD performed significantly better than all likelihood-free baselines. 
Further, by using models where the likelihood is available as a test bed, 
we found that it was able to almost match the analogous likelihood-based adaptive approach, which acts as an upper bound on what might be achieved without access to the likelihood itself. 
In conclusion, we believe iDAD marks a step change in Bayesian experimental design for \emph{implicit} models, allowing designs to be proposed quickly, adaptively, and non-myopically during the live experiment.

\begin{ack}
DRI is supported by EPSRC through the Modern Statistics and Statistical Machine Learning (StatML) CDT programme, grant no. EP/S023151/1. AF gratefully acknowledges funding from EPSRC grant no. EP/N509711/1. SK was supported in part by the EPSRC Centre for Doctoral Training in Data Science, funded by the UK Engineering and Physical Sciences Research Council (grant EP/L016427/1) and the University of Edinburgh.

\end{ack}

\bibliographystyle{plainnat}
\bibliography{refs}

\appendix
\newpage
\appendix

\section{Proofs}\label{sec:app_proofs}

We present proof for all propositions made in the paper,  restating each for convenience. We also include additional discussion on technical aspects of the paper.

\subsection{Unified objective for non-exchangeable experiments}
\unified*

\begin{proof}

Starting with the definition of the total EIG~\eqref{eq:EIG_policy} of a policy $\pi$:
\begin{align}
    \sEIG{\pi} &= \E_{\prior \histlik[\pi]} \left[\sum\nolimits_{t=1}^{T} I_{h_{t-1}}(\xi_t)\right]
    \intertext{we have by linearity of expectation}
               &= \sum\nolimits_{t=1}^{T} \E_{\prior \histlik[\pi]} \left[ I_{h_{t-1}}(\xi_t)\right]
    \intertext{and since $I_{h_{t-1}}$ doesn't depend on data acquired after $t-1$ (the future doesn't influence the past)}
     &= \sum\nolimits_{t=1}^{T} \E_{\prior p(h_{t-1}|\theta, \pi)} \left[ I_{h_{t-1}}(\xi_t)\right]
    \intertext{which, applying Bayes rule, is equivalent to}
    &=\sum\nolimits_{t=1}^{T} \E_{p(h_{t-1}|\pi) p(\theta|h_{t-1})} \left[I_{h_{t-1}}(\xi_t)\right] \label{eq:eig_rearranged}
\end{align}

Next, using Bayes rule  we similarly rearrange $I_{h_{t-1}}$:
\begin{align}
    I_{h_{t-1}}(\xi_t) &= \E_{p(\theta|h_{t-1}) p(y_t|\theta, \xi_t, h_{t -1})} \left[ \log\frac{ p(y_t|\theta, \xi_t, h_{t -1})}{p(y_t|\xi_t, h_{t-1})}\right] \\
    &= \E_{p(\theta|h_{t-1}) p(y_t|\theta, \xi_t, h_{t -1})} \left[ \log\frac{ p(\theta|y_t, \xi_t, h_{t -1})}{p(\theta| h_{t-1})} \right] \\
    &= \E_{p(\theta|h_{t-1}) p(y_t|\theta, \xi_t, h_{t -1})} \left[ \log p(\theta|y_t, \xi_t, h_{t-1}) \right] - \E_{p(\theta|h_{t-1})} \left[\log p(\theta| h_{t-1})\right]\\
    &=\E_{p(\theta|y_t, \xi_t, h_{t-1}) p(y_t|\xi_t, h_{t -1})} \left[ \log p(\theta|y_t, \xi_t, h_{t-1}) \right] - \E_{p(\theta|h_{t-1})} \left[\log p(\theta| h_{t-1})\right]
    \intertext{and noting $h_t = h_{t-1}\cup\{(\xi_t,y_t)\}$}
    &=\E_{p(\theta|h_{t}) p(y_t|\xi_t, h_{t -1})} \left[ \log p(\theta|h_{t}) \right] - \E_{p(\theta|h_{t-1})} \left[\log p(\theta| h_{t-1})\right] \\
    &= \E_{p(y_t|\xi_t, h_{t -1})} \Big[  \E_{p(\theta|h_{t})}\left[\log p(\theta|h_{t}) \right] - \E_{p(\theta|h_{t-1})} \left[\log p(\theta| h_{t-1})\right] \Big]
\end{align}
Substituting this in \eqref{eq:eig_rearranged}, noting that $\theta$ has already been integrated out, yields
\begin{align}
    \sEIG{\pi} &= \sum\nolimits_{t=1}^{T} \E_{p(h_{t-1}|\pi)} \E_{p(y_t|\xi_t, h_{t -1})} \Big[  \E_{p(\theta|h_{t})}\left[\log p(\theta|h_{t}) \right] - \E_{p(\theta|h_{t-1})} \left[\log p(\theta| h_{t-1})\right] \Big] \\
    &=  \sum\nolimits_{t=1}^{T} \E_{p(h_{t}|\pi)}  \Big[  \E_{p(\theta|h_{t})}\left[\log p(\theta|h_{t}) \right] - \E_{p(\theta|h_{t-1})} \left[\log p(\theta| h_{t-1})\right] \Big] \\
    &= \E_{p(h_{T}|\pi)} \left[ \sum\nolimits_{t=1}^{T} \E_{p(\theta|h_{t})}    \left[\log p(\theta|h_{t}) \right] - \E_{p(\theta|h_{t-1})} \left[\log p(\theta| h_{t-1})\right] \right],
\intertext{since we have a telescopic sum this simplifies to}
    &= \E_{p(h_{T}|\pi)} \Big[  \E_{p(\theta|h_{T})} \left[\log p(\theta| h_{T})\right] - \E_{p(\theta)} \left[\log p(\theta) \right]\Big]
\intertext{and finally we apply Bayes rule again to rewrite as}
    &= \E_{p(h_{T}|\pi)p(\theta|h_{T})}\Big[ \log p(\theta| h_{T}) - \E_{p(\theta)} \left[\log p(\theta) \right]\Big] \\
    &= \E_{\prior\histlik[\pi]}\left[ \log p(\theta| h_{T}) -\log p(\theta)\right] \\
    &= \E_{\prior\histlik[\pi]}\left[ \log \histlik[\pi] - \histmarg[\pi]\right]
\end{align}

\end{proof}

\subsection{Objective function as a mutual information}
We provide some additional discussion on the interpretation of $\mathcal{I}_T(\policy)$ in \eqref{eq:our_unified_EIG} as a mutual information. 
First, $\mathcal{I}_T(\policy)$ is not a conventional mutual information between $\theta$ and $h_T$. This is because, for the deterministic policy $\pi$ considered in this paper, the random variable $h_T$ does not have a density with respect to Lebesgue measure on $\Xi^T \times \mathcal{Y}^T$.
Indeed, since the designs $\xi_{1:T}$ are deterministic functions of the observations $y_{1:T}$, to express the sampling distribution of $h_T$ we would have to use Dirac deltas, specifically 
\begin{equation}
    \label{eq:deltajoint}
    p(y_{1:T},\xi_{1:T}|\theta,\pi) = \prod_{t=1}^T \delta_{\pi(h_{t-1})}(\xi_t) p(y_t|\theta,\xi_t,h_{t-1}).
\end{equation}
Due to the presence of Dirac deltas, this is not a conventional probability density, and hence we do not regard $\mathcal{I}_T(\pi)$ as the conventional mutual information between $\theta$ and $h_T$.

We note that we \emph{defined} $p(h_T|\theta,\pi)$ in Proposition~\ref{thm:unified} differently to $p(y_{1:T},\xi_{1:T}|\theta,\pi)$ in \eqref{eq:deltajoint}. Specifically, our definition 
\begin{equation}
    p(h_T | \theta, \pi) = \prod_{t=1}^T p(y_t | \theta, \xi_t, h_{t-1})
\end{equation}

only involves probability densities for $y_{1:T}$, meaning that our $p(h_T|\theta,\pi)$ is a well-defined probability density on $\mathcal{Y}^T$.
Formally, we can treat the designs $\xi_t$, not as additional random variables, but as part of the density for $y_{1:T}$. Indeed, since the policy $\pi$ is deterministic, it is possible to reconstruct $h_{t-1}$ and $\xi_t$ from $y_{1:t-1}$ and $\pi$, so we could write $p(y_t|\theta,y_{1:t-1},\pi) \coloneqq p(y_t|\theta,\xi_t,h_{t-1})$. In this formulation, only $y_{1:T}$ are regarded as random variables. This provides a formal justification for the form of $p(h_T|\theta,\pi)$ that we give in Proposition~\ref{thm:unified}.
In this setting, we could formally identify $\mathcal{I}_T(\pi)$ as the mutual information between $\theta$ and $y_{1:T}$.

However, it is helpful to think of $\mathcal{I}_T(\pi)$ as a mutual information between $\theta$ and $h_T$, because this naturally leads to critics that have access to $\theta$ and $h_T$, rather than $\theta$ and $y_{1:T}$. This way of thinking also connects naturally to the case of stochastic policies, which we now discuss.

If we consider additional noise in the design process so that designs are no longer a deterministic function of past data, then $\mathcal{I}_T(\policy)$ is the mutual information between $\theta$ and $h_T$.
In this case, we introduce an additional likelihood for designs $p(\xi|\pi, h)$, leading to the overall sampling distribution for the data 
\begin{equation}
    p(h_T|\theta,\pi) = \prod_{t=1}^T p(\xi_t|\pi,h_{t-1}) p(y_t|\theta,\xi_t,h_{t-1}).
\end{equation}
Unlike in the deterministic case, this is valid probability density on $\Xi^T \times \mathcal{Y}^T$.
If we now consider the mutual information between $\theta$ and $h_T$ for a fixed policy $\pi$ we have
\begin{align}
    \label{eq:thetahtmi}
    I(\theta,h_T) & = \E_{p(\theta)p(h_T|\theta,\pi)}\left[ \log \frac{\prod_{t=1}^T p(\xi_t|\pi,h_{t-1}) p(y_t|\theta,\xi_t,h_{t-1})}{\int_\Theta p(\theta) \prod_{t=1}^T p(\xi_t|\pi,h_{t-1}) p(y_t|\theta,\xi_t,h_{t-1}) \ d\theta} \right] \\
    & = \E_{p(\theta)p(h_T|\theta,\pi)}\left[ \log \frac{\prod_{t=1}^T p(\xi_t|\pi,h_{t-1}) \prod_{t=1}^T p(y_t|\theta,\xi_t,h_{t-1})}{\prod_{t=1}^T p(\xi_t|\pi,h_{t-1}) \int_\Theta p(\theta) \prod_{t=1}^T  p(y_t|\theta,\xi_t,h_{t-1}) \ d\theta} \right] \\
    & = \E_{p(\theta)p(h_T|\theta,\pi)}\left[ \log \frac{\prod_{t=1}^T p(y_t|\theta,\xi_t,h_{t-1})}{\int_\Theta p(\theta) \prod_{t=1}^T  p(y_t|\theta,\xi_t,h_{t-1}) \ d\theta} \right] 
\end{align}
noticing that the design likelihood terms cancel out in the integrand, and we reduce to the same integrand given in Proposition~\ref{thm:unified}.
Even when the policy is stochastic, the integrand in $I(\theta,h_T)$ only involves terms of the form $p(y_t|\theta,\xi_t,h_{t-1})$, and the likelihood of the design process completely cancels. 
Thus, the stochasticity of the designs is only present in the sampling distribution $p(h_T|\theta,\pi)$.
We therefore see that, as we consider the limiting case of $p(\xi|\pi,h)$ as it approaches a deterministic policy, only the sampling distribution of designs in $I(\theta,h_T)$ changes, with the integrand remaining the same. 
Under mild assumptions, then, the mutual information between $\theta$ and $h_T$ approaches $\mathcal{I}_T(\pi)$ in this limit.

\subsection{NWJ and InfoNCE bounds}
The next two propositions show that the two bounds---NWJ and InfoNCE---can be applied to the policy-based adaptive BOED setting.

\seqboundsnwj*

\begin{proof}
Let $\pi:\mathcal{H}^* \rightarrow \Xi$ be any (deterministic) policy taking histories $h_{t}$ as inputs and returning a design $\xi$ as output,
$U: \mathcal{H}^{T}\times \Theta \rightarrow \R$ be any function and define $g(h_T, \theta) \coloneqq \frac{\exp(U(h_T, \theta))}{\E_{\histmarg[\policy]} \left[\exp(U(h_T, \theta))\right]}$.

First, we multiply the numerator and denominator of the unified objective~\eqref{eq:our_unified_EIG} by $g(h_T, \theta)>0$
\begin{align}
    \mathcal{I}_T(\policy) &= \E_{\prior\histlik[\policy]} \left[ \log \frac{\histlik[\policy]}{\histmarg[\policy] } \right] \\
     & = \E_{\prior\histlik[\policy]} \log \left[ \frac{\histlik[\policy]}{\histmarg[\policy] } \frac{ g(h_T, \theta)}{ g(h_T, \theta)} \right] \\
     & =  \E_{\prior\histlik[\policy]} \left[ \log g(h_T, \theta) \right]  + \E_{\prior\histlik[\policy]}  \left[  \log \frac{\histlik[\policy]}{ \histmarg[\policy]  g(h_T, \theta) }  \right] 
\end{align}
Next, note that the second term is a KL divergence between two distributions
\begin{align}
    \E_{\prior\histlik[\policy]}  \left[  \log \frac{\histlik[\policy]}{ \histmarg[\policy]  g(h_T, \theta) }  \right] &= \E_{\prior\histlik[\policy]}  \left[  \log \frac{\prior \histlik[\policy]}{ \prior \histmarg[\policy]  g(h_T, \theta) }  \right] \\
    & = KL(\prior \histlik[\policy] || \hat{p}(h_T, \theta) ) \geq 0
\end{align}
where $\hat{p}(h_T, \theta)=\prior \histmarg[\policy]  g(h_T, \theta)$ is a valid distribution since
\begin{align}
    \int \prior \histmarg[\policy]  g(h_T, \theta) d\theta d h_T  &=  \E_{\prior \histmarg[\policy]} \frac{\exp(U(h_T, \theta))}{\E_{\histmarg[\policy]} \left[\exp(U(h_T, \theta))\right]} \\
    & = \E_{\prior} 1 =1.
\end{align}

Therefore, we have 
\begin{align}
    \mathcal{I}_T(\policy) & \geq  \E_{\prior\histlik[\policy]} \left[ \log g(h_T, \theta) \right] \\
    &= \E_{\prior\histlik[\policy]} [U(h_T, \theta) - \log \E_{\histmarg[\policy]} \exp(U(h_T, \theta)) ] \\
    &= \E_{\prior\histlik[\policy]} [U(h_T, \theta)] - \E_{\prior} \left[ \log \E_{\histmarg[\policy]} \exp(U(h_T, \theta))\right] \\
    \intertext{Now using the inequality $\log x \le e^{-1}x$} 
    & \geq \E_{\prior\histlik[\policy]} [U(h_T, \theta)] - e^{-1}\E_{\prior \histmarg[\policy]} \left[ \exp(U(h_T, \theta))\right] \\
    &= \mathcal{L}_T^{NWJ}(\policy, U)
\end{align}

Finally, substituting $U^*(h_T, \theta) = \log\frac{\histlik[\policy]}{\histmarg[\policy]} + 1$ in the bound we get
\begin{align}
     \mathcal{L}_T^{NWJ}(\policy, U^*) &= \E_{\prior\histlik[\policy]} \left[ \log\frac{\histlik[\policy]}{\histmarg[\policy]} + 1\right] - e^{-1}\E_{\prior \histmarg[\policy]} \left[  \frac{\histlik[\policy]}{\histmarg[\policy]}e^1 \right] \\
     &= \mathcal{I}_T(\policy) + 1 - \E_{\prior \histmarg[\policy]} \left[  \frac{\histlik[\policy]}{\histmarg[\policy]}\right]\\
     &=\mathcal{I}_T(\policy),
\end{align}
where we used $ \E_{\prior \histmarg[\policy]} \left[  \frac{\histlik[\policy]}{\histmarg[\policy]}\right] =  \E_{\prior \histlik[\policy]} \left[  1\right]=1$, establishing that the bound is tight for the optimal critic.

\end{proof}

\seqboundsinfo*
\begin{proof}
Let $\pi:\mathcal{H}^* \rightarrow \Xi$ be any (deterministic) policy taking histories $h_{t}$ as inputs and returning a design $\xi$ as output.
Choose any function (critic) $U: \mathcal{H}^{T}\times \Theta \rightarrow \R$. 

We introduce the shorthand
\begin{align}
    g(h_T, \theta_{0:L}) \coloneqq \frac{ \exp(U(h_T,\theta_0)) }{\frac{1}{L+1}\sum_{i=0}^L \exp(U(h_T,\theta_i)) } 
\end{align}
Starting with the definition of the unified objective from Equation~\eqref{eq:our_unified_EIG} we multiply its numerator and denominator by $g(h_T, \theta_{0:L})>0$ to get

\begin{align}
    \mathcal{I}_T(\policy) & = \E_{\priori{0}\histliki{0}[\policy]} \left[ \log \frac{\histliki{0}[\policy]}{\histmarg[\policy]} \right] \\
    \intertext{where $\priori{0}\histliki{0}[\policy] \equiv \prior\histlik[\policy]$}
    & = \E_{\priori{0}\histliki{0}[\policy]\priori{1:L}} \left[ \log \frac{\histliki{0}[\policy]}{\histmarg[\pi]} \right] \\
    & = \E_{\priori{0}\histliki{0}[\policy]\priori{1:L}}\left[ \log \frac{ \histliki{0}[\policy] g(h_T, \theta_{0:L}) }{\histmarg[\policy] g(h_T, \theta_{0:L})}  \right] \\
    \intertext{We next split the expectation into two terms one of which does not contain the unknown likelihoods and equals $\mathcal{L}^{NCE}$}
    \begin{split}
    &= \E_{\priori{0}\histliki{0}[\policy]\priori{1:L}} \left[ \log \frac{ \histliki{0}[\policy]}{ \histmarg[\policy] g(h_T, \theta_{0:L})} \right] \\ 
    &\quad + \E_{\priori{0}\histliki{0}[\policy]\priori{1:L}} \left[ \log g(h_T, \theta_{0:L}) \right]  \\
    &= \E_{\priori{0}\histliki{0}[\policy]\priori{1:L}} \left[ \log \frac{ \histliki{0}[\policy]}{ \histmarg[\policy] g(h_T, \theta_{0:L})} \right] + \mathcal{L}^{NCE}(\policy, U; L)
    \end{split} 
\end{align}

We now show that the first term is a KL divergence and hence non-negative. To see why, first write

\begin{align}
 &\E_{\priori{0}\histliki{0}[\policy]\priori{1:L}} \left[ \log \frac{ \histliki{0}[\policy]}{ \histmarg[\pi] g(h_T, \theta_{0:L})} \right] \\
  & = \E_{\priori{0}\histliki{0}[\policy]\priori{1:L}} \left[ \log \frac{ \priori{0} \histliki{0}[\policy] \priori{1:L}}{ \priori{0} \histmarg[\pi]\priori{1:L} g(h_T, \theta_{0:L})} \right] \\
  & = \E_{\priori{0}\histliki{0}[\policy]\priori{1:L}} \left[ \log \frac{ \priori{0} \histliki{0}[\policy] \priori{1:L}}{ \hat{p}(\theta_{0:L}, h_T | \pi) } \right] \\
  &= KL(\histliki{0}[\policy]\priori{0:L} || \hat{p}(\theta_{0:L}, h_T | \pi) ).
\end{align}

and $\hat{p}(\theta_{0:L}, h_T | \pi) $ is a valid distribution since

\begin{align}
    \int \hat{p}(\theta_{0:L}, h_T | \pi) d\theta_{0:L}d h_T & =  \int \priori{0} \histmarg[\pi]\priori{1:L} g(h_T, \theta_{0:L}) d\theta_{0:L}d h_T \\
    & = \E_{\priori{0}\histliki{0}[\policy]\priori{1:L}} \left[ \frac{ \exp(U(h_T,\theta_0)) }{\frac{1}{L+1}\sum_{i=0}^L \exp(U(h_T,\theta_i)) }  \right], \\
    \intertext{because of the symmetry $\theta_0 \overset{d}{=} \theta_j\;\forall j=1,\dots,L$}
    & = \frac{1}{L+1} \E_{\priori{0}\histliki{0}[\policy]\priori{1:L}} \left[ \frac{ \sum_{j=0}^L \exp(U(h_T,\theta_j)) }{\frac{1}{L+1}\sum_{i=0}^L \exp(U(h_T,\theta_i)) } \right] \\
    & = 1.
\end{align} 

Thus we have established
\begin{align}
    \mathcal{I}_T(\policy) = KL(\histliki{0}[\policy]\priori{0:L} || \hat{p}(\theta_{0:L}, h_T | \pi) )+ \mathcal{L}_T^{NCE}(\policy, U; L) \geq \mathcal{L}_T^{NCE}(\policy, U; L).
\end{align}

Next, substituting $U^*(h_T, \theta)=\log \histlik[\policy] + c(h_T)$ in the definition of $\mathcal{L}^{NCE}(\policy, U; L)$ we obtain
\begin{align}
    \mathcal{L}_T^{NCE}(\policy, U^*; L) &= \E_{\priori{0}\histliki{0}[\policy]\priori{1:L}} \left[ \frac{\histliki{0}[\policy] \exp(c(h_T))}{\frac{1}{L+1} \sum_{i=0}^L\histliki{i}[\policy] \exp(c(h_T)) } \right] \\
    & = \E_{\priori{0}\histliki{0}[\policy]\priori{1:L}} \left[ \frac{\histliki{0}[\policy]}{\frac{1}{L+1} \sum_{i=0}^L\histliki{i}[\policy]} \right],
\end{align}
which is exactly the sPCE bound~\eqref{eq:sPCE_objective}, which is monotonically increasing in $L$ and tight in the limit as $L \rightarrow \infty$ \citep[see][Theorem~2]{foster2021dad}.
\end{proof}

\subsection{A note on optimal critics}
An interesting feature of our approach is that, for both the InfoNCE and NWJ bounds, the optimal critics do not depend on the policy.
This is because we include the designs as explicit inputs to the critics.
Indeed, we have
\begin{align}
    U^*_\text{NCE}(h_T,\theta) &= \log \left( \prod_{t=1}^T p(y_t|\theta,\xi_t,h_{t-1}) \right) + c(h_T),\\
    U^*_\text{NWJ}(h_T,\theta) &= \log \left(\frac{\prod_{t=1}^T p(y_t|\theta,\xi_t,h_{t-1})}{\int_\Theta p(\theta) \prod_{t=1}^T p(y_t|\theta,\xi_t,h_{t-1}) \ d\theta} \right) + 1.
\end{align}
In previous work that utilized critics for gradient-based BOED~\citep{foster2020unified,kleinegesse2020mine}, it was typical to not treat the designs $\xi_{1:T}$ as an input to the critic, which renders the optimal critic implicitly dependent on the designs.
This makes more sense for static designs, for which the additional design input does not change.
Our approach avoids an implicit dependence between policy and optimal critic which may be beneficial for the joint optimization.

\section{Theoretical Comparison and Additional Bounds}
\label{sec:app_comp}

Recently, a number of studies have discussed the challenges of estimating mutual information, an in particular those associated with variational MI estimators
\citep{song2020limitations,mcallester2020formal,poole2018variational}.

Starting with the InfoNCE bound, it is trivial to show that the bound cannot exceed $\log(L+1)$, where $L$ is the number of contrastive samples used to approximate the marginal in the denominator. Indeed,
\begin{align}
    \mathcal{L}^{NCE}_T(\policy, \critic; L) &= \E_{\priori{0}\histliki{0}[\policy]}\E_{\priori{1:L}} \left[ \log \frac{\exp(\critic(h_T, \theta_0))} {\frac{1}{L+1}\sum_{i=0}^L \exp(\critic(h_T, \theta_i))} \right] \\
    &\leq \log(L+1) + \E_{\priori{0}\histliki{0}[\policy]}\E_{\priori{1:L}} \left[ \log \frac{\exp(\critic(h_T, \theta_0))} { \exp(\critic(h_T, \theta_0))} \right] \\
    & = \log(L+1)
\end{align}

This means that the corresponding Monte Carlo estimator will be highly biases whenever the true mutual information exceeds $\log(L+1)$, regardless of whether we have access to the optimal critic or not. This high bias estimator, however, comes with low variance \citep[see e.g.][for discussion]{poole2018variational}. With the optimal critic we would require exponential (in the MI) number of samples to accurately estimate the true mutual information. 

It might appear at first  that the NWJ bound might offer a better trade-off between bias and variance.
Recall from the proof of Proposition~\ref{thm:seqbounds_nwj}, we have for the \emph{optimal} critic 
\begin{align}
    \mathcal{L}_T^{NWJ}(\policy, U^*) &= \E_{\prior\histlik[\policy]} \left[ \log\frac{\histlik[\policy]}{\histmarg[\policy]}\right] + 1 - e^{-1}\E_{\prior \histmarg[\policy]} \left[  \frac{\histlik[\policy]}{\histmarg[\policy]}e^1 \right], 
    \intertext{of which we form a Monte carlo estimate using $N$ ($M$) samples for the first (second) term, respectively}
    &\approx \frac{1}{N}\sum_{n=1}^N \log \frac{ p(h_{T,n}|\theta_n, \pi)  }{ p(h_{T,n}|\pi)} + \left(1 -\frac{1}{M} \sum_{m=1}^M \log \frac{ p(h_{T,m}|\theta_m, \pi)  }{ p(h_{T,m}|\pi)} \right),
\end{align}
where $\theta_n, h_{T,n} \sim \prior\histlik[\policy]$ are samples from the joint distribution and $\theta_m, h_{T,m} \sim \prior\histmarg[\policy]$ are samples from the product of marginals. 
The first term is a Monte Carlo estimate of the mutual information, while the second has mean zero, meaning that this estimator is unbiased. 
The second term, however has variance which grows exponentially with the value of the (true) mutual information \citep[see Theorem 2 in][]{song2020limitations}. 
What this means is that even with an optimal critic, we will need an exponential (in the MI) number of samples to control the variance of the NWJ estimator. 
One might then hope that the variance can be reduced when using a sub-optimal critic at the cost of introducing some (hopefully small) bias. Unfortunately, according to a recent result \citep[see Theorems 3.1 and 4.1 in][and the discussion therein]{mcallester2020formal}, it is not possible to guarantee that a likelihood-free lower bound on the mutual information can exceed $\log(N)$. 
Indeed, the authors show theoretically and empirically that all high-confidence distribution-free lower bounds on the mutual information require exponential (in the the MI) number of samples.

Constructing a better lower bound on the mutual information---one that does not need exponential number of samples---therefore, requires us to make additional assumptions. 
Foster et al. \citep{foster2021dad} propose one such bound, namely the sequential Adaptive Constrative Estimation (sACE). The sACE bound introduces a proposal distribution $q(\theta; h_T)$, which aims to approximate the posterior $p(\theta|h_T)$. 
Since implicit models were not the focus of the work in \citep{foster2021dad} the proposed bound, relies on analytically available likelihood. 
The following proposition shows we can derive a likelihood-free version of the sACE bound. 

\begin{restatable}[Sequential Likelihood-free ACE]{proposition}{likelihoodfreeace}
	\label{thm:likelihoodfreeace}
	For a design function $\pi$, a critic function $U$, a number of contrastive samples $L\ge 1$, and a proposal $q(\theta;h_T)$, we have the sequential Likelihood-free Adaptive Contrastive Estimation (sLACE) lower bound
	\begin{equation}
	\mathcal{L}^\text{sLACE}_T(\pi, U, q;L) \coloneqq \E_{\priori{0}\histliki{0}[\pi]q(\theta_{1:L};h_T)} \left[ \log \frac{U(h_T, \theta_{0} )}{\frac{1}{L+1}\sum_{\ell=0}^L \frac{U(h_T, \theta_{\ell})p(\theta_\ell) }{q(\theta_\ell;h_T)}} \right] \leq \mathcal{I}_T(\pi). 
	\label{eq:sACE}
	\end{equation}
	The bound is tight as $L \rightarrow \infty$ for the optimal critic $U^*(h_T, \theta) = \log p(h_T | \theta, \pi) +c(h_T)$, where $c(h_T)$ is arbitrary.
	In addition, if $q(\theta;h_T)=p(\theta|h_T)$, the bound is tight for the optimal critic $U^*(h_T, \theta)$ with any $L\geq 0$.
\end{restatable}
\begin{proof}
The proof follows similar arguments to the ones in Propositions~\ref{thm:seqbounds_nwj} and~\ref{thm:seqbounds_info}. 
First let 
\begin{equation}
    g(h_T, \theta_{0:L}) \coloneqq \frac{U(h_T, \theta_{0} )}{\frac{1}{L+1}\sum_{\ell=0}^L \frac{U(h_T, \theta_{\ell})p(\theta_\ell) }{q(\theta_\ell;h_T)}}
\end{equation}

Starting with the definition of the EIG:
\begin{align}
    \mathcal{I}_T(\policy) &= \E_{\priori{0}\histliki{0}[\policy]} \left[ \log \frac{\histliki{0}[\policy]}{\histmarg[\policy] } \right] 
    \intertext{since $q(\theta_{i};h_T)$ is a valid density}
      &= \E_{\priori{0}\histliki{0}[\policy]q(\theta_{1:L};h_T)} \left[ \log \frac{\histliki{0}[\policy]}{\histmarg[\policy] } \right] \\
    \intertext{multiplying its numerator and denominator inside the log by $g(h_T, \theta_{0:L}) > 0$}
    &= \E_{\priori{0}\histliki{0}[\policy]q(\theta_{1:L};h_T)} \left[ \log \frac{\histliki{0}[\policy] g(h_T, \theta_{0:L})}{\histmarg[\policy] g(h_T, \theta_{0:L})}   \right]\\
    \begin{split}
    &= \E_{\priori{0}\histliki{0}[\policy]q(\theta_{1:L};h_T)} \left[ \log g(h_T, \theta_{0:L}) \right] \\
    & \quad + \E_{\priori{0}\histliki{0}[\policy]q(\theta_{1:L};h_T)} \left[ \log \frac{\histliki{0}[\policy]}{\histmarg[\policy] g(h_T, \theta_{0:L})}   \right] 
    \end{split}
\end{align}
The first term is exactly the sLACE bound, $\mathcal{L}^\text{sLACE}_T(\pi, U, q;L)$. We now show that the second term is a KL divergence between two distribitions and hence non-negative. To see this
\begin{align}
    &\E_{\priori{0}\histliki{0}[\policy]q(\theta_{1:L};h_T)} \left[ \log \frac{\histliki{0}[\policy]}{\histmarg[\policy] g(h_T, \theta_{0:L})} \right] \\
    &= \E_{\priori{0}\histliki{0}[\policy]q(\theta_{1:L};h_T)} \left[ \log \frac{ \priori{0} \histliki{0}[\policy] q(\theta_{1:L};h_T)}{\histmarg[\policy] g(h_T, \theta_{0:L}) \priori{0}q(\theta_{1:L};h_T) } \right] \\
    &= KL(\priori{0} \histliki{0}[\policy] q(\theta_{1:L};h_T)||\hat{p}(h_T, \theta_{0:L})),
\end{align}
since $\hat{p}(h_T, \theta_{0:L}) \coloneqq \histmarg[\policy] g(h_T, \theta_{0:L}) \priori{0}q(\theta_{1:L};h_T)$ is a valid density. Indeed:
\begin{align}
    \int \hat{p}(h_T, \theta_{0:L}) d h_T d\theta_{0:L} &= \E_{q(\theta_{1:L};h_T)\histmarg[\policy]} \left[ \priori{0} g(h_T, \theta_{0:L}) \right] \\
    &= \E_{q(\theta_{1:L};h_T)\histmarg[\policy]} \left[ \priori{0}\frac{U(h_T, \theta_{0} )}{\frac{1}{L+1}\sum_{\ell=0}^L \frac{U(h_T, \theta_{\ell})p(\theta_\ell) }{q(\theta_\ell;h_T)}} \right] \\
    &= \E_{q(\theta_{0:L};h_T)\histmarg[\policy]}\left[\frac{\frac{U(h_T, \theta_{0}) \priori{0}}{q(\theta_0;h_T)} }{\frac{1}{L+1}\sum_{\ell=0}^L \frac{U(h_T, \theta_{\ell})p(\theta_\ell) }{q(\theta_\ell;h_T)}} \right]
    \intertext{by symmetry}
    &= \E_{q(\theta_{0:L};h_T)\histmarg[\policy]}\left[\frac{ \frac{1}{L+1}\sum_{\ell=0}^L \frac{U(h_T, \theta_{\ell}) \priori{\ell}}{q(\theta_\ell;h_T)} }{\frac{1}{L+1}\sum_{\ell=0}^L \frac{U(h_T, \theta_{\ell})p(\theta_\ell) }{q(\theta_\ell;h_T)}} \right] \\
    &=1.
\end{align}

With the optimal critic we recover the sACE bound from \citep{foster2021dad}, which  under mild conditions converges to the mutual information $\mathcal{I}_T(\policy)$. To see that start by writing
\begin{align}
    \mathcal{L}^\text{sLACE}_T(\pi, U^*, q;L) &=  \E_{\priori{0}\histliki{0}[\pi]q(\theta_{1:L};h_T)} \left[\log \frac{p\left(h_{T} | \theta_{0}, \pi\right)}{\frac{1}{L+1} \sum_{\ell=0}^{L} \frac{p\left(h_{T} | \theta_{\ell}, \pi\right) p\left(\theta_{\ell}\right)}{q\left(\theta_{\ell} ; h_{T}\right)}}\right].
\end{align}
The denominator is a consistent estimator of the marginal, provided that each term in the sum is bounded, and so by the Strong Law of Large Numbers we have 
\begin{align}
    \frac{1}{L+1}\sum_{\ell=0}^{L} \frac{p\left(h_{T} | \theta_{\ell}, \pi\right) p\left(\theta_{\ell}\right)}{q\left(\theta_{\ell} ; h_{T}\right)} \rightarrow p(h_T|\pi)\;a.s. \text{\quad as } L\rightarrow\infty,
\end{align}
which establishes point-wise convergence of the integrand to $p(h_T|\theta_0, \pi)/p(h_T|\pi)$. We can apply Bounded convergence theorem to establish $\mathcal{L}^\text{sACE}_T(\pi, U^*, q;L) \rightarrow \mathcal{I}_T(\pi)$ as $L\rightarrow\infty$.

If in addition $q(\theta;h_T)=p(\theta|h_T)$ we have by Bayes rule:
\begin{align}
    \mathcal{L}^\text{sLACE}_T(\pi, U^*, q;L) &= \E_{\priori{0}\histliki{0}[\pi]p(\theta_{1:L}|h_T)} \left[ \log \frac{p(h_T| \theta_{0},\pi )} {\frac{1}{L+1}\sum_{\ell=0}^L \frac{p(h_T| \theta_{\ell},\pi)p(\theta_\ell) }{p(\theta_\ell|h_T)}} \right] \\
    &= \E_{\priori{0}\histliki{0}[\pi]p(\theta_{1:L}|h_T)} \left[ \log \frac{p(h_T| \theta_{0},\pi )} {\frac{1}{L+1}\sum_{\ell=0}^L p(h_T|\pi)} \right] \\
    &= \E_{\priori{0}\histliki{0}[\pi]} \left[ \log \frac{p(h_T| \theta_{0},\pi )} {p(h_T|\pi)} \right] \\
    &= \mathcal{I}_T(\pi) \quad \forall L\geq0.
\end{align}
\end{proof}

In practice, we parameterize the policy, the critic and the density of the proposal distribution by neural networks $\pi_\phi$, $U_\psi$ and $q_\zeta$ and optimize $\mathcal{L}^\text{sLACE}_T$ with respect to the parameters of these networks, $\phi, \psi$ and $\zeta$ with SGA. As before, optimizing with respect to $\phi$ improves the quality of the designs, proposed by the policy, whilst optimizing with respect to $\psi$ and $\zeta$ tightens the bound. If the parametric density $q_\zeta$ and the critic $U_\psi$ are expressive enough, so that we can recover the optimal critic and the true posterior, then the bound is tight for any number of contrastive samples $L$. If, on the other hand, we fix $q_\zeta(\theta; h_T) =\prior$ instead of training it, then we recover the InfoNCE bound. Therefore, as long as $q_\zeta$ approximates the posterior better than the prior, then even an imperfect proposal $q_\zeta$ can benefit training. 

In addition to introducing another set of optimizable parameters, $\zeta$, the sLACE bound assumes that we know the prior $\prior$ and can evaluate its density.

\section{Neural architecture}
\label{sec:appendix:arch}

\subsection{Permutation invariance of the critic for exchangeable experiments}
We show that if the BOED problem is exchangeable then the critic function $U$ should be permutation-invariant.

\begin{restatable}[Permutation invariance]{proposition}{invariance}
	\label{thm:invariance}
	 Let $\sigma$ be a permutation acting on a history $h_T^1$ yielding $h_T^2 = \{(\xi_{\sigma(i)}, y_{\sigma(i)})\}_{i=1}^T$. If the data generating process is conditionally independent of its past given $\theta$, then the optimal critics for both \eqref{eq:seqNWJ_objective} and \eqref{eq:seqNCE_objective} are invariant under permutations of the history, i.e. 
	 \begin{equation}
	     \prior \prod_{t=1}^T p(y_t|\theta, \xi_t(h_{t-1}), h_{t-1}) = \prior\prod_{t=1}^T \obsliki{t} \implies U^*(h_T^1, \theta) = U^*(h_T^2, \theta). 
	     \label{eq:indep_proposition}
	 \end{equation}
\end{restatable}
\begin{proof}
This is a direct consequence from the form of the optimal critics. To see this formally, let $h_T^1$ be a history and $h_T^2$ be a permutation of it. 

Starting with the InfoNCE bound we have
\begin{align}
    U^*_{\text{NCE}}(h_T^1,\theta) &= \log p(h_T^1| \theta, \pi) + c(h_T^1) \\
    &= \log \prod_{t=1}^T \obsliki{t} + c(\{(\xi_t, y_t)\}_{t=1}^T) \\
    \intertext{since $c(h_T)$ is arbitrary, we can choose it to be permutation invariant}
    &= \log \prod_{t=1}^T \obsliki{\sigma(t)} + c(\{(\xi_{\sigma(t)}, y_{\sigma(t)})\}_{t=1}^T) \\
    &= \log p(h_T^2| \theta, \pi) + c(h_T^2) \\
    &= U^*_{\text{NCE}}(h_T^2,\theta)
\end{align}

Similarly, for the optimal critic of the NWJ bound we have
\begin{align}
    U^*_{\text{NWJ}}(h_T^1,\theta) &= \log \frac{p(h_T^1| \theta, \pi)}{p(h_T^1|\pi)} +1 \\
    &= \log \frac{\prod_{t=1}^T \obsliki{t} }{\E_{\prior}\left[\prod_{s=1}^T \obsliki{s} \right]} +1 \\
    &= \log \frac{\prod_{t=1}^T \obsliki{\sigma(t)} }{\E_{\prior}\left[\prod_{s=1}^T \obsliki{\sigma(s)} \right]} +1 \\
    &= \log \frac{p(h_T^2| \theta, \pi)}{p(h_T^2|\pi)} +1 = U^*_{\text{NWJ}}(h_T^2,\theta).
\end{align}
\end{proof}

To the best of our knowledge, we are the first to propose a critic architecture that is tailored to BOED problems with exchangeable models. Previous work in the static BOED setting, where MI information objective is optimized with variational lower bounds and thus require the training of critics \citep[e.g.][]{kleinegesse2020mine, zhang2021sagabed}, did not discuss what an appropriate critic architecture might be. In particular, in all experiments \citep{kleinegesse2020mine, zhang2021sagabed} use a generic architecture for both exchangeable and non-exchangeable problems. 
An expressive enough generic architecture should be able to obtain the optimal critic, and thus achieve a tight bound, however, the optimisation process will be considerably more difficult as the network needs to learn this key invariance structure. We therefore recommend using permutation invariant architectures whenever the model is exchangeable, especially if achieving tight bounds (and therefore learning an optimal critic) is of importance.

\subsection{Further details on the history encoder}

Figure~\ref{fig:history_encoder_architectures} shows the history encoders we use in the policy network $\designnet$ and the critic network $\criticnet$. First, we encode the individual design-outcome pairs, $(\xi_t, y_t)$, with an MLP, which gives us a vector of representations $r_t \in \R^m$, where $m$ is the encoding dimension we have selected. The representations $\{r_i\}_{i=1}^t$ are row-stacked into a matrix $R$ of dimension $t\times m$, which we then aggregate back to a vector of size $m$ by an appropriate layer(s).

When conditional independence of the experiments holds, we apply 8-head self-attention, based on the Image Transformer \cite{parmar2018image} and as implemented by \cite{dubois2020npf}. Applying self-attention leaves the dimension of the matrix $R$ unchanged.  We then apply sum-pooling across time $t$, which gives us the final encoding vector $E \in \R^m$.  

When experiments are not conditionally independent, we pass the matrix $R$ though an LSTM with two hidden layers and hidden state of size $m$ (see the  \href{https://pytorch.org/docs/stable/generated/torch.nn.LSTM.html}{LSTM module in Pytorch} for more details). The LSTM returns hidden state vectors associated with the history $h_t$ for each $t$; we keep the last hidden state of the last layer, which is our final encoding vector $E \in R^m$.  

In both cases the resulting encoding $E$ is a vector of size $m$. It is passed through final fully connected "head" layers, which output either a design (in the case of the policy) or a vector (in the case of the critic). We train two separate history encoders---one for the design network $\designnet$ and one for the critic network $\criticnet$, although we note that all the weights except those in the head layers can be shared. 
 
\begin{figure}[t]

\begin{subfigure}[b]{0.54\textwidth}
\centering
\tikzset{every picture/.style={line width=0.75pt}} 

\begin{tikzpicture}[x=0.75pt,y=0.75pt,yscale=-1,xscale=1]

\draw  [fill={rgb, 255:red, 155; green, 201; blue, 175 }  ,fill opacity=0.7 ] (235,94.9) .. controls (235,89.16) and (239.66,84.5) .. (245.4,84.5) -- (276.6,84.5) .. controls (282.34,84.5) and (287,89.16) .. (287,94.9) -- (287,181.1) .. controls (287,186.84) and (282.34,191.5) .. (276.6,191.5) -- (245.4,191.5) .. controls (239.66,191.5) and (235,186.84) .. (235,181.1) -- cycle ;
\draw  [fill={rgb, 255:red, 155; green, 201; blue, 225 }  ,fill opacity=0.65 ] (304,91.3) .. controls (304,87.54) and (307.04,84.5) .. (310.8,84.5) -- (331.2,84.5) .. controls (334.96,84.5) and (338,87.54) .. (338,91.3) -- (338,184.7) .. controls (338,188.46) and (334.96,191.5) .. (331.2,191.5) -- (310.8,191.5) .. controls (307.04,191.5) and (304,188.46) .. (304,184.7) -- cycle ;
\draw  [fill={rgb, 255:red, 255; green, 255; blue, 255 }  ,fill opacity=0.37 ] (352.71,127.99) .. controls (352.71,124.56) and (355.49,121.79) .. (358.91,121.79) -- (378.51,121.79) .. controls (381.94,121.79) and (384.71,124.56) .. (384.71,127.99) -- (384.71,146.59) .. controls (384.71,150.01) and (381.94,152.79) .. (378.51,152.79) -- (358.91,152.79) .. controls (355.49,152.79) and (352.71,150.01) .. (352.71,146.59) -- cycle ;
\draw  [fill={rgb, 255:red, 245; green, 166; blue, 35 }  ,fill opacity=0.37 ] (108.71,93.79) .. controls (108.71,90.47) and (111.4,87.79) .. (114.71,87.79) -- (132.71,87.79) .. controls (136.03,87.79) and (138.71,90.47) .. (138.71,93.79) -- (138.71,111.79) .. controls (138.71,115.1) and (136.03,117.79) .. (132.71,117.79) -- (114.71,117.79) .. controls (111.4,117.79) and (108.71,115.1) .. (108.71,111.79) -- cycle ;
\draw  [fill={rgb, 255:red, 208; green, 2; blue, 27 }  ,fill opacity=0.34 ] (141.71,93.79) .. controls (141.71,90.47) and (144.4,87.79) .. (147.71,87.79) -- (165.71,87.79) .. controls (169.03,87.79) and (171.71,90.47) .. (171.71,93.79) -- (171.71,111.79) .. controls (171.71,115.1) and (169.03,117.79) .. (165.71,117.79) -- (147.71,117.79) .. controls (144.4,117.79) and (141.71,115.1) .. (141.71,111.79) -- cycle ;
\draw  [fill={rgb, 255:red, 245; green, 166; blue, 35 }  ,fill opacity=0.37 ] (108.71,163.79) .. controls (108.71,160.47) and (111.4,157.79) .. (114.71,157.79) -- (132.71,157.79) .. controls (136.03,157.79) and (138.71,160.47) .. (138.71,163.79) -- (138.71,181.79) .. controls (138.71,185.1) and (136.03,187.79) .. (132.71,187.79) -- (114.71,187.79) .. controls (111.4,187.79) and (108.71,185.1) .. (108.71,181.79) -- cycle ;
\draw  [fill={rgb, 255:red, 208; green, 2; blue, 27 }  ,fill opacity=0.34 ] (141.71,163.79) .. controls (141.71,160.47) and (144.4,157.79) .. (147.71,157.79) -- (165.71,157.79) .. controls (169.03,157.79) and (171.71,160.47) .. (171.71,163.79) -- (171.71,181.79) .. controls (171.71,185.1) and (169.03,187.79) .. (165.71,187.79) -- (147.71,187.79) .. controls (144.4,187.79) and (141.71,185.1) .. (141.71,181.79) -- cycle ;
\draw    (173,103.71) -- (187.71,103.71) ;
\draw [shift={(189.71,103.71)}, rotate = 180] [color={rgb, 255:red, 0; green, 0; blue, 0 }  ][line width=0.75]    (10.93,-3.29) .. controls (6.95,-1.4) and (3.31,-0.3) .. (0,0) .. controls (3.31,0.3) and (6.95,1.4) .. (10.93,3.29)   ;
\draw  [fill={rgb, 255:red, 19; green, 142; blue, 254 }  ,fill opacity=0.46 ] (188.71,93.8) .. controls (188.71,90.49) and (191.4,87.8) .. (194.71,87.8) -- (212.71,87.8) .. controls (216.03,87.8) and (218.71,90.49) .. (218.71,93.8) -- (218.71,111.8) .. controls (218.71,115.11) and (216.03,117.8) .. (212.71,117.8) -- (194.71,117.8) .. controls (191.4,117.8) and (188.71,115.11) .. (188.71,111.8) -- cycle ;
\draw  [fill={rgb, 255:red, 19; green, 142; blue, 254 }  ,fill opacity=0.46 ] (188.71,163.8) .. controls (188.71,160.49) and (191.4,157.8) .. (194.71,157.8) -- (212.71,157.8) .. controls (216.03,157.8) and (218.71,160.49) .. (218.71,163.8) -- (218.71,181.8) .. controls (218.71,185.11) and (216.03,187.8) .. (212.71,187.8) -- (194.71,187.8) .. controls (191.4,187.8) and (188.71,185.11) .. (188.71,181.8) -- cycle ;
\draw    (173,172.71) -- (187.71,172.71) ;
\draw [shift={(189.71,172.71)}, rotate = 180] [color={rgb, 255:red, 0; green, 0; blue, 0 }  ][line width=0.75]    (10.93,-3.29) .. controls (6.95,-1.4) and (3.31,-0.3) .. (0,0) .. controls (3.31,0.3) and (6.95,1.4) .. (10.93,3.29)   ;
\draw    (218,172.71) -- (232.71,172.71) ;
\draw [shift={(234.71,172.71)}, rotate = 180] [color={rgb, 255:red, 0; green, 0; blue, 0 }  ][line width=0.75]    (10.93,-3.29) .. controls (6.95,-1.4) and (3.31,-0.3) .. (0,0) .. controls (3.31,0.3) and (6.95,1.4) .. (10.93,3.29)   ;
\draw    (218.71,103.6) -- (233.43,103.6) ;
\draw [shift={(235.43,103.6)}, rotate = 180] [color={rgb, 255:red, 0; green, 0; blue, 0 }  ][line width=0.75]    (10.93,-3.29) .. controls (6.95,-1.4) and (3.31,-0.3) .. (0,0) .. controls (3.31,0.3) and (6.95,1.4) .. (10.93,3.29)   ;
\draw    (286.71,137.6) -- (301.43,137.6) ;
\draw [shift={(303.43,137.6)}, rotate = 180] [color={rgb, 255:red, 0; green, 0; blue, 0 }  ][line width=0.75]    (10.93,-3.29) .. controls (6.95,-1.4) and (3.31,-0.3) .. (0,0) .. controls (3.31,0.3) and (6.95,1.4) .. (10.93,3.29)   ;
\draw    (336.71,137.6) -- (351.43,137.6) ;
\draw [shift={(353.43,137.6)}, rotate = 180] [color={rgb, 255:red, 0; green, 0; blue, 0 }  ][line width=0.75]    (10.93,-3.29) .. controls (6.95,-1.4) and (3.31,-0.3) .. (0,0) .. controls (3.31,0.3) and (6.95,1.4) .. (10.93,3.29)   ;

\draw (261.33,135.57) node   [align=left] {\begin{minipage}[lt]{41.48pt}\setlength\topsep{0pt}
\begin{center}
Self-\\attention
\end{center}

\end{minipage}};
\draw (321.33,135.57) node   [align=left] {\begin{minipage}[lt]{26.86pt}\setlength\topsep{0pt}
\begin{center}
Sum-\\pool
\end{center}

\end{minipage}};
\draw (369.05,137.14) node   [align=left] {$\displaystyle E$};
\draw (125.05,103.14) node   [align=left] {$\displaystyle \xi _{1}$};
\draw (158.05,103.14) node   [align=left] {$\displaystyle y_{1}$};
\draw (125.05,173.14) node   [align=left] {$\displaystyle \xi _{t}$};
\draw (158.05,173.14) node   [align=left] {$\displaystyle y_{t}$};
\draw (204.05,103.14) node   [align=left] {$\displaystyle r_{1}$};
\draw (204.05,172.14) node   [align=left] {$\displaystyle r_{t}$};
\draw (145.86,139.96) node  [rotate=-90] [align=left] {$\displaystyle \dotsc $};
\draw (207.86,139.96) node  [rotate=-90] [align=left] {$\displaystyle \dotsc $};

\end{tikzpicture}

\caption{Exchangeable models}
\label{fig:app_arch_invariant}
\end{subfigure}
\hfill
\begin{subfigure}[b]{0.43\textwidth}
\centering
\tikzset{every picture/.style={line width=0.75pt}} 

\begin{tikzpicture}[x=0.75pt,y=0.75pt,yscale=-1,xscale=1]

\draw  [fill={rgb, 255:red, 155; green, 201; blue, 175 }  ,fill opacity=0.7 ] (231,92.3) .. controls (231,87.99) and (234.49,84.5) .. (238.8,84.5) -- (262.2,84.5) .. controls (266.51,84.5) and (270,87.99) .. (270,92.3) -- (270,183.7) .. controls (270,188.01) and (266.51,191.5) .. (262.2,191.5) -- (238.8,191.5) .. controls (234.49,191.5) and (231,188.01) .. (231,183.7) -- cycle ;
\draw  [fill={rgb, 255:red, 255; green, 255; blue, 255 }  ,fill opacity=0.37 ] (286.71,127.99) .. controls (286.71,124.56) and (289.49,121.79) .. (292.91,121.79) -- (312.51,121.79) .. controls (315.94,121.79) and (318.71,124.56) .. (318.71,127.99) -- (318.71,146.59) .. controls (318.71,150.01) and (315.94,152.79) .. (312.51,152.79) -- (292.91,152.79) .. controls (289.49,152.79) and (286.71,150.01) .. (286.71,146.59) -- cycle ;
\draw    (270,137.71) -- (284.71,137.71) ;
\draw [shift={(286.71,137.71)}, rotate = 180] [color={rgb, 255:red, 0; green, 0; blue, 0 }  ][line width=0.75]    (10.93,-3.29) .. controls (6.95,-1.4) and (3.31,-0.3) .. (0,0) .. controls (3.31,0.3) and (6.95,1.4) .. (10.93,3.29)   ;
\draw  [fill={rgb, 255:red, 245; green, 166; blue, 35 }  ,fill opacity=0.37 ] (104.71,93.79) .. controls (104.71,90.47) and (107.4,87.79) .. (110.71,87.79) -- (128.71,87.79) .. controls (132.03,87.79) and (134.71,90.47) .. (134.71,93.79) -- (134.71,111.79) .. controls (134.71,115.1) and (132.03,117.79) .. (128.71,117.79) -- (110.71,117.79) .. controls (107.4,117.79) and (104.71,115.1) .. (104.71,111.79) -- cycle ;
\draw  [fill={rgb, 255:red, 208; green, 2; blue, 27 }  ,fill opacity=0.34 ] (137.71,93.79) .. controls (137.71,90.47) and (140.4,87.79) .. (143.71,87.79) -- (161.71,87.79) .. controls (165.03,87.79) and (167.71,90.47) .. (167.71,93.79) -- (167.71,111.79) .. controls (167.71,115.1) and (165.03,117.79) .. (161.71,117.79) -- (143.71,117.79) .. controls (140.4,117.79) and (137.71,115.1) .. (137.71,111.79) -- cycle ;
\draw  [fill={rgb, 255:red, 245; green, 166; blue, 35 }  ,fill opacity=0.37 ] (104.71,163.79) .. controls (104.71,160.47) and (107.4,157.79) .. (110.71,157.79) -- (128.71,157.79) .. controls (132.03,157.79) and (134.71,160.47) .. (134.71,163.79) -- (134.71,181.79) .. controls (134.71,185.1) and (132.03,187.79) .. (128.71,187.79) -- (110.71,187.79) .. controls (107.4,187.79) and (104.71,185.1) .. (104.71,181.79) -- cycle ;
\draw  [fill={rgb, 255:red, 208; green, 2; blue, 27 }  ,fill opacity=0.34 ] (137.71,163.79) .. controls (137.71,160.47) and (140.4,157.79) .. (143.71,157.79) -- (161.71,157.79) .. controls (165.03,157.79) and (167.71,160.47) .. (167.71,163.79) -- (167.71,181.79) .. controls (167.71,185.1) and (165.03,187.79) .. (161.71,187.79) -- (143.71,187.79) .. controls (140.4,187.79) and (137.71,185.1) .. (137.71,181.79) -- cycle ;
\draw    (169,103.71) -- (183.71,103.71) ;
\draw [shift={(185.71,103.71)}, rotate = 180] [color={rgb, 255:red, 0; green, 0; blue, 0 }  ][line width=0.75]    (10.93,-3.29) .. controls (6.95,-1.4) and (3.31,-0.3) .. (0,0) .. controls (3.31,0.3) and (6.95,1.4) .. (10.93,3.29)   ;
\draw  [fill={rgb, 255:red, 19; green, 142; blue, 254 }  ,fill opacity=0.46 ] (184.71,93.8) .. controls (184.71,90.49) and (187.4,87.8) .. (190.71,87.8) -- (208.71,87.8) .. controls (212.03,87.8) and (214.71,90.49) .. (214.71,93.8) -- (214.71,111.8) .. controls (214.71,115.11) and (212.03,117.8) .. (208.71,117.8) -- (190.71,117.8) .. controls (187.4,117.8) and (184.71,115.11) .. (184.71,111.8) -- cycle ;
\draw  [fill={rgb, 255:red, 19; green, 142; blue, 254 }  ,fill opacity=0.46 ] (184.71,163.8) .. controls (184.71,160.49) and (187.4,157.8) .. (190.71,157.8) -- (208.71,157.8) .. controls (212.03,157.8) and (214.71,160.49) .. (214.71,163.8) -- (214.71,181.8) .. controls (214.71,185.11) and (212.03,187.8) .. (208.71,187.8) -- (190.71,187.8) .. controls (187.4,187.8) and (184.71,185.11) .. (184.71,181.8) -- cycle ;
\draw    (169,172.71) -- (183.71,172.71) ;
\draw [shift={(185.71,172.71)}, rotate = 180] [color={rgb, 255:red, 0; green, 0; blue, 0 }  ][line width=0.75]    (10.93,-3.29) .. controls (6.95,-1.4) and (3.31,-0.3) .. (0,0) .. controls (3.31,0.3) and (6.95,1.4) .. (10.93,3.29)   ;
\draw    (214,172.71) -- (228.71,172.71) ;
\draw [shift={(230.71,172.71)}, rotate = 180] [color={rgb, 255:red, 0; green, 0; blue, 0 }  ][line width=0.75]    (10.93,-3.29) .. controls (6.95,-1.4) and (3.31,-0.3) .. (0,0) .. controls (3.31,0.3) and (6.95,1.4) .. (10.93,3.29)   ;
\draw    (214.71,103.6) -- (229.43,103.6) ;
\draw [shift={(231.43,103.6)}, rotate = 180] [color={rgb, 255:red, 0; green, 0; blue, 0 }  ][line width=0.75]    (10.93,-3.29) .. controls (6.95,-1.4) and (3.31,-0.3) .. (0,0) .. controls (3.31,0.3) and (6.95,1.4) .. (10.93,3.29)   ;

\draw (251.33,135.57) node   [align=left] {\begin{minipage}[lt]{29.67pt}\setlength\topsep{0pt}
\begin{center}
LSTM
\end{center}

\end{minipage}};
\draw (303.05,137.14) node   [align=left] {$\displaystyle E$};
\draw (121.05,103.14) node   [align=left] {$\displaystyle \xi _{1}$};
\draw (154.05,103.14) node   [align=left] {$\displaystyle y_{1}$};
\draw (121.05,173.14) node   [align=left] {$\displaystyle \xi _{t}$};
\draw (154.05,173.14) node   [align=left] {$\displaystyle y_{t}$};
\draw (200.05,103.14) node   [align=left] {$\displaystyle r_{1}$};
\draw (200.05,172.14) node   [align=left] {$\displaystyle r_{t}$};
\draw (141.86,139.96) node  [rotate=-90] [align=left] {$\displaystyle \dotsc $};
\draw (203.86,139.96) node  [rotate=-90] [align=left] {$\displaystyle \dotsc $};

\end{tikzpicture}

\caption{Non-exchangeable models }
\label{fig:app_arch_lstm}
\end{subfigure}
\caption{History encoder architectures for different classes of models. When conditional independence of the experiments holds, we use self-attention, followed by sum-pooling, making the history encoder permutation invariant. When experiments are not conditionally independent we use LSTM  and only keep its last hidden state. We train two separate history encoders---one for the design network $\designnet$ and one for the critic network $\criticnet$, although we note that all the weights except those in the head layers can be shared. } 
\label{fig:history_encoder_architectures}
\end{figure}
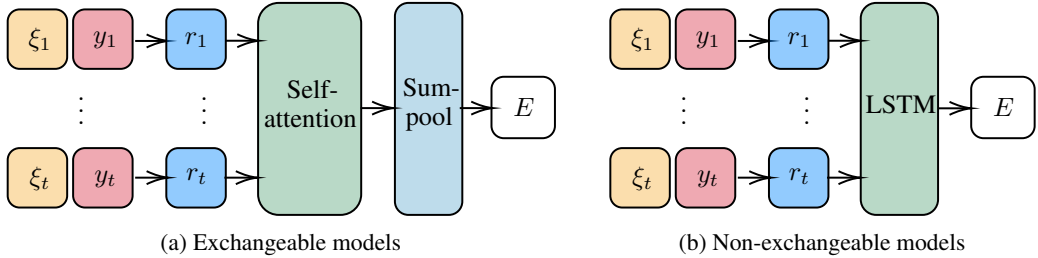

\section{Experiments}\label{sec:appendix_experiments}

\subsection{Computational resources}

All of the experiments were implemented in Python using open-source software. All estimators and models were implemented in PyTorch \citep{paszke2019pytorch} (BSD license) and Pyro \citep{pyro} (Apache License Version 2.0), whilst MlFlow \citep{zaharia2018accelerating} (Apache License Version 2.0) was used for experiment tracking and management. The self-attention architecture from \cite{dubois2020npf} was used to implement the self-attention mechanisms in the design and critic networks. For full details on package versions, environment set-up and commands for running the code, see instructions in the \texttt{README.md} file.

Experiments were ran on internal GPU clusters, consisting of GeForce RTX 3090 (24GB memory), GeForce RTX 2080 Ti (11GB memory) and GeForce GTX 1080 Ti GPUs (11GB memory).

The deployment-time of iDAD (Table~\ref{tab:locfin_T10_deployment_times}) was estimated on a lightweight CPU machine with the following  specifications
\begin{center}
\begin{tabular}{ll}
    Processor & 2.8 GHz Quad-Core Intel Core i7 \\
    Memory &  16 GB \\
    Operating system & macOS Big Sur v11.2.3
\end{tabular}
\end{center}

\subsection{CO2 Emission Related to Experiments}

Experiments were conducted using a private infrastructure, which has an estimated carbon efficiency of 0.432 kgCO$_2$eq/kWh. A cumulative of 160 hours of computation was performed on hardware of type RTX 2080 Ti (TDP of 250W), or similar. The training time of each experiment (including the baselines that require optimization), took on average between 1-3 GPU hours, depending on the number of experiments $T$. 

Total emissions are estimated to be 17.28 kgCO$_2$eq of which 0\% was directly offset.

Estimations were conducted using the \href{https://mlco2.github.io/impact#compute}{Machine Learning Impact calculator} presented in \cite{lacoste2019quantifying}.

\subsection{Traditional sequential BOED with variational posterior estimator}\label{sec:app_variational_baseline}
The variational posterior estimator from \citep{foster2019variational} is based on the Barbar-Agakov lower bound \citep{barber2003imalgorithm}, which takes the form
\begin{align}
    \mathcal{L}^{\text{post}}(\xi, q_\psi) = \E_{\prior \obslik}\left[ \log \frac{ q_\psi(\theta;y,\xi)}{p(\theta)} \right] \leq  \mathcal{I}(\xi),
    \label{eq:ba_bound}
\end{align}
where $q_\psi$ is any normalized distribution over the parameters $\theta$. The bound is tight when $q_\psi(\theta; y, \xi) = p(\theta|y, \xi)$, i.e. if we can recover the true posterior. We assume mean-field variational family and optimize the parameters $\psi$ by maximizing the bound~\eqref{eq:ba_bound} using stochastic gradient schemes. Simultaneously we optimize the bound with respect to the design variable $\xi$ to select the optimal design $\xi^*$. At the inference stage, denoting by $y^*$ the outcome of experiment $\xi^*$, we obtain an approximate posterior by evaluating $q_\psi(\theta; y^*, \xi^*)$, i.e. we reuse the learnt variational posterior. We repeat this process at each stage of the experiments by substituting the the approximate posterior, $q_\psi(\theta; y^*, \xi^*)$,  as the prior in~\eqref{eq:ba_bound}.

\subsection{Location Finding}\label{sec:appendix_experiments_locfin}

In this experiment we have $K$ hidden objects (\textit{sources}) in $\R^2$ and we wish to learn their locations, $\theta= \{\theta_1, \dots, \theta_K\}$. The number of sources, $K$, is assumed to be known. Each source emits a signal with intensity obeying the inverse-square law. Put differently, if a source is located at $\theta_k$ and we perform a measurement at a point $\xi$, the signal strength emitted from that source only will be proportional to $\frac{1}{\|\theta_k-\xi\|^2}$. 
The total intensity at location $\xi$, emitted from all $K$ sources, is a superposition of the individual ones
\begin{equation}
\mu(\theta, \xi) = b + \sum_{k=1}^K \frac{\alpha_k}{m+ \|\theta_k-\xi\|^2},
\end{equation}
where $\alpha_k$ can be known constants or random variables,  $b>0$ is a constant background signal and $m$ is a constant, controlling the maximum  signal. 

We  place a standard normal prior on each of the location parameters $\theta_k$ and we observe the log-total intensity with some Gaussian noise. We therefore have the following prior and likelihood:
\begin{equation}
\theta_k \iid \mathcal{N}(0_d, I_d)\quad \log y \mid \theta, \quad \xi \sim \mathcal{N}(\log \mu(\theta, \xi), \sigma^2)
\end{equation}

\subsubsection{Training details}\label{sec:app_locfin_traindeets} 
All our experiments are performed with the following model hyperparameters

\begin{center}
\begin{tabular}{lr}
    Parameter & Value \\
    \midrule
    Number of sources, K & 2 \\
    $\alpha_k$ & 1 $\forall k$ \\
    Max signal, $m$ & $10^{-4}$ \\
    Base signal, $b$ & $10^{-1}$ \\
    Observation noise scale,  $\sigma$ & 0.5
\end{tabular}
\end{center}

The architecture of the design network $\designnet$ used in Table~\ref{tab:locfin_T10_highd} and all its hyperparameters are in the following tables. For the encoder of the design-outcome pairs we used the following:

\begin{center}
\begin{tabular}{llrrr}
    Layer & Description & iDAD, InfoNCE & iDAD, NWJ & Activation \\
    \midrule
    Input & $\xi, y$ & 3 & 3 & - \\
    H1 & Fully connected & 64 & 64 & ReLU \\
    H2 & Fully connected & 512 & 512 & ReLU \\
    Output & Fully connected & 64 & 64 & - \\
    Attention & 8 heads & 64& 64 & - 
\end{tabular}
\end{center}

The output of the encoder, $R(h_t)$, is fed into an emitter network, for which we used the following:
\begin{center}
\begin{tabular}{llrrr}
    Layer & Description & iDAD, InfoNCE & iDAD, NWJ & Activation \\
    \midrule
    Input & $R(h_t)$ & 64 & 64 & - \\
    H1 & Fully connected & 256 & 256 & ReLU \\
    H2 & Fully connected & 64 & 64 & ReLU \\
    Output & Fully connected & 2 & 2 & - \\
\end{tabular}
\end{center}

The architecture of the critic network $\criticnet$ used in Table~\ref{tab:locfin_T10_highd} and all its hyperparameters are in the tables that follow. First, the encoder network of the latent variables is:
\begin{center}
\begin{tabular}{llrrr}
    Layer & Description & iDAD, InfoNCE & iDAD, NWJ & Activation \\
    \midrule 
    Input & $\theta$ & 4 & 4 & - \\ 
    H1  & Fully connected & 16 & 16 & ReLU \\
    H2  & Fully connected & 64 & 64 & ReLU \\
    H3  & Fully connected & 512 & 512 & ReLU \\
    Output & Fully connected & 64 & 64 & -\\
\end{tabular}
\end{center}

For the design-outcome pairs encoder we use the same architecture as in the design network, namely:
\begin{center}
\begin{tabular}{llrrr}
    Layer & Description & iDAD, InfoNCE & iDAD, NWJ & Activation \\
    \midrule
    Input & $\xi, y$ & 3 & 3 & - \\
    H1 & Fully connected & 64 & 64 & ReLU \\
    H2 & Fully connected & 512 & 512 & ReLU \\
    Output & Fully connected & 64 & 64 & - \\
    Attention & 8 heads & 64& 64 & - 
\end{tabular}
\end{center}

The output of the encoder, $R(h_t)$, is fed into fully connected head layers:
\begin{center}
\begin{tabular}{llrrr}
    Layer & Description & iDAD, InfoNCE & iDAD, NWJ & Activation \\
    \midrule
    Input & $R(h_t)$ & 64 & 64 & - \\
    H1 & Fully connected & 1024 & 1024 & ReLU \\
    H2 & Fully connected & 512 & 512 & ReLU   \\
    H3 & Fully connected & 512 & 512 & ReLU  \\
    Output & Fully connected & 64 & 64 & - \\
\end{tabular}
\end{center}

The optimisation was performed with Adam \citep{kingma2014adam} with \verb|ReduceLROnPlateau| learning rate scheduler, with the following hyperparameters:
\begin{center}
\begin{tabular}{lrr}
    Parameter & iDAD, InfoNCE & iDAD, NWJ \\
    \midrule
    Batch size & 2048 & 2048 \\
    Number of contrastive/negative samples &   2047 & 2047 \\
    Number of gradient steps & 100000 & 100000 \\ 
    Initial learning rate (LR) & 0.0005 & 0.0005 \\
    LR annealing factor & 0.8 & 0.8 \\
    LR annealing frequency (if no improvement) & 2000 & 2000  
    
\end{tabular}
\end{center}

\subsubsection{Performance of the variational baseline}

\begin{figure}[t]
  \centering
  \includegraphics[width=0.32\textwidth, height=4.5cm]{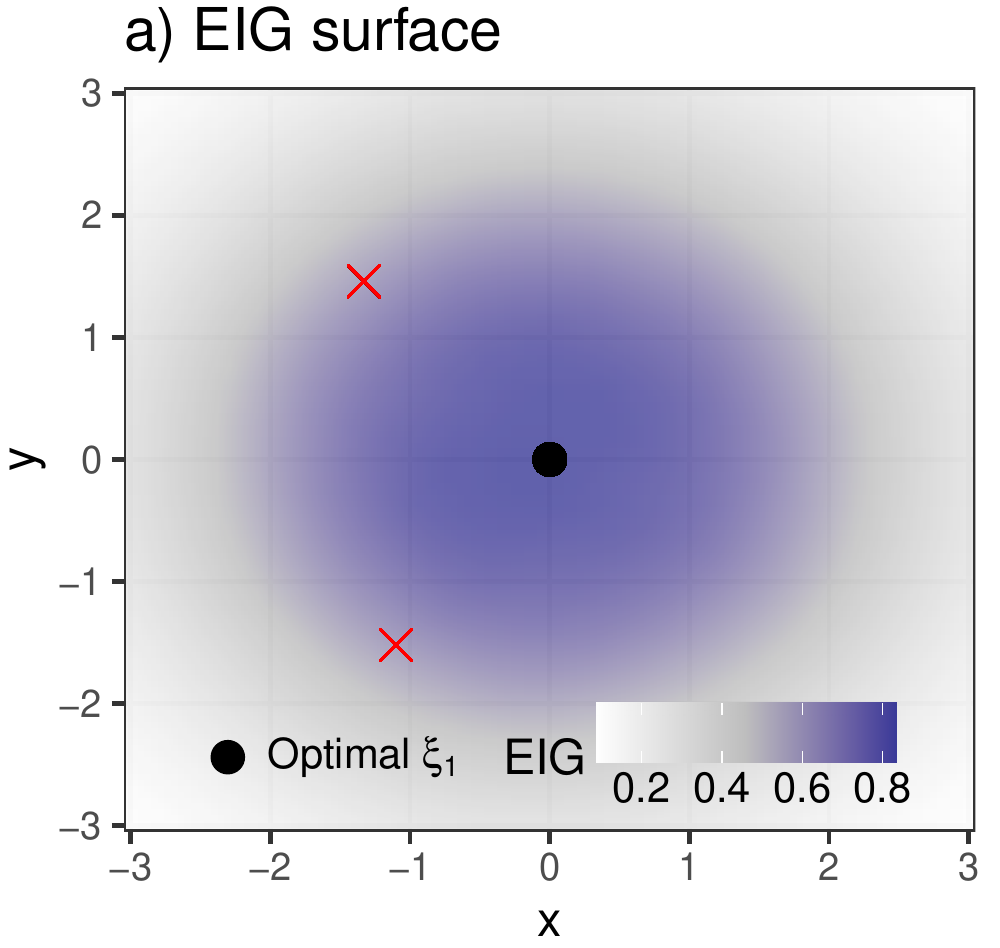}
  \includegraphics[width=0.32\textwidth, height=4.5cm]{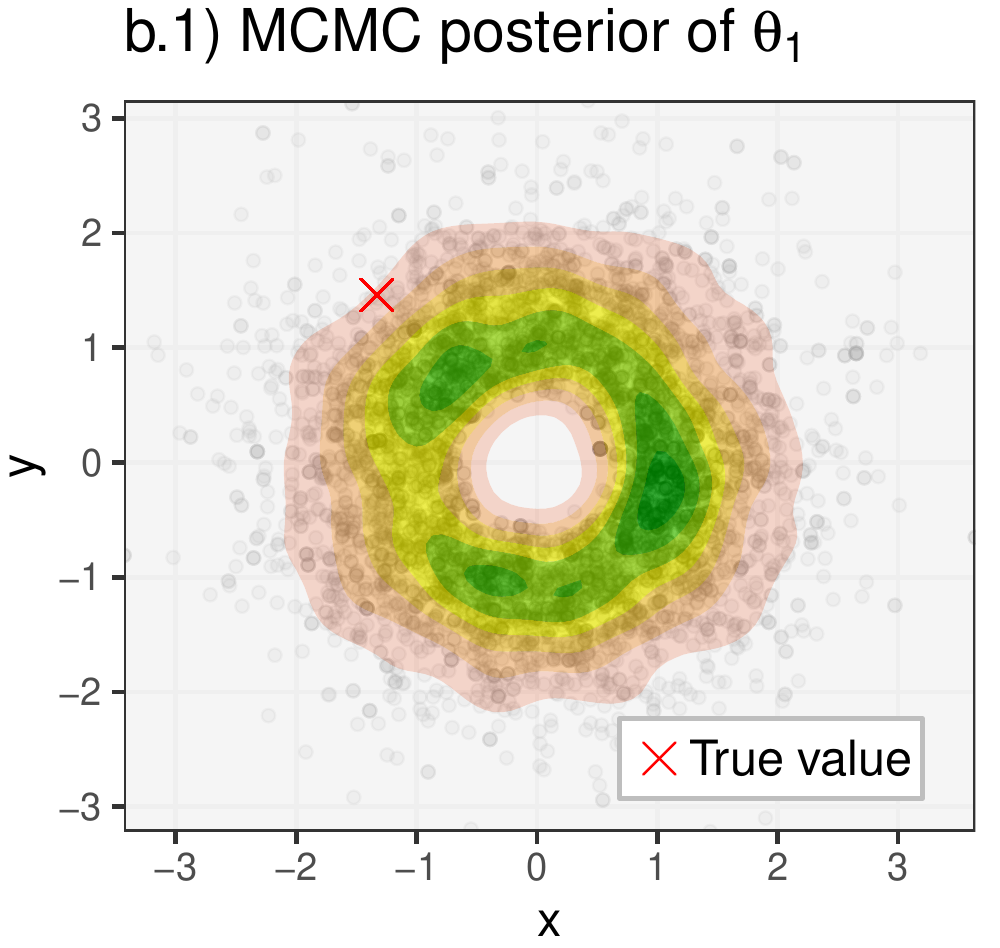}
  \includegraphics[width=0.32\textwidth, height=4.5cm]{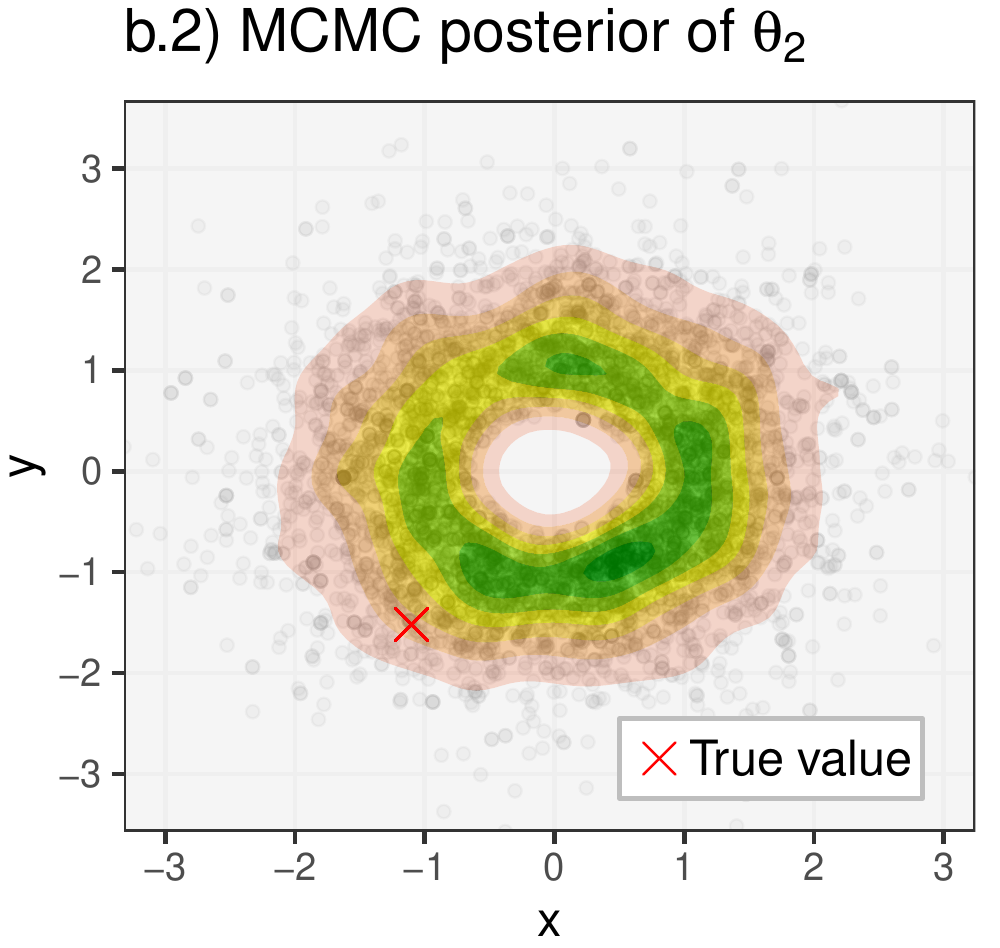}
  \caption{a): EIG surface induced by the prior; b) Samples from $p(\theta| \xi_1, y_1)$---the posterior distribution of the locations, after performing experiment $\xi_1$ and observing $y_1$, along with a KDE.}
  \label{fig:app_locfin_posteriors}
\end{figure}

As we saw in Table~\ref{tab:locfin_T10_highd}, this variational approach to (myopic) adaptive BOED performed very poorly, despite its large computational budget. The likely reason for that is that the mean-field variational approximation cannot adequately capture the complex non-Gaussian posterior of this problem. Figure~\ref{fig:app_locfin_posteriors} clearly demonstrates this: before any data is observed it is optimal to sample at the origin (since the prior is centered at it). After observing a low signal (the locations in this example are not close to the origin), we can only conclude that the sources are not within a small radius of the origin, but anywhere outside of it would be a plausible location, as indeed indicated by the fitted posteriors.

\subsubsection{Hyperparameter selection}\label{sec:app_locfin_hyperparam}
We did not perform extensive hyperparameters search; in particular, the network sizes were guided by two hyperparametes: hidden-dimension ($HD=512$) and encoding dimension ($ED=64$). We set-up all the networks to scale up with the number of experiments as follows:
\begin{itemize}
    \item Design-outcome encoder has three hidden layers of sizes $[64, HD, ED]$.
    \item Design emitter network  has three hidden layers of sizes [$HD/2, ED, 2]$, where 2 is the dimension of the design variable.
    \item The latent encoder for the critic network has four hidden layers of sizes $[16, 64, HD, ED]$.
    \item The critic design-outcome encoder's head layer has four hidden layers of sizes $[HD\times\log(T), HD\times\log(T)/2, HD, ED]$.
\end{itemize}

Since our multi-head attention layer has $8$ heads, the encoding dimension we use has to be a multiple of $8$. In addition to $ED=64$ we tried $ED=32$ which provided marginally worse results. We did not try other values for these hyperparameters.

For the learning rate, we tried $0.001$, which was too high, as well as $0.0005$ (which we selected) and $0.0001$ (which yielded very similar results). 

We performed similar level of hyperparameter tuning for all trainable baselines as well (DAD, MINEBED and SG-BOED). 

\subsubsection{Further ablation studies}\label{sec:app_experiments_locfin_furtherresults}

\begin{table}[t]
\centering
	\caption{Upper bounds on the total information, $\mathcal{I}_{10}(\pi)$, for the location finding experiment in Section~\ref{sec:experiments_locfin}. The bounds were estimated using $L=5\times10^5$ contrastive samples. Errors indicate $\pm1$ s.e. estimated over 4096  histories (128 for variational). Lower bounds are presented in Table \ref{tab:locfin_T10_highd}.
	} 
	\label{tab:locfin_T10_highd_upper}
{
	\begin{tabular}{lcccc}
		Method \textbackslash~$\theta$ dim.   &  4D &  6D & 10D & 20D   \\
		\toprule
		Random &   4.794 $\pm$ 0.041 &  3.506 $\pm$ 0.004 &  1.895 $\pm$ 0.003 & 0.552 $\pm$  0.001  \\
		MINEBED       & 5.522    $\pm$  0.028 & 4.229 $\pm$ 0.029 &  2.459 $\pm$ 0.029 & 0.801 $\pm$ 0.019  \\
		SG-BOED       &  5.549   $\pm$ 0.028 & 4.220 $\pm$ 0.030 & 2.455 $\pm$ 0.029 &  0.803 $\pm$ 0.019 \\
		Variational   &  4.644 $\pm$  0.146 &  3.626 $\pm$ 0.167 & 2.181 $\pm$ 0.152 & 0.669 $\pm$ 0.097 \\
		\textbf{iDAD} (NWJ)    &  \textbf{7.806 $\pm$ 0.050} & 5.851 $\pm$ 0.041 & \textbf{3.264} $\pm$ \textbf{0.039} & \textbf{0.877 $\pm$ 0.022}  \\
		\textbf{iDAD} (InfoNCE) &  \textbf{7.863} $\pm$ \textbf{0.043} & \textbf{6.068} $\pm$ \textbf{0.039} & \textbf{3.257} $\pm$ \textbf{0.040} &\textbf{ 0.872} $\pm$ \textbf{0.020} \\
		\midrule
		DAD     &  8.034 $\pm$  0.038  & 6.310 $\pm$ 0.031 &  3.358  $\pm$ 0.040 & 0.953 $\pm$  0.022 \\
		\bottomrule
	\end{tabular}
}
\end{table}

{\bf Scalability with number of experiments.} We first demonstrate that iDAD can scale to a larger number of experiments $T$. We train policy networks to perform $T=20$ experiments and compare them to baselines in Table~\ref{tab:locfin_T20_results}. We omit the variational baseline as it is too computationally costly to run for a large enough number of histories, and as we saw in the previous subsection, it is not particularly suited to this model.    
\begin{table}[t]
\centering
	\caption{Upper and lower bounds on the total information, $\mathcal{I}_{20}(\pi)$, for the location finding experiment in 2D from Section~\ref{sec:experiments_locfin}. The bounds were estimated using $L=5\times10^5$ contrastive samples. Errors indicate $\pm1$ s.e. estimated over 4096  histories. 
	}
	\label{tab:locfin_T20_results}
{
	\begin{tabular}{lcc}
		Method  &  Lower bound &  Upper bound   \\
		\toprule
		Random       &  \phantom{0}7.000 $\pm$ 0.034 &  \phantom{0}7.020 $\pm$ 0.034   \\
		MINEBED      &  \phantom{0}7.672   $\pm$ 0.030    & \phantom{0}7.690 $\pm$   0.031      \\
		SG-BOED        & \phantom{0}7.701   $\pm$ 0.030   &  \phantom{0}7.728   $\pm$ 0.031     \\
		\textbf{iDAD} (NWJ)    &  \phantom{0}\textbf{9.961} $\pm$ \textbf{0.033}  & \textbf{10.372} $\pm$   \textbf{0.048}     \\
		\textbf{iDAD} (InfoNCE)& \textbf{10.075} $\pm$ \textbf{0.032}  & \textbf{10.463} $\pm$ \textbf{0.043}     \\ 
		\midrule
		DAD     & 10.424   $\pm$  0.031   &  10.996   $\pm$ 0.049 \\
		\bottomrule
	\end{tabular}
}
\end{table}

{\bf Training stability.} To assess the robustness of the results and the stability of the training process, we trained 5 additional iDAD networks with each of the two bounds, using different seeds but the same hyperparameters (described in Subsection~\ref{sec:app_locfin_traindeets}) we used to produce the results of the location finding experiment in 2D (Table~\ref{tab:locfin_T10_highd} in the main text). We report upper and lower bounds on the mutual information along with their mean and standard error in the table below.

\begin{center}
\begin{tabular}{llrrrrrrr}
    Estimator & Bound & Run 1 &  Run 2 & Run 3 & Run 4 & Run 5 & \textbf{Mean} & \textbf{SE}\\
    \midrule
    InfoNCE   & Lower & 7.826 & 7.682 & 	7.856	& 7.713 & 7.804 & \textbf{7.776} & \textbf{0.034}  \\
    InfoNCE   & Upper & 7.933 & 7.791 &     7.856   &  7.807 & 7.925 & \textbf{7.862} & \textbf{0.029} \\
    \midrule
    NWJ & Lower & 7.820 &	7.545 &	7.592 & 7.555 &	7.691 &  \textbf{7.641} & \textbf{0.052}\\
    NWJ & Upper & 7.976 &   7.640 & 7.669 & 7.651 & 7.800 & \textbf{7.747} & \textbf{0.064}
\end{tabular}
\end{center}

We can see that the iDAD networks trained with InfoNCE are highly stable, with the 5 additional runs achieving very similar mutual information values to each other and to the iDAD network used the report the results in the main paper. The performance of the iDAD networks trained with the NWJ bound is more variable and empirically achieve slightly lower average value of mutual information. This higher variance is in-line with the discussion in Section~\ref{sec:app_comp}.

We similarly verify the robustness of the static baselines, reporting the results in the table below:
\begin{center}
\begin{tabular}{llrrrrrrr}
    Estimator & Bound & Run 1 &  Run 2 & Run 3 & Run 4 & Run 5 & \textbf{Mean} & \textbf{SE}\\
    \midrule
    SG-BOED   & Lower & 5.537 & 5.536 & 5.473 & 5.523 & 5.518 & \textbf{5.517} & \textbf{0.013}  \\
    SG-BOED   & Upper & 5.553 & 5.548 & 5.491 & 5.541 & 5.531 & \textbf{5.533} &  \textbf{0.012} \\
    \midrule
    MINEBED & Lower & 5.460 & 5.506 & 5.553 & 5.539 & 5.565 & \textbf{5.524} & \textbf{0.021}  \\
    MINEBED & Upper & 5.473 & 5.526 & 5.567 & 5.554 & 5.574 &  \textbf{5.540} & \textbf{0.022}
\end{tabular}
\end{center}

{\bf Performance sensitivity to errors in the policy.} Finally, we investigate the effect of slight errors in the design policy network. To this end, we look at the performance achieved by partially trained design networks (there will be some errors or inaccuracies in networks that were not trained until convergence). Table  \ref{tab:locfin_error_sensitivity} shows the performance of iDAD as a function of training time, demonstrating that small errors in the network only lead to small drops in performance.

In detail, our results show that with just $8\%$ of the total training budget, this slightly inaccurate network still performs relatively well, achieving total mutual information of 7.1, compared to the fully trained network that reached 7.8. We also highlight that iDAD outperforms all baselines with as little as $1\%$ of the total training budget (the best performing baseline achieves mutual information of 5.5, see Table \ref{tab:locfin_T10_highd}).

\begin{table}[t]
\centering
	\caption{Ablation study on the performance of iDAD as a function of training time for the location finding experiment.
	}
	\label{tab:locfin_error_sensitivity}
\begin{tabular}{lr}
	Training budget &	MI lower bound \\
	\midrule
    0.1$\%$ &	3.38 \\
    1.0$\%$	&   6.09 \\
    2.0$\%$ &	6.46\\
    4.0$\%$	&   6.81\\
    8.0$\%$ &	7.08\\
    16.0$\%$ &	7.33\\
    32.0$\%$ &	7.56\\
    64.0$\%$ &	7.78\\
    100.0$\%$ &	7.82\\
    \bottomrule
\end{tabular}
\end{table}

\subsection{PK model} \label{sec:appendix_experiments_pk}

The drug concentration $z$, measured $\xi$ hours after administering it, and the corresponding noisy observation $y$ are given by
\begin{equation}
    z(\xi;\theta) = \frac{D_V}{V} \frac{k_\alpha}{k_\alpha - k_e} [e^{-k_e \xi} - e^{-k_\alpha \xi}], \quad y(\xi;\theta) = z(\xi;\theta)(1 + \epsilon) + \eta
\end{equation}
where $\theta = (k_\alpha, k_e, V)$,  $D_V=400$ is a constant, $\epsilon \sim \mathcal{N}(0, 0.01)$ is multiplicative noise to account for heteroscedasticity and $\eta \sim \mathcal{N}(0, 0.1)$ is an additive observation noise. Since both noise sources are Gaussian, the observation likelihood is also Gaussian i.e.
\begin{equation}
    y(\xi;\theta) \sim \mathcal{N}(z(\xi;\theta), 0.01 z(\xi;\theta)^2 + 0.1)
\end{equation}
The prior for the parameters $\theta$ that we used
\begin{equation}
    \log \theta \sim \mathcal{N}\left(\left[\begin{array}{c}
    \log 1 \\
    \log 0.1 \\
    \log 20
    \end{array}\right],\left[\begin{array}{ccc}
    0.05 & 0 & 0 \\
    0 & 0.05 & 0 \\
    0 & 0 & 0.05
    \end{array}\right]\right)
\end{equation}

\subsubsection{Training details}

The architecture of the design network $\designnet$ used for Figure~\ref{fig:pk_trajectories_designs_eval} and~\ref{fig:pk_loss} and all its hyperparameters are in the following tables. For the encoder of the design-outcome pairs we used the following:

\begin{center}
\begin{tabular}{llrrr}
    Layer & Description & iDAD, InfoNCE & iDAD, NWJ & Activation \\
    \midrule
    Input & $\xi, y$ & 2 & 2 & - \\
    H1 & Fully connected & 64 & 64 & ReLU \\
    H2 & Fully connected & 512 & 512 & ReLU \\
    Output & Fully connected & 32 & 32 & - \\
    Attention & 8 heads & 32 & 32 & - 
\end{tabular}
\end{center}

The outputs of the encoder, $\{R(h_t)\}_{t=1}^T$, are summed and the resulting vector (of dimension 32) is fed into an emitter network, for which we used the following:
\begin{center}
\begin{tabular}{llrrr}
    Layer & Description & iDAD, InfoNCE & iDAD, NWJ & Activation \\
    \midrule
    Input & $R(h_t)$ & 32 & 32 & - \\
    H1 & Fully connected & 256 & 256 & ReLU \\
    H2 & Fully connected & 32 & 32 & ReLU \\
    Output & Fully connected & 1 & 1 & Sigmoid \\
\end{tabular}
\end{center}

The architecture of the critic network $\criticnet$ used in Figures~\ref{fig:pk_trajectories_designs_eval} and~\ref{fig:pk_loss} and all its hyperparameters are in the following tables. For the encoder of the design-outcome pairs we used the same architecture as for the design network, namely:
\begin{center}
\begin{tabular}{llrrr}
    Layer & Description & iDAD, InfoNCE & iDAD, NWJ & Activation \\
    \midrule
    Input & $\xi, y$ & 2 & 2 & - \\
    H1 & Fully connected & 64 & 64 & ReLU \\
    H2 & Fully connected & 512 & 512 & ReLU \\
    Output & Fully connected & 32 & 32 & - \\
    Attention & 8 heads & 32 & 32 & - 
\end{tabular}
\end{center}
The resulting pooled representation, $R(h_T)$ is fed into fully connected critic head layers with the following architecture:
\begin{center}
\begin{tabular}{llrrr}
    Layer & Description & iDAD, InfoNCE & iDAD, NWJ & Activation \\
    \midrule
    Input & $R(h_T)$ & 32 & 32 & - \\
    H1 & Fully connected & 512 & 512 & ReLU \\
    H2 & Fully connected & 256 & 256 & ReLU   \\
    H3 & Fully connected & 512 & 512 & ReLU  \\
    Output & Fully connected & 32 & 32 & - \\
\end{tabular}
\end{center}

Finally, for the latent variable encoder network we used:
\begin{center}
\begin{tabular}{llrrr}
    Layer & Description & iDAD, InfoNCE & iDAD, NWJ & Activation \\
    \midrule 
    Input & $\theta$ & 3 & 3 & - \\ 
    H1  & Fully connected & 8 & 8 & ReLU \\
    H2  & Fully connected & 64 & 64 & ReLU \\
    H3  & Fully connected & 512 & 512 & ReLU \\
    Output & Fully connected & 32 & 32 & -\\
\end{tabular}
\end{center}

The optimisation was performed with Adam \citep{kingma2014adam} with the following hyperparameters:
\begin{center}
\begin{tabular}{lrr}
    Parameter & iDAD, InfoNCE & iDAD, NWJ \\
    \midrule
    Batch size &  1024 & 1024  \\
    Number of contrastive/negative samples &    1023 & 1023  \\
    Number of gradient steps & 100000  & 100000  \\ 
    Initial learning rate (LR) &  0.0001 &  0.0001  \\
    LR annealing factor & 0.8 & 0.5 \\
    LR annealing frequency (if no improvement) & 2000 & 2000  
\end{tabular}
\end{center}

\subsubsection{Hyperparameter selection}
Hyperparameter selection was done in a way similar to the Location Finding experiment (see~\ref{sec:app_locfin_hyperparam}). We tried encodin dimensions $ED=32, 64$ and selected the smaller size as there were no clear benefits to larger networks (relatively speaking, this is an easier model that the location finding). We used the same hidden dimension, i.e. $HD=512$. In terms of learning rates, we tried $0.0001$, $0.0005$ and $0.001$; we found $0.0001$ to be appropriate, although NWJ bound was exhibiting high variance, so used a smaller learning rate annealing factor for that network (0.5 vs 0.8 for InfoNCE). We performed similar level of hyperparameter tuning for all trainable baselines as well (DAD, MINEBED and SG-BOED).

\subsubsection{Further results}
Table~\ref{tab:pharmaco_T5_results} reports the results shown in Figure~\ref{fig:pk_trajectories_designs_eval}\textcolor{red}{c)}, along with the corresponding upper bounds and deployment times, while Table~\ref{tab:pharmaco_T10_results} reports the results for $T=10$.
\begin{table}[t]
\centering
	\caption{Upper and lower bounds on the total information, $\mathcal{I}_{5}(\pi)$, for the pharmacokinetic experiment. Errors indicate $\pm1$ s.e. estimated over 4096 (126 for variational) histories and $L=5\times10^5$.
	}
	\label{tab:pharmaco_T5_results}
	\begin{tabular}{lrrr}
		Method               &  Lower bound &  Upper bound & Deployment time   \\
		\toprule
		Random            & 2.523 $\pm$ 0.033 &   2.523 $\pm$ 0.033 & N/A  \\
		Equal interval    & 2.651 $\pm$ 0.022 &   2.651 $\pm$ 0.022 & N/A \\
		MINEBED           &  2.955 $\pm$    0.030  &   2.956  $\pm$ 0.030 & N/A \\
		SG-BOED           &     2.985   $\pm$   0.027 &  2.985 $\pm$ 0.027   & N/A \\
		Variational  & 2.683 $\pm$ 0.093    & 2.683 $\pm$ 0.093 & 505.4\phantom{00} $\pm$ 1\%\\ %
	    \textbf{IDAD} (NWJ) &  \textbf{3.163} $\pm$ \textbf{0.023}  & \textbf{3.163} $\pm$ \textbf{0.023}  &  0.007 $\pm$ 7\% \\
		\textbf{IDAD} (InfoNCE) & \textbf{3.200} $\pm$ \textbf{0.024}  & \textbf{3.200} $\pm$ \textbf{0.024}  &  0.007 $\pm$ 8\% \\
	    \midrule
	    DAD & 3.234 $\pm$ 0.023 & 3.234 $\pm$ 0.023 & 0.002 $\pm$ 7\% \\
		\bottomrule
	\end{tabular}
	  \vspace{-10pt}
\end{table}

\begin{table}[t]
\centering
	\caption{Upper and lower bounds on the total information, $\mathcal{I}_{10}(\pi)$, for the pharmacokinetic experiment. Errors indicate $\pm1$ s.e. estimated over 4096 (126 for variational)  histories and $L=5\times10^5$.
	}
	\label{tab:pharmaco_T10_results}
	\begin{tabular}{lrrr}
		Method               &  Lower bound &  Upper bound & Deployment time   \\
		\toprule
		Random            & 3.344 $\pm$ 0.034 &   3.345 $\pm$ 0.034 &  N/A \\
		Equal interval    & 3.422 $\pm$ 0.026 &   3.423 $\pm$ 0.026 &N/A \\
		MINEBED           & 3.849 $\pm$ 0.034     &   3.849 $\pm$ 0.034 & N/A \\
		SG-BOED          &   3.824     $\pm$   0.034      & 3.824    $\pm$  0.034 &N/A  \\
		Variational & 3.624 $\pm$ 0.099    & 3.624 $\pm$ 0.099 & 1055.2\phantom{00} $\pm$ 8\% \\ 
	    \textbf{IDAD} (NWJ) &  \textbf{ 4.034} $\pm$ \textbf{0.025}  & \textbf{4.034} $\pm$ \textbf{0.025}  & 0.007 $\pm$  6\%  \\
		\textbf{IDAD} (InfoNCE) & \textbf{4.045} $\pm$ \textbf{0.026}  & \textbf{4.045} $\pm$ \textbf{0.026}  & 0.007 $\pm$  5\%  \\
	    \midrule
	    DAD     & 4.116 $\pm$ 0.024  & 4.117  $\pm$ 0.024 & 0.007 $\pm$ 8\%  \\
		\bottomrule
	\end{tabular}
\end{table}

{\bf Training stability.} To assess the robustness of the results and the stability of the training process, we trained 5 additional iDAD networks with each of the two bounds, using different seeds but the same hyperparameters we used to produce the results of the pharmacokinetic experiment (Figure~\ref{fig:pk_trajectories_designs_eval}\textcolor{red}{c)} and corresponding Table~\ref{tab:pharmaco_T5_results}). We report upper and lower bounds on the mutual information along with their mean and standard error in the table below.

\begin{center}
\begin{tabular}{llrrrrrrr}
    Method & Bound & Run 1 &  Run 2 & Run 3 & Run 4 & Run 5 & \textbf{Mean} & \textbf{SE}\\
    \midrule
    iDAD, InfoNCE   & Lower & 3.209 & 3.165 & 3.198 & 3.221 & 3.128  &  \textbf{3.185}  & \textbf{0.019} \\
    iDAD, InfoNCE   & Upper & 3.210 & 3.166 & 3.201 & 3.223 & 3.130 &  \textbf{3.186}  & \textbf{0.019}\\
    \midrule
    iDAD, NWJ & Lower &  3.034 & 3.049 & 2.608& 3.149& 3.082 & \textbf{3.034} & \textbf{0.107}  \\
    iDAD, NWJ & Upper &  3.034 & 3.049 & 2.609 &3.150 & 3.083    &   \textbf{3.034} & \textbf{0.107}
\end{tabular}
\end{center}

We repeat the same procedure for the static baselines. The results reported in the table below demonstrate the training stability of these baselines as well.  
\begin{center}
\begin{tabular}{llrrrrrrr}
    Method & Bound & Run 1 &  Run 2 & Run 3 & Run 4 & Run 5 & \textbf{Mean} & \textbf{SE}\\
    \midrule
    SG-BOED   & Lower & 2.932 & 2.452 & 2.448 & 2.991 & 2.962 & \textbf{2.757} & \textbf{0.140}  \\
    SG-BOED   & Upper & 2.932 & 2.453 & 2.449 & 2.992 & 2.962 &  \textbf{2.757}  & \textbf{0.140} \\
    \midrule
    MINEBED & Lower & 2.912  &  2.213 & 3.014 & 2.092 & 2.941 & \textbf{2.634}  & \textbf{0.221}  \\
    MINEBED & Upper & 2.914  &  2.213 & 3.015 & 2.092 & 2.942 &  \textbf{2.635} & \textbf{0.222}
\end{tabular}
\end{center}

\vspace{-0.2cm}
\subsection{SIR Model}
\label{sec:appendix_sir}

Generally speaking, the SIR model advocates that, within a fixed population of size $N$, susceptible individuals $S(\tau)$, where $\tau$ is time, can become infected and move to an infected state $I(\tau)$. The infected individuals can then recover from the disease and move to the recovered state $R(\tau)$. The dynamics of these events are governed by the infection rate $\beta$ and recovery rate $\gamma$, which define the particular disease in question. In the context of BOED, the aim is generally to estimate these two model parameters by observing state populations at particular measurement times $\tau$, which are the experimental design variables. The SIR model has been studied extensively in the context of BOED, e.g.~in~\cite{kleinegesse2018efficient, kleinegesse2020sequential, kleinegesse2021gradientbased, dehideniya2018}.

Stochastic versions of the SIR model are usually formulated via continuous-time Markov chains (CTMC), which can be simulated from via the Gillespie algorithm~\cite{allen2017}, yielding discrete state populations. However, iDAD requires us to differentiate through the sampling path of the state populations to the experimental designs, which is impossible if the simulated data is discrete as gradients are undefined. Thus, we here implement an alternative formulation of the stochastic SIR model that is based on stochastic differential equations (SDEs), as studied in~\cite{kleinegesse2021gradientbased}, which yields continuous state populations that can be differentiated.

Following~\cite{kleinegesse2021gradientbased}, let us first define a state population vector $\mathbf{X}(\tau) = (S(\tau), I(\tau))^\top$, where we can safely ignore the population of recovered $R(\tau)$ for modelling purposes because we assume that the total population stays fixed. The system of Itô SDEs that defines the stochastic SIR model is given by
\begin{equation} \label{eq:sde}
\mathrm{d}\mathbf{X}(\tau) = \mathbf{f}(\mathbf{X}(\tau)) \mathrm{d}\tau + \mathbf{G}(\mathbf{X}(\tau))\mathrm{d}\mathbf{W}(\tau),
\end{equation}
where $\mathbf{f}$ is a drift vector, $\mathbf{G}$ is a diffusion matrix and $\mathbf{W}(\tau)$ is a vector of independent Wiener processes. \cite{kleinegesse2021gradientbased} showed that the drift vector and diffusion matrix are given by
\begin{equation}
\mathbf{f}(\mathbf{X}(\tau)) 
= \begin{pmatrix}
        -\beta \frac{S(\tau)I(\tau)}{N} \\[1em]
        \beta \frac{S(\tau)I(\tau)}{N} - \gamma I(\tau) \\[1em]
    \end{pmatrix} \quad \text{and} \quad
\mathbf{G}(\mathbf{X}(t)) 
=  \begin{pmatrix}
        - \sqrt{\beta \frac{S(\tau)I(\tau)}{N}} & 0 \\[1em]
        \sqrt{\beta \frac{S(\tau)I(\tau)}{N}}   & - \sqrt{\gamma I(\tau)} \\[1em]
    \end{pmatrix}.
\end{equation}
Given the system of Itô SDEs in~\eqref{eq:sde}, as well as the above drift vector and diffusion matrix, we can then simulate state populations $\mathbf{X}(\tau)$ by solving the SDE using finite-difference methods, such as e.g.~the Euler-Maruyama method. See~\cite{kleinegesse2021gradientbased} for more information on the SDE-based SIR model, including derivations of the drift vector and diffusion matrix.

Importantly, we note that~\cite{kleinegesse2021gradientbased} further used the solutions of~\eqref{eq:sde} as an input to a Poisson observation model, which increases the noise in simulated data. We here opt to simply use the solutions of~\eqref{eq:sde} as data and do not consider an additional Poisson observational model.

\subsubsection{Training details}

As previously mentioned, the design variable for this model is the measurement time $\tau \in [0, 100]$. When solving the SDE with the Euler-Maruyama method, we discretize the time domain with a resolution of $\Delta \tau = 10^{-2}$. We here only use the number of infected $I(\tau)$ as the observed data, as others might be difficult to measure in reality. The total population is fixed at $N=500$ and the initial conditions are $\mathbf{X}(\tau=0) = (0, 2)^\top$. The model parameters $\beta$ and $\gamma$ have log-normal priors, i.e.~$p(\beta) = \text{Lognorm}(0.50, 0.50^2)$ and $p(\gamma) = \text{Lognorm}(0.10, 0.50^2)$. Importantly, because solving SDEs is expensive, we pre-simulate our data on a time grid, store it in memory and then access the relevant data during training.

We present the network architectures and hyper-parameters corresponding to the $T=5$ iDAD results shown in Table~\ref{tab:sir_results} of the main text. For the encoder of the design-outcome pairs we used:
\begin{center}
\begin{tabular}{llrrr}
    Layer & Description & iDAD, InfoNCE & iDAD, NWJ & Activation \\
    \midrule
    Input & $\xi, y$ &  2 & 2 & - \\
    H1 & Fully connected  & 8 & 8 & ReLU \\
    H2 & Fully connected  & 64 & 64 & ReLU \\
    H3 & Fully connected  & 512 & 512 & ReLU \\
    Output & Fully connected & 32 & 32 & - \\
\end{tabular}
\end{center}
The resulting representations, \{$R(h_t)\}_{t=1}^{T-1}$, are stacked into a matrix (as new design--outcome pairs are obtained) and fed into an emitter network, which contains an LSTM cell with two hidden layers. We only keep the last hidden state of the LSTM's output and pass it through a final FC layer: 
\begin{center}
\begin{tabular}{llrrr}
    Layer & Description & iDAD, InfoNCE & iDAD, NWJ & Activation \\
    \midrule
    Input & $\{R(h_t)\}_{t=1}^{T-1}$ & 32 $\times$ t & 32 $\times$ t& - \\
    H1 \& H2 & LSTM & 32 & 32 & - \\
    H3 & Fully connected & 16 &  16 & ReLU \\
    Output & Fully connected & 1 & 1 & - \\
\end{tabular}
\end{center}

The architecture of the critic network $\criticnet$ used in Table~\ref{tab:sir_results} and all its hyper-parameters are in the tables that follow. First, the encoder network of the latent variables is:
\begin{center}
\begin{tabular}{llrrr}
    Layer & Description & iDAD, InfoNCE & iDAD, NWJ & Activation \\
    \midrule
    Input & $\theta$ &  2 & 2 & - \\
    H1 & Fully connected  & 8 & 8 & ReLU \\
    H2 & Fully connected  & 64 & 64 & ReLU \\
    H3 & Fully connected  & 512 & 512 & ReLU \\
    Output & Fully connected & 32 & 32 & - \\
\end{tabular}
\end{center}

For the design-outcome pairs encoder we use the same architecture as in the design network, namely:
\begin{center}
\begin{tabular}{llrrr}
    Layer & Description & iDAD, InfoNCE & iDAD, NWJ & Activation \\
    \midrule
    Input & $\xi, y$ &  2 & 2 & - \\
    H1 & Fully connected  & 8 & 8 & ReLU \\
    H2 & Fully connected  & 64 & 64 & ReLU \\
    H3 & Fully connected  & 512 & 512 & ReLU \\
    Output & Fully connected & 32 & 32 & - \\
\end{tabular}
\end{center}

The outputs of the encoder, $\{R(h_t)\}_t$, are stacked and fed into an LSTM cell with two hidden layers. We only keep the last hidden state of the LSTM's output and pass it through a FC layer:
\begin{center}
\begin{tabular}{llrrr}
    Layer & Description & iDAD, InfoNCE & iDAD, NWJ & Activation \\
    \midrule
    Input & $\{R(h_t)\}_{t=1}^{T-1}$ & 32 $\times$ t & 32 $\times$ t& - \\
    H1 \& H2 & LSTM & 32 & 32 & - \\
    H3 & Fully connected & 16 &  16 & ReLU \\
    Output & Fully connected & 32 & 32 & - \\
\end{tabular}
\end{center}

The optimization was performed with Adam \citep{kingma2014adam} with learning rate annealing with the following hyper-parameters:
\begin{center}
{
\begin{tabular}{lrr}
    Parameter & iDAD InfoNCE & iDAD, NWJ \\
    \midrule
    Batch size & 512 & 512 \\
    Number of contrastive/negative samples & 511  & 511 \\
    Number of gradient steps & 100000 & 100000 \\ 
    Initial learning rate (LR) & 0.0005 & 0.0005 \\
    LR annealing factor & 0.96 & 0.96 \\
    LR annealing frequency  & 1000 & 1000
\end{tabular}
}
\end{center}

\subsubsection{Further results}

{\bf Different number of experiments $T$.}
In Table~\ref{tab:sir_saturation_results} we show lower bound estimates when applying iDAD with the InfoNCE lower bound to the SDE-based SIR model for different number of measurements $T$. The design network and critic architectures are the same as for $T=5$. Table~\ref{tab:sir_saturation_results} shows that more measurements yield higher expected information gains, as one might intuitively expect. Furthermore, the increase in expected information gain saturates with increasing $T$, which is why we presented the results for $T=5$ in the main text. 
The biggest increase, however, occurs from $T=1$ to $T=2$. This is intuitive, because the SIR model has two model parameters that we wish to estimate but we only gather one data point with one measurement. 
Hence, in order to accurately estimate both of these parameters, we would need at least $2$ measurements, which is reflected in Table~\ref{tab:sir_saturation_results}. We note that all of these numbers, with the exception of $T=1$, are larger than those found by~\cite{kleinegesse2021gradientbased}. This increase in expected information gain may be explained by the fact that~\cite{kleinegesse2021gradientbased} use an additional Poisson observation model, which means that the resulting data are inherently noisier and less informative.

\begin{table}[!h]
\centering
	\caption{InfoNCE lower bound estimates ($\pm$ s.e.) when applying iDAD to the SDE-based SIR model for different number of measurements $T$.}
	\label{tab:sir_saturation_results}
	\begin{tabular}{rll}
		$T$  &  iDAD, InfoNCE &  iDAD, NWJ \\
		\toprule
		1  & 1.396 $\pm$ 0.018 & 1.417 $\pm$ 0.001 \\
		2  & 2.714 $\pm$  0.019 & 2.699 $\pm$ 0.001 \\
		3  &  3.554 $\pm$ 0.021 & 3.515 $\pm$ 0.001 \\
		4  & 3.600 $\pm$ 0.018 & 3.749 $\pm$  0.001 \\
		5  & 3.915 $\pm$ 0.020 & 3.869 $\pm$ 0.001 \\
		7  & 4.027 $\pm$ 0.019 & 3.911 $\pm$ 0.001 \\
		10 & 4.100 $\pm$ 0.020 & 4.019 $\pm$ 0.001 \\       
		\bottomrule
	\end{tabular}
\end{table}

\begin{figure}[t]
  \centering
  \includegraphics[width=0.96\textwidth]{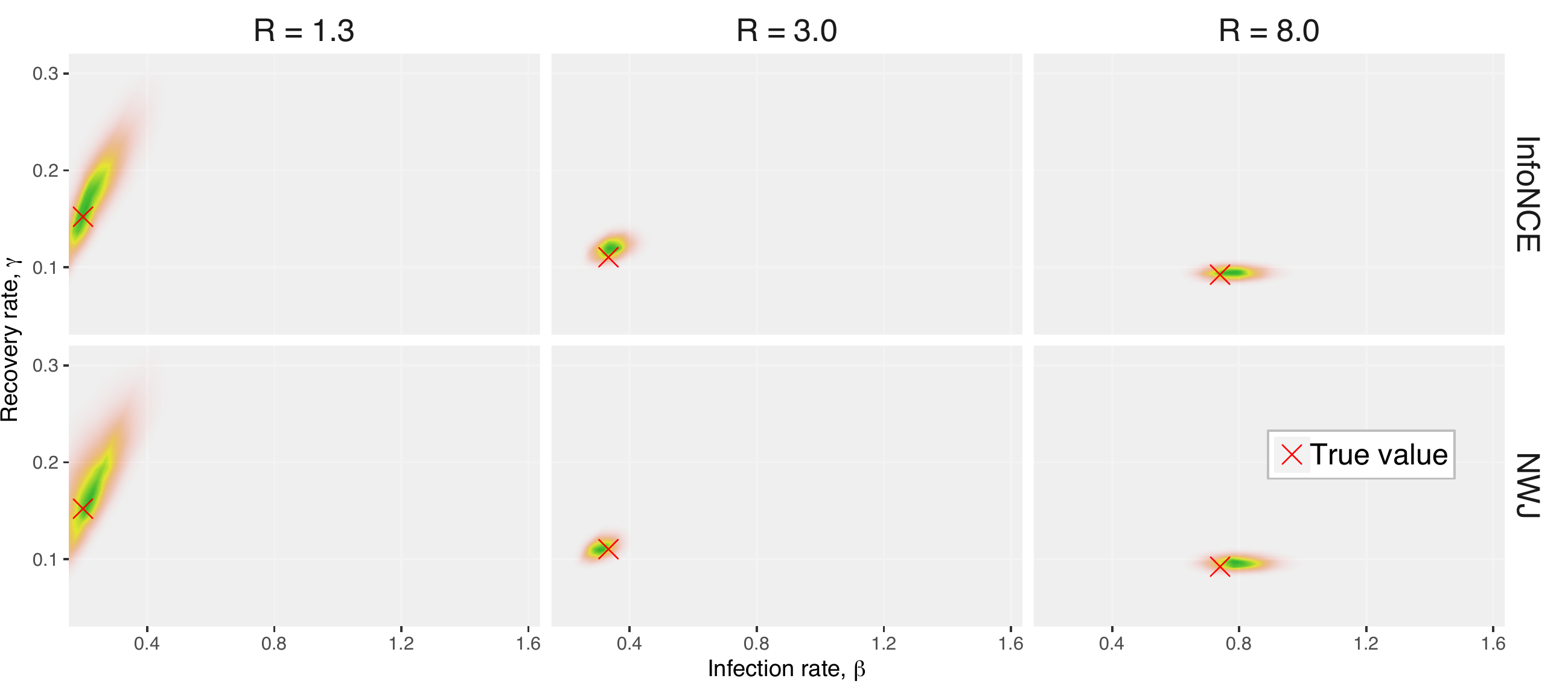}
  \caption{ Approximate posteriors for the SIR model.}
  \label{fig:sir_posteriors_multiple}
  \vspace{-0.3cm}
\end{figure}

{\bf Training stability.} To assess the robustness of the results and the stability of the training process, we trained 5 additional iDAD networks with each of the two bounds, using different seeds but the same hyperparameters we used to produce the results of Table~\ref{tab:sir_results} in the main text. We report upper and lower bounds on the mutual information along with their mean and standard error in the table below.
\begin{center}
	\small
\begin{tabular}{llrrrrrrr}
    Method & Bound & Run 1 &  Run 2 & Run 3 & Run 4 & Run 5 & \textbf{Mean} & \textbf{SE}\\
    \midrule
    iDAD, InfoNCE   & Lower & 3.900	 &  3.919	& 	3.919		& 3.901	 & 3.887  & \textbf{3.906} & \textbf{0.007}  \\
    iDAD, NWJ & Lower & 3.872	  &  3.838	& 3.854  & 3.883	 & 3.848 & \textbf{3.859} & \textbf{0.009}  \\
\end{tabular}
\end{center}
We repeat the same procedure for the static baselines. The results reported in the table below demonstrate the training stability of these baselines as well.  
\begin{center}
	\small
\begin{tabular}{llrrrrrrr}
    Method & Bound & Run 1 &  Run 2 & Run 3 & Run 4 & Run 5 & \textbf{Mean} & \textbf{SE}\\
    \midrule
    SG-BOED   & Lower &  3.713		 &  3.765 	& 	3.767	& 3.764	  &  3.739	& \textbf{3.749} & \textbf{0.012 }  \\
    MINEBED & Lower &  3.373  &  3.438 	& 	3.376	& 3.379  & 3.420  & \textbf{ 3.397} & \textbf{0.015 }  \\
\end{tabular}
\end{center}

\end{document}